\documentclass[11pt]{article}

\usepackage{epsfig,epsf,fancybox}
\usepackage{amsmath}
\usepackage{mathrsfs}
\usepackage{amssymb}
\usepackage{graphicx}
\usepackage{color}
\usepackage{multirow}
\usepackage{paralist}
\usepackage{verbatim}
\usepackage{galois}
\usepackage{algorithm}
\usepackage{algorithmic}
\usepackage{boxedminipage}
\usepackage{booktabs}
\usepackage{accents}
\usepackage{stmaryrd}
\usepackage{xcolor}
\usepackage{wrapfig}
\usepackage{hhline}

\usepackage{subfig}

\usepackage{natbib}

\usepackage{url}
\usepackage[colorlinks,linkcolor=magenta,citecolor=blue, pagebackref=true,backref=true]{hyperref}
\renewcommand*{\backrefalt}[4]{%
    \ifcase #1 \footnotesize{(Not cited.)}%
    \or        \footnotesize{(Cited on page~#2.)}%
    \else      \footnotesize{(Cited on pages~#2.)}%
    \fi}

\newtheorem{theorem}{Theorem}[section]

\newtheorem{lemma}[theorem]{Lemma}
\newtheorem{proposition}[theorem]{Proposition}

\newtheorem{definition}[theorem]{Definition}
\newtheorem{example}{Example}[section]

\newtheorem{remark}[theorem]{Remark}

\numberwithin{equation}{section}

\newcommand{\BB}{\mathbb{B}}
\newcommand{\EE}{\mathbb{E}}

\newcommand{\vol}{\textnormal{vol}}

\newcommand{\poly}{\textnormal{poly}}

\newcommand{\z}{\mathbf z}

\newcommand{\argmin}{\mathop{\rm argmin}}

\newcommand{\GCal}{\mathcal{G}}

\newcommand{\SCal}{\mathcal{S}}
\newcommand{\XCal}{\mathcal{X}}

\newcommand{\ICal}{\mathcal{I}}

\newcommand{\dom}{\textnormal{dom}}

\newcommand{\br}{\mathbb{R}}
\newcommand{\bs}{\mathbb{S}}

\newcommand{\ba}{\begin{array}}
\newcommand{\ea}{\end{array}}

\newcommand{\PCal}{\mathcal{P}}

\newcommand{\NCal}{\mathcal{N}}

\newcommand{\one}{\textbf{1}}

\newcommand{\mydefn}{:=}

\newcommand{\red}{\color{red}}

\begin{document}


\begin{center}

{\bf{\LARGE{Doubly Optimal No-Regret Online Learning \\ [.2cm] in Strongly Monotone Games with Bandit Feedback}}}

\vspace*{.2in}
{\large{ \begin{tabular}{c}
Wenjia Ba$^\diamond$ \and Tianyi Lin$^\dagger$ \and Jiawei Zhang$^\ddagger$ \and Zhengyuan Zhou$^\ddagger$
\end{tabular}
}}

\vspace*{.2in}

\begin{tabular}{c}
Sauder School of Business, University of British Columbia$^\diamond$ \\
Department of Industrial Engineering and Operations Research (IEOR), Columbia University$^\dagger$ \\
Stern School of Business, New York University$^\ddagger$
\end{tabular}

\vspace*{.2in}

\today

\vspace*{.2in}

\begin{abstract}
We consider online no-regret learning in unknown games with bandit feedback, where each player can only observe its reward at each time -- determined by all players' current joint action -- rather than its gradient. We focus on the class of \textit{smooth and strongly monotone} games and study optimal no-regret learning therein. Leveraging self-concordant barrier functions, we first construct a new bandit learning algorithm and show that it achieves the single-agent optimal regret of $\tilde{\Theta}(n\sqrt{T})$ under smooth and strongly concave reward functions ($n \geq 1$ is the problem dimension). We then show that if each player applies this no-regret learning algorithm in strongly monotone games, the joint action converges in the \textit{last iterate} to the unique Nash equilibrium at a rate of $\tilde{\Theta}(nT^{-1/2})$. Prior to our work, the best-known convergence rate in the same class of games is $\tilde{O}(n^{2/3}T^{-1/3})$ (achieved by a different algorithm), thus leaving open the problem of optimal no-regret learning algorithms (since the known lower bound is $\Omega(nT^{-1/2})$). Our results thus settle this open problem and contribute to the broad landscape of bandit game-theoretical learning by identifying the first doubly optimal bandit learning algorithm, in that it achieves (up to log factors) both optimal regret in the single-agent learning and optimal last-iterate convergence rate in the multi-agent learning. We also present preliminary numerical results on several application problems to demonstrate the efficacy of our algorithm in terms of iteration count.
\end{abstract}

\end{center}

\section{Introduction}\label{sec:intro}
In multi-agent online learning~\citep{Cesa-2006-Prediction, Shoham-2008-Multiagent, Bucsoniu-2010-Multi}, a set of players are repeatedly making decisions and accumulating rewards over time, where each player's action impacts not only its own reward, but that of the others. However, the mechanism of this interaction -- the underlying game that specifies how a player's reward depends on the joint action of all -- is unknown to players, and players may not even be aware that there is such a game. As such, from each player's own perspective, it is simply engaged in an online decision-making process, where the environment consists of all other players who are simultaneously making such sequential decisions, which are of consequence to all players.

In the past two decades, the above problem has actively engaged researchers from two fields: machine learning (and online learning in particular), which aims to develop single-agent online learning algorithms that are no-regret in an arbitrarily time-varying and/or adversarial environment~\citep{Blum-1998-Online, Shalev-2012-Online, Arora-2012-Multiplicative, Hazan-2016-Introduction}; and game theory, which aims to develop (ideally distributed) algorithms (see~\citet{Fudenberg-1998-Theory} and references therein) that efficiently compute a Nash equilibrium (a joint optimal outcome where no one can do better by deviating unilaterally) for games with special structures\footnote{Computing a Nash equilibrium is in general computationally intractable: the problem is indeed PPAD-complete~\citep{Daskalakis-2009-Complexity, Chen-2009-Settling})}. Although these two research threads initially developed separately, they have subsequently merged and formed the core of multi-agent/game-theoretical online learning, whose main research agenda can be phrased as follows: \textit{Will joint no-regret learning lead to a Nash equilibrium}, thereby reaping both the transient benefits (conferred by low finite-time regret) and the long-run benefits (conferred by Nash equilibria)?

More specifically, through the online learning lens, the $i^\textnormal{th}$ player's reward function at $t$ -- viewed as a function solely of its own action -- is $u_i^t(\cdot)$, and it needs to select an action $x_i^t \in \XCal_i \subseteq \br^{n_i}$ before $u_i^t(\cdot)$ -- or other feedback associated with it -- is revealed. In this context, no-regret algorithms ensure that the difference between the cumulative performance of the best fixed action and that of the learning algorithm, a widely adopted metric known as regret ($\text{Reg}_T = \max_{x \in \XCal_i}  \sum_{t=1}^T (u_i^t(x) - u_i^t(x_i^t))$), grows sublinearly in $T$. This problem has been extensively studied; in particular, when gradient feedback is available -- $\nabla_{x_i} u_i^t(x_i^t)$ can be observed after $x_i^t$ is selected -- the minimax optimal regret is $\Theta(\sqrt{T})$ for concave $u_i^t(\cdot)$ and $\Theta(\log{T})$ for strongly concave $u_i^t(\cdot)$. Further, several algorithms have been developed that achieve these optimal regret bounds, including follow-the-regularized-leader (FTRL)~\citep{Kalai-2005-Efficient}, online gradient descent (OGD)~\citep{Zinkevich-2003-Online}, multiplicative/exponential weights (MW)~\citep{Arora-2012-Multiplicative} and online mirror descent (OMD)~\citep{Shalev-2006-Convex}. As these algorithms provide optimal regret bounds and hence naturally raise high expectations in terms of performance guarantees, a recent line of work has investigated when all players apply a no-regret learning algorithm, what the evolution of the joint action would be, and in particular, whether the joint action would converge \textit{in last iterate} to a Nash equilibrium (if one exists). 

These questions turn out to be difficult and had remained open until a few years ago, mainly because the traditional Nash-seeking algorithms in the economics literature are mostly either not no-regret or only exhibit convergence in time-average (ergodic convergence), or both. Despite the challenging landscape, in the past five years, affirmative answers have emerged from a fruitful line of work and new analysis tools developed therein, first on the qualitative last-iterate convergence~\citep{Balandat-2016-Minimizing, Zhou-2017-Countering, Zhou-2017-Mirror, Zhou-2018-Learning, Mertikopoulos-2019-Optimistic, Mertikopoulos-2019-Learning} in different classes of continuous games (such as the variationally stable games), and then on the quantitative last-iterate convergence rates in more structured games such as routing games~\citep{Krichene-2015-Convergence}, co-coercive games~\citep{Lin-2020-Finite}, highly smooth games~\citep{Golowich-2020-Tight} and strongly monotone games~\citep{Zhou-2021-Robust}. These foundations have been inspirational to machine learning researchers in recent years, where zero-sum games have found applications, e.g., generative adversarial networks (GANs)~\citep{Goodfellow-2014-Generative}; indeed, there is the strand of literature on the quantitative last-iterate convergence rates in structured zero-sum games, including zero-sum games~\citep{Liang-2019-Interaction, Hsieh-2019-Convergence, Golowich-2020-Last, Wei-2021-Linear},  infinite-horizon discounted two-player zero-sum Markov games~\citep{Wei-2021-Last} and zero-sum extensive-form games with perfect recall~\citep{Lee-2021-Last}. Among these works,~\citet{Zhou-2021-Robust} has recently shown that if each player applies a version of online gradient descent, then the expectation of squared Euclidean distance between the joint action in the last iterate and the unique Nash equilibrium in a strongly monotone game converges to 0 at an optimal rate of $O(1/T)$ when the gradient feedback is not necessarily perfect.

Despite this remarkably pioneering line of work, which has elegantly bridged no-regret learning with convergence to Nash in continuous games and thus renewed the excitement of game-theoretical learning, a gap still exists and limits its practical usage in many real-world application problems. Specifically, in the multi-agent learning setting, a player is rarely able to observe gradient feedback. Instead, in most cases, only bandit feedback is available: each player observes only its own reward after choosing an action each time (rather than the gradient at the chosen action). For instance, in Cournot competition, each firm only observes the resulting market price, but not other firms production levels and hence can only recover its own profit (rather than its profit gradient); in Kelly auctions, each bidder can only observe the
share of the resources it wins from the exchange, rather than other bidders' bids, and hence can only compute its own gain; in pricing, each retailer only observes its own revenue, without the knowledge of other retailers' revenue. This consideration of practical feasible algorithms then brings us to a more challenging and less explored desideratum: if each player applies a no-regret bandit learning algorithm, would the joint action still converge to the Nash equilibrium? At what rate and in what class of games? 

\subsection{Related Work}
To appreciate the difficulty and the broad scope of this research agenda, we start by describing the existing related literature. First of all, we note that single-agent bandit learning algorithms -- and their theoretical regret characterizations -- are not as well-developed as their gradient counterparts. More specifically,~\citet{Kleinberg-2004-Nearly} and~\citet{Flaxman-2005-Online} provided the first bandit learning algorithm (with continuous action) -- known as FKM -- that achieved the regret bound of $O(\sqrt{n}T^{3/4})$ for Lipschitz and concave reward functions. However, it was unclear whether $O(\sqrt{n}T^{3/4})$ is optimal. Subsequently,~\citet{Saha-2011-Improved} developed a barrier-based bandit learning algorithm and established a $\tilde{O}(n^{2/3}T^{2/3})$ regret bound for smooth and concave reward functions, a result that has further been improved to $\tilde{O}(\sqrt{n}T^{5/8})$~\citep{Dekel-2015-Bandit} via a variant on the algorithm and new analysis. More recently, progress has been made on developing bandit learning algorithms that achieve minimax-optimal regret. In particular,~\citet{Bubeck-2015-Bandit} and~\citet{Bubeck-2016-Multi} provided the elegant non-constructive arguments showing that the minimax regret bound of $\tilde{\Theta}(\sqrt{T})$ (only considering the dependence on $T$) in one and high dimensions are achievable respectively, without providing any algorithm. Later,~\citet{Bubeck-2017-Kernel} developed a kernel method based bandit learning algorithm which attains the regret bound of $\tilde{O}(\poly(n)\sqrt{T})$. Independently,~\citet{Hazan-2016-Optimal} developed an ellipsoid method based bandit learning algorithm that also achieves the regret bound of $\tilde{O}(\poly(n)\sqrt{T})$. 

For strongly concave reward functions,~\citet{Agarwal-2010-Optimal} showed that the FKM algorithm achieves an improved regret bound of $\tilde{O}(n^{2/3}T^{2/3})$. For smooth and strongly concave reward functions,~\citet{Hazan-2014-Bandit} demonstrated that another variant of the barrier based bandit learning algorithm given in~\citet{Saha-2011-Improved} achieves the minimax-optimal regret of $\tilde{O}(n\sqrt{T})$ that matches the lower bound of $\Omega(n\sqrt{T})$ derived in~\citet{Shamir-2013-Complexity} up to log factors. For an overview of the relevant theory and applications, we refer to the recent surveys~\citep{Bubeck-2012-Regret, Lattimore-2020-Bandit} and references therein. 

However, much remains unknown in understanding the convergence of these no-regret bandit learning algorithms to Nash equilibria.~\citet{Bervoets-2020-Learning} developed a specialized distributed reward-based algorithm that asymptotically converges to the unique Nash equilibrium in the class of strictly monotone games. However, the algorithm is not known to be no-regret and no rate is given.~\citet{Heliou-2020-Gradient} considered a variant of the FKM algorithm and showed that it is no regret even under delays; further, provided the delays are not too large, the induced joint action would converge to the unique Nash equilibrium in the strictly monotone games (again without rates). At this writing, the most relevant and the-state-of-the-art result on this topic is presented in~\citet{Bravo-2018-Bandit}: if each player applies the FKM algorithm in strongly monotone games, then the last-iterate convergence to a unique Nash equilibrium is guaranteed at a rate of $\tilde{O}(n^{2/3}T^{-1/3})$. Per~\citet{Bravo-2018-Bandit}, the analysis itself is unlikely to be improved to yield any tighter rate. However, a sizable gap still exists between this bound and the best known lower bound given in~\citet{Shamir-2013-Complexity}, which established that in optimization problems with smooth and strongly concave objectives (which is a one-player Nash-seeking problem), no algorithm that uses only bandit feedback (i.e. zeroth-order oracle) can compute the unique optimal solution at a rate faster than $\Omega(nT^{-1/2})$. Consequently, it remains unknown as to whether other algorithms can improve the rate of $\tilde{O}(n^{2/3}T^{-1/3})$ as well as what the true optimal convergence rate is. In particular, since the lower bound of $\Omega(nT^{-1/2})$ is established for the special case of optimization problems, it is plausible that in the multi-agent setting -- where a natural potential function in optimization does not exist -- the problem is inherently more difficult, and hence the convergence might be intrinsically slower. Further, note that the lower bound in~\citet{Shamir-2013-Complexity} is established against the class of all bandit learning algorithms, not necessarily no-regret; essentially, a priori, that could mean a larger lower bound when the algorithms are further restricted to be no-regret. As such, it has been a challenging open problem to close the gap.

\subsection{Our Contributions} 
We tackle the problem of no-regret learning in strongly monotone games with bandit feedback and settle the above open problem by establishing that the convergence rate of $\tilde{O}(nT^{-1/2})$ -- and hence minimax optimal (up to log factors) -- is achievable. More specifically, we start by
studying (in Section~\ref{sec:SA}) single-agent learning with bandit feedback -- in particular with smooth and strongly concave reward functions -- and develop a mirror descent variant of the barrier-based family of bandit learning algorithms. We establish that the algorithm achieves the minimax optimal regret bound of $\tilde{\Theta}(n\sqrt{T})$ where $n \geq 1$ is the problem dimension. As such, our algorithm outperforms the FKM algorithm for the same setting in terms of regret\footnote{~\citet{Agarwal-2010-Optimal} also showed that the FKM algorithm achieves a regret bound of $\tilde{O}(n\sqrt{T})$ but for the special setting where the action set $\XCal = \br^n$. In general, the regret bound is $\tilde{O}(n^{2/3}T^{2/3})$ as mentioned before.}. 

Extending to multi-agent learning in strongly monotone games, we show that if all players employ this optimal no-regret learning algorithm (see Algorithm~\ref{Alg:MA}), the joint action converges in the last iterate to the unique Nash equilibrium at a rate of $\tilde{O}(nT^{-1/2})$. As such, we provide the first bandit learning algorithm (with continuous action) that is doubly optimal (up to log factors): it achieves optimal regret in single-agent settings under smooth and strongly concave reward functions and optimal convergence rate to Nash in multi-agent settings under smooth and strongly monotone games. We also provide numerical results on Cournot competition and Kelly auction in Section~\ref{sec:experiments}. These results demonstrate that our algorithm outperforms the multi-agent FKM algorithm in terms of iteration count. Missing proofs and additional experimental results can be found in the appendix. 

\section{Single-Agent Learning with Bandit Feedback}\label{sec:SA}
In this section, we provide a simple single-agent bandit learning algorithm that one player could employ to increase her individual reward in an online manner and prove that the algorithm achieves the near-optimal regret minimization property for bandit strongly concave optimization\footnote{This setting is the same as bandit convex optimization in the literature and we consider maximizing concave reward functions instead of minimizing convex loss functions.}.

In our setting, an adversary first chooses a sequence of $\beta$-strongly concave reward functions $u^1, u^2, \ldots, u^T: \XCal \mapsto \br$, (formally, $u^t(x) - u^t(x') - (x - x')^\top \nabla u^t(x) \leq -\frac{\beta}{2}\|x' - x\|^2$ for all $x, x' \in \XCal$), where $\XCal$ is a closed, convex and compact set. At each round $t = 1, 2, \ldots, T$, a (possibly randomized\footnote{Randomization plays an important role in the online learning literature. For example, the follow-the-leader (FTL) algorithm does not attain any non-trivial regret guarantee for linear reward functions (in the worst case it can be $\Omega(T)$ if the reward functions are chosen in an adversarial manner). However,~\citet{Hannan-1957-Approximation} proposed a randomized variant of FTL, called follow-the-perturbed-leader (FTPL), which could attain an optimal regret of $O(\sqrt{T})$ for linear reward functions over the simplex set.}) decision maker has to choose a point $x^t \in \XCal$ and it will incur a reward of $u^t(x^t)$ after committing to her decision. Her expected reward (where the expectation is taken with respect to her random choice) is $\EE[\sum_{t=1}^T u^t(x^t)]$ and the corresponding notion of \textit{regret} is defined by $\text{Reg}_T = \max_{x \in \XCal} \sum_{t=1}^T u^t(x) - \EE[\sum_{t=1}^T u^t(x^t)]$. In such bandit setting, it is worth mentioning that the feedback is limited to the value of reward function at the point that she has chosen, i.e., $u^t(x^t)$.

In what follows, we present the individual components used in our algorithm and then summarize the full scheme in Algorithm~\ref{Alg:SA} and the regret minimization property in Theorem~\ref{Thm:regret-optimal}. 

\subsection{Self-Concordant Barrier Function} 
Most of the existing bandit learning algorithms are developed based on an online mirror descent (OMD) framework~\citep{Cesa-2006-Prediction} and a self-concordant barrier function. Note that using the self-concordant barrier function as the regularizer in the OMD framework serves as a key ingredient in regret-optimal bandit algorithms when the reward function is either linear~\citep{Abernethy-2008-Competing} or smooth and strongly concave~\citep{Hazan-2014-Bandit}. For the sake of completeness, we provide a brief overview of self-concordant barrier functions and refer to~\citet{Nesterov-1994-Interior} for more details.
\begin{definition}
A function $R: \textnormal{int}(\XCal) \mapsto \br$ is a $\nu$-self concordant barrier for a closed convex set $\XCal \subseteq \br^n$, where $\textnormal{int}(\XCal)$ is an interior of $\XCal$, if (i) $R$ is three times continuously differentiable; (ii) $R(x) \rightarrow \infty$ if $x \rightarrow \partial \XCal$, where $ \partial \XCal$ is a boundary of $\XCal$; (iii) for $\forall x \in \textnormal{int}(\XCal)$ and $\forall h \in \br^n$, we have $|\nabla^3 R(x)[h, h, h]| \leq 2(h^\top \nabla^2 R(x) h)^{3/2}$ and $|\nabla R(x)^\top h| \leq \sqrt{\nu}(h^\top \nabla^2 R(x) h)^{1/2}$ where $\nabla^3 R(x)[h_1, h_2, h_3] = \tfrac{\partial^3}{\partial t_1 \partial t_2 \partial t_3}R(x+t_1 h_1+t_2 h_2 + t_3h_3)\vert_{t_1=t_2=t_3=0}$.  
\end{definition}
Resembling the existing bandit learning algorithms~\citep{Abernethy-2008-Competing, Saha-2011-Improved, Hazan-2014-Bandit, Dekel-2015-Bandit}, our algorithm requires a $\nu$-self-concordant barrier function for the set $\XCal$. However, this does not weaken the applicability of our algorithm; indeed, it is well known that any convex and compact set in $\br^n$ admits a non-degenerate $\nu$-self-concordant barrier function with $\nu = O(n)$~\citep{Nesterov-1994-Interior}, and such barrier can be efficiently represented and evaluated for numerous choices of $\XCal$ in real application problems. For example, the function $-\log(b-a^\top x)$ is a $1$-self-concordant barrier function for the set $\XCal = \{x: a^\top x \leq b\}$ and there exist computationally tractable $n$-self-concordant barrier functions for a $n$-dimensional simplex. For a $n$-dimensional ball, the function $-\log(1-\|x\|^2)$ is a $1$-self-concordant barrier function. 

The above definition is only given for the sake of completeness and our analysis relies on some useful facts about self-concordant barrier functions. In particular, the Hessian of a self-concordant barrier function $R$ can induce a local norm for any given point $x \in \textnormal{int}(\XCal)$; that is, $\|h\|_x = \sqrt{h^\top \nabla^2 R(x) h}$ and $\|h\|_{x, \star} = \sqrt{h^\top(\nabla^2 R(x))^{-1}h}$ for all $h \in \br^n$. It is clear that the non-degeneracy of $R$ guarantees that both $\|\cdot\|_x$ and $\|\cdot\|_{x, \star}$ are well defined.

The first important notion is the so-called \textit{Dikin ellipsoid}: $W(x)=\{x' \in \br^n: \|x'-x\|_x \leq 1\}$, which is defined at any $x \in \textnormal{int}(\XCal)$. The following lemma summarizes some nontrivial facts (see~\citet[Theorem~2.1.1]{Nesterov-1994-Interior} for a proof): 
\begin{lemma}\label{Lemma:SCB-Dikin}
Let $W(x)$ be the Dikin ellipsoid at any $x \in \textnormal{int}(\XCal)$, the following statements hold true: 
\begin{enumerate}
\item $W(x) \subseteq \XCal$ for every $x \in \textnormal{int}(\XCal)$;
\item For $\forall x' \in W(x)$, we have $(1-\|x'-x\|_x)\nabla^2 R(x) \preceq \nabla^2 R(x') \preceq (1-\|x'-x\|_x)^{-2}\nabla^2 R(x)$.
\end{enumerate}
\end{lemma}
\begin{remark}
Lemma~\ref{Lemma:SCB-Dikin} is crucial to the scheme of our algorithm. Indeed, we let the player choose her next action using $\hat{x}^t \leftarrow x^t + A^t z^t$ where $A_t \preceq (\nabla^2 R(x^t))^{-1/2}$, $z^t \sim \bs^n$ and $x^t \in \XCal$ is the previous iterate. Since $\|\hat{x}^t - x^t\|_{x^t} \leq \|z^t\| \leq 1$, we have $\hat{x}^t \in W(x^t)$. This together with Lemma~\ref{Lemma:SCB-Dikin} will guarantee that $\hat{x}^t \in \XCal$ and further allow us to get a reward $\hat{u}^t \leftarrow u^t(\hat{x}^t)$; see Algorithm~\ref{Alg:SA} for the details. 
\end{remark}
Further, we define the Minkowski function~\citep[Page~34]{Nesterov-1994-Interior} on $\XCal$ (which is parametrized by a point $x$) as $\pi_x(y) = \inf\{t \geq 0: x + \frac{1}{t}(y - x) \in \XCal\}$. Accordingly, the scaled version of $\XCal$ is given by 
\begin{equation*}
\XCal_\epsilon = \left\{x \in \br^n: \pi_{\bar{x}}(x) \leq \frac{1}{1+\epsilon}\right\}, \quad \textnormal{for all } \epsilon \in (0, 1]. 
\end{equation*}
A point $\bar{x}$ is a ``center" of $\XCal$ satisfying that $\bar{x} = \argmin_{x \in \XCal} R(x)$ where $R$ is a $\nu$-self-concordant barrier function for $\XCal$. The following lemma shows that $R$ is rather flat around the points that are far from the boundary (see~\citet[Proposition 2.3.2 and 2.3.3]{Nesterov-1994-Interior}): 
\begin{lemma}\label{Lemma:SCB-upper-bound}
Suppose that $\XCal$ is a closed, convex and compact set, $R$ is a $\nu$-self-concordant barrier function for $\XCal$ and $\bar{x}=\argmin_{x \in \XCal} R(x)$ is a center. Then, we have $R(x) - R(\bar{x}) \leq \nu\log(\frac{1}{1-\pi_{\bar{x}}(x)})$. For any $\epsilon \in (0, 1]$ and $x \in \XCal_\epsilon$, we have $\pi_{\bar{x}}(x) \leq \frac{1}{1+\epsilon}$ and $R(x) - R(\bar{x}) \leq \nu \log(1+\frac{1}{\epsilon})$. 
\end{lemma}
Finally, we recall that the \textit{Newton decrement} for a self-concordant function $g$ is defined as $\lambda(x, g) \mydefn \|\nabla g(x)\|_{x, \star} = \|(\nabla^2 g(x))^{-1}\nabla g(x)\|_x$ where $\|\cdot\|_x$ and $\|\cdot\|_{x, \star}$ are a local norm and its dual respectively. This can be used to roughly measure how far a point is from a global optimum of $g$. Formally, we summarize the results in the following lemma (see~\citet{Nemirovski-2008-Interior} for a proof): 
\begin{lemma}\label{Lemma:SC-key-estimate}
For any self-concordant function $g$ and let $\lambda(x, g) \leq \tfrac{1}{2}$, we have $\|x - \argmin_{x' \in \XCal} g(x')\|_x \leq 2\lambda(x, g)$, where $\|\cdot\|_x$ is the local norm given by $\|h\|_x \mydefn \sqrt{h^\top\nabla^2 g(x)h}$. 
\end{lemma}

\subsection{Single-Shot Ellipsoidal Estimator} 
It was~\citet{Flaxman-2005-Online} that introduced a single-shot spherical estimator in the literature and combined it with online mirror descent. In particular, let $u: \br^n \mapsto \br$ be a function, $\delta > 0$ and $z \sim \bs^n$ where $\bs^n$ is a $n$-dimensional unit sphere, a single-shot spherical estimator is defined by 
\begin{equation}\label{def:spherical}
\hat{v} = \frac{n}{\delta} \cdot u(x + \delta z)z. 
\end{equation}
This estimator is an unbiased prediction for the gradient of a smoothed version; that is, $\EE[\hat{v}] = \nabla\hat{u}(x)$ where $\hat{u}(x) = \EE_{w \sim \BB^n}[u(x+\delta w)]$ where $\BB^n$ is a $n$-dimensional unit ball. As $\delta \rightarrow 0^+$, the bias caused by the difference between $u$ and $\hat{u}$ vanishes while the variability of $\hat{v}$ explodes. This manifestation of the bias-variance dilemma plays a key role in designing bandit learning algorithms and a single-shot spherical estimator is known to be suboptimal in terms of bias-variance trade-off and hence regret minimization. This gap is finally closed by using a more sophisticated single-shot ellipsoidal estimator based on the self-concordant barrier function for $\XCal$~\citep{Saha-2011-Improved, Hazan-2014-Bandit}. Comparing to the spherical estimator in Eq.~\eqref{def:spherical}, they proposed to sample the direction w.r.t. an ellipsoid and give another unbiased gradient estimate of the scaled smooth version. In particular, a single-shot ellipsoidal estimator for an invertible matrix $A$ is defined by  
\begin{equation}\label{def:ellipsoidal}
\hat{v} = n \cdot u(x + Az) A^{-1}z.  
\end{equation}
The following lemma is a modification of~\citet[Corollary~6 and Lemma~7]{Hazan-2014-Bandit}. 
\begin{lemma}\label{Lemma:ellipsodial}
Suppose that $u$ is a concave function and $A \in \br^{n \times n}$ is an invertible matrix, we define the smoothed version of $u$ with respect to $A$ by $\hat{u}(x) = \EE_{w \sim \BB^n}[u(x+ Aw)]$ where $\BB^n$ is a $n$-dimensional unit ball. Then, the following statements hold true: 
\begin{enumerate}
\item $\nabla \hat{u}(x) = \EE_{z \sim \bs^n}[n \cdot u(x+Az)A^{-1}z]$ where $\bs^n$ is a $n$-dimensional unit sphere.
\item If $u$ is $\beta$-strongly concave, we have $\hat{u}$ is also $\beta$-strongly concave.  
\item If $\nabla u$ is $\ell$-Lipschitz continuous and we let $\sigma_{\max}(A)$ be the largest eigenvalue of $A$, we have $0 \leq u(x) - \hat{u}(x) \leq \frac{1}{2}\ell(\sigma_{\max}(A))^2$. 
\end{enumerate}
\end{lemma}
\begin{remark}
We see from Lemma~\ref{Lemma:ellipsodial} that $\EE[\hat{v}] = \nabla \hat{u}(x)$ where $\hat{v}$ is defined in Eq.~\eqref{def:ellipsoidal} and $\hat{u}(x) = \EE_{w \sim \BB^n}[u(x+Aw)]$. In our algorithm, we set $A^t$ using a self-concordant barrier function $R$ for $\XCal$ and perform the shrinking sampling~\citep{Hazan-2014-Bandit}. This is the key to a better bias-variance trade-off than that achieved by the classical spherical estimators; see Algorithm~\ref{Alg:SA} for the details. 
\end{remark}

\subsection{Mirror Descent} 
Combining the idea of mirror descent\footnote{In reward maximization, we shall use mirror ascent instead of mirror descent since the player seeks to maximize her reward (as opposed to minimizing her loss). Nonetheless, we keep the original term ``descent" throughout this paper because, despite the role reversal, it is the standard name associated with the method.}~\citep{Nemirovski-1983-Problem}, our algorithm generates a new \textit{feasible} point $x^+$ by taking a ``mirror step" from a starting point $x$ along an ``approximate gradient" direction $\hat{v}$. By abuse of notation, we let $R: \textnormal{int}(\XCal) \mapsto \br$ be a strictly convex \textit{distance-generating} (or \textit{regularizer}) function, i.e., $R(tx + (1-t)x') \leq tR(x) + (1-t)R(x')$ with equality if and only if $x = x'$ for all $x, x' \in \XCal$ and all $t \in [0, 1]$. We also assume that $R$ is continuously differentiable, i.e., $\nabla R: \textnormal{int}(\XCal) \mapsto \br^n$ is continuous. This leads to a \textit{Bregman divergence} on $\XCal$ via the relation 
\begin{equation}\label{def:Bregman}
D_R(x', x) = R(x') - R(x) - \langle\nabla R(x), x' - x\rangle, 
\end{equation}
for all $x' \in \XCal$ and $x \in \textnormal{int}(\XCal)$, which might fail to be symmetric and/or satisfy the triangle inequality. Nevertheless, $D_R(x', x) \geq 0$ with equality if and only if $x' = x$, so the asymptotic convergence of $x^t$ to $p$ can be checked by showing that $D_R(p, x^t) \rightarrow 0$. We continue with some basic relations connecting the Bregman divergence relative to a target point before and after a prox-map. The key ingredient is ``three-point identity" which generalizes the law of cosines, and which is widely used in the literature~\citep{Chen-1993-Convergence, Beck-2003-Mirror}.  
\begin{lemma}\label{Lemma:three-point}
Let $R$ be a regularizer on $\XCal$. For all $p \in \XCal$ and all $x, x' \in \dom(R)$, we have $D_R(p, x') = D_R(p, x) + D_R(x, x') + \langle\nabla R(x') - \nabla R(x), x - p\rangle$. 
\end{lemma}
\begin{remark}
In our algorithm, we set $R$ as a self-concordant barrier function for $\XCal$; indeed, the function $D_R$ is well defined since the self-concordant function for a convex set is known as a strictly convex distance-generating function. 
\end{remark}
The key notion for Bregman divergence $D_R(x', x)$ is the induced prox-mapping given by  
\begin{equation}\label{def:prox-map-old}
\PCal_R(x, \hat{v}, \eta) = \argmin_{x' \in \XCal} \ \eta\langle\hat{v}, x - x'\rangle + D_R(x', x), \quad \textnormal{ for all } x \in \textnormal{int}(\XCal) \textnormal{ and all } \hat{v} \in \br^n, 
\end{equation}
which reflects the intuition behind mirror descent. Indeed, it starts with $x \in \textnormal{int}(\XCal)$ and generates a \textit{feasible} point $x^+ = \PCal_R(x, \hat{v}, \eta)$ using the vector $\hat{v} \in \br^n$. In our algorithm, we propose to use the specific prox-mapping as follows, 
\begin{equation}\label{def:prox-map}
\PCal_R(x, \hat{v}, \eta) = \argmin_{x' \in \XCal} \ \eta\langle\hat{v}, x - x'\rangle + \tfrac{\eta \beta(t+1)}{2}\|x - x'\|^2 + D_R(x', x), 
\end{equation}
for all $x \in \textnormal{int}(\XCal)$ and all $\hat{v} \in \br^n$, where $R$ is a self-concordant barrier function for $\XCal$. We remark that the prox-mapping in Eq.~\eqref{def:prox-map} explicitly incorporates the problem structure information: the first term $\tfrac{\eta\beta(t+1)}{2}\|x - x'\|^2$ is squared Euclidean norm with the coefficient proportional to strong concavity parameter $\beta > 0$ and the second term is a Bregman divergence with respect to a self-concordant barrier function $R$. Such proximal mapping is crucial to establish the last-iterate convergence rate for our algorithm in the multi-agent setting; see Remark~\ref{remark:last-iterate}. In our algorithm, we employ the recursion $x^{t+1} \leftarrow \PCal_R(x^t, \hat{v}^t, \eta_t)$ in which $\eta_t > 0$ is the chosen step-size and $\hat{v}^t$ is a single-shot ellipsoidal estimator as mentioned before.  

The use of a mirror descent framework with specific prox-mapping in Eq.~\eqref{def:prox-map} is the main difference between our algorithm and the one developed in~\citet{Hazan-2014-Bandit}; indeed, their algorithm is developed based on the follow-the-regularized-leader (FTRL) framework with some different prox-mappings. Note that such seemingly minor modification is crucial to last-iterate convergence analysis when we extend our algorithm (cf. Algorithm~\ref{Alg:SA}) to the multi-agent setting; indeed, the multi-agent mirror descent can achieve the last-iterate convergence rate~\citep{Zhou-2021-Robust} but we are not aware of such results for multi-agent FTRL. In Section~\ref{sec:MA}, we will explain this point in more detail from a technical point of view.

\subsection{Algorithm and Regret Bound}
We consider the single-agent setting in which the adversary is limited to choosing smooth and strongly concave functions $u^1, u^2, \ldots, u^T: \XCal \mapsto \br$ and $\XCal$ is a closed, convex and compact set. With the components presented before, we summarize the scheme of our algorithm in Algorithm~\ref{Alg:SA}. 
\begin{algorithm}[!t]
\caption{Mirror Descent Self-Concordant Barrier Bandit Learning}\label{Alg:SA}
\begin{algorithmic}[1]
\STATE \textbf{Input:} step size $\eta_t > 0$, module $\beta > 0$ and barrier $R: \textnormal{int}(\XCal) \mapsto \br$.  
\STATE \textbf{Initialization:} $x^1 = \argmin_{x \in \XCal} R(x)$.  
\FOR{$t = 1, 2, \ldots$} 
\STATE set $A^t \leftarrow (\nabla^2 R(x^t) + \eta_t \beta(t+1) I_n)^{-1/2}$. \hfill \textsc{\# scaling matrix}
\STATE draw $z^t \sim \bs^n$. \hfill \textsc{\# perturbation direction} 
\STATE play $\hat{x}^t \leftarrow x^t + A^t z^t$. \hfill \textsc{\# choose action} 
\STATE receive $\hat{u}^t \leftarrow u^t(\hat{x}^t)$. \hfill \textsc{\# get reward} 
\STATE set $\hat{v}^t \leftarrow n \cdot \hat{u}^t(A^t)^{-1}z^t$. \hfill \textsc{\# estimate gradient} 
\STATE update $x^{t+1} \leftarrow \PCal_R(x^t, \hat{v}^t, \eta_t)$. \hfill \textsc{\# update iterate} 
\ENDFOR
\end{algorithmic}
\end{algorithm}
To facilitate the readers, we present the detailed scheme of~\citet[Algorithm~1]{Hazan-2014-Bandit} using our notations such that we can see the difference. More specifically, their algorithm generates the estimators $\{\hat{v}^\tau\}_{1 \leq \tau \leq t}$ using the same strategy as ours (cf. Step 4 - Step 8) but performs the following FTRL update at each iteration: 
\begin{equation*}
x^{t+1} \leftarrow \argmin_{x \in \br^n} \ \left\{\sum_{\tau = 1}^t \langle\hat{v}^\tau, x^\tau - x\rangle + \tfrac{\beta}{2}\|x - x^\tau\|^2\right\} + \tfrac{1}{\eta_t}R(x), 
\end{equation*}
Since $\XCal \neq \br^n$, the iterates generated by~\citet[Algorithm~1]{Hazan-2014-Bandit} are different from the ones generated by our algorithm and it remains unclear if the multi-agent variant of~\citet[Algorithm~1]{Hazan-2014-Bandit} can achieve the last-iterate convergence rate.

The following theorem shows that Algorithm~\ref{Alg:SA} can achieve the regret bound of $\tilde{O}(n\sqrt{T})$, which has matched the lower bound established in~\citet{Shamir-2013-Complexity}. 
\begin{theorem}\label{Thm:regret-optimal}
Suppose that the adversary is limited to choosing smooth and $\beta$-strongly concave functions $u^1, u^2, \ldots, u^T: \XCal \mapsto \br$. Each function $u^t$ is Lipschitz continuous and satisfies that $|u^t(x)| \leq L$ for all $x \in \XCal$. If $T \geq 1$ is fixed and the player employs Algorithm~\ref{Alg:SA} with the stepsize choice of $\eta_t = \frac{1}{2nL\sqrt{T}}$, we have 
\begin{equation*}
\textnormal{Reg}_T = \max_{x \in \XCal} \left\{\sum_{t=1}^T u^t(x) - \EE\left[\sum_{t=1}^T u^t(\hat{x}^t)\right]\right\} = \tilde{O}(n\sqrt{T}). 
\end{equation*}
\end{theorem}
\begin{remark}
Theorem~\ref{Thm:regret-optimal} shows that Algorithm~\ref{Alg:SA} is a near-regret-optimal bandit learning algorithm when the adversary is limited to choosing smooth and strongly concave functions. Note that the regret is better than the best-known one of FKM in the same setting~\citep{Agarwal-2010-Optimal}, demonstrating another way where our result improves upon~\citet{Bravo-2018-Bandit} which uses FKM. However, the subproblem solving in Eq.~\eqref{def:prox-map} might become more difficult due to the use of a self-concordant barrier function; indeed, we do not have a closed-form solution even if the subproblem is unconstrained and convex. In contrast, the general prox-mapping can admit a closed-form solution for the proper regularizers, e.g., the quadratic regularizer for a box set and entropy regularizer for a simplex set. This sheds light on the trade-off between the computational efficiency of subproblem solving and regret minimization. This also highlights the practical advantage of multi-agent FKM over Algorithm~\ref{Alg:SA} in terms of computational time for certain application problems. 
\end{remark}
To prove Theorem~\ref{Thm:regret-optimal}, we present our descent inequality for the iterates generated by Algorithm~\ref{Alg:SA}. 
\begin{lemma}\label{Lemma:SA-main}
Suppose that the iterate $\{x^t\}_{t \geq 1}$ is generated by Algorithm~\ref{Alg:SA} and let each function $u^t$ satisfy that $|u^t(x)| \leq L$ for all $x \in \XCal$ and $0 < \eta_t \leq \frac{1}{2nL}$, we have
\begin{equation*}
D_R(p, x^{t+1}) + \tfrac{\eta_t \beta(t+1)}{2}\|x^{t+1} - p\|^2 \leq D_R(p, x^t) + \tfrac{\eta_t\beta(t+1)}{2}\|x^t - p\|^2 + 2\eta_t^2\|A^t\hat{v}^t\|^2 + \eta_t \langle\hat{v}^t, x^t - p\rangle.
\end{equation*}
where $p \in \XCal$ and the sequence $\{\eta_t\}_{t \geq 1}$ is assumed to be non-increasing. 
\end{lemma}
See the proofs of Lemma~\ref{Lemma:SA-main} and Theorem~\ref{Thm:regret-optimal} in Appendix~\ref{app:SA-main} and~\ref{app:regret-optimal}. 

\section{Multi-Agent Learning with Bandit Feedback}\label{sec:MA}
In this section, we consider multi-agent learning with bandit feedback and characterize the behavior of the system when each player employs the multi-agent variant of Algorithm~\ref{Alg:SA}. We first present some basic definitions and notations, discuss a few important examples of strongly monotone games and finally provide our multi-agent mirror descent self-concordant barrier bandit learning algorithm.   

\subsection{Basic Definition and Notation}
We focus on games played by a set of players $i \in \NCal = \{1, 2, \ldots, N\}$. At each round, each player selects an \textit{action} $x_i$ from a convex and compact subset $\XCal_i$ of a finite-dimensional vector space $\br^{n_i}$ and their reward is determined by the profile $x = (x_i; x_{-i}) = (x_1, x_2, \ldots, x_N)$ of the action of all players; subsequently, each player receives the \textit{reward}, and repeats this process. We denote $\|\cdot\|$ as the Euclidean norm (in the corresponding vector space): other norms can be easily accommodated in our framework (and different $\XCal_i$'s can have different norms), although we will not bother with these since we do not play with (and benefit from) complicated geometries.
\begin{definition}\label{def:CG}
A smooth and concave game is a tuple $\GCal = (\NCal, \XCal = \prod_{i=1}^N \XCal_i, \{u_i\}_{i=1}^N)$, where $\NCal$ is a set of $N$ players, $\XCal_i$ is a convex and compact subset of finite-dimensional vector space $\br^{n_i}$ representing the $i^\textnormal{th}$ player's action space, and $u_i: \XCal \mapsto \br$ is the $i^\textnormal{th}$ player's reward function satisfying that: (i) $u_i(x_i; x_{-i})$ is continuous in $x$ and concave in $x_i$ for all fixed $x_{-i}$; (ii) $u_i(x_i; x_{-i})$ is continuously differentiable in $x_i$ and the individual reward gradient $v_i(x)=\nabla_i u_i(x_i; x_{-i})$ is Lipschitz continuous.
\end{definition}
A commonly used solution concept for non-cooperative games is \textit{Nash equilibrium} (NE). For smooth and concave games considered in this paper, we are interested in the pure-strategy Nash equilibria. Indeed, for finite games, the mixed strategy NE is a probability distribution over the pure strategy NE. Our setting assumes continuous and convex action sets, where each action already lives in a continuum, and pursuing pure-strategy Nash equilibria is sufficient. 
\begin{definition}
An action profile $x^\star \in \XCal$ is called a (pure-strategy) Nash equilibrium of a game $G$ if it is resilient to unilateral deviations; that is, $u_i(x_i^\star; x_{-i}^\star) \geq u_i(x_i; x_{-i}^\star)$ for all $x_i \in \XCal_i$ and $i \in \NCal$. 
\end{definition}
It is known that every smooth and concave game admits at least one Nash equilibrium when all action sets are compact~\citep{Debreu-1952-Social} and Nash equilibria admit a variational characterization. We summarize this result in the following proposition. 
\begin{proposition}\label{prop:CG}
In a smooth and concave game $\GCal$, the profile $x^\star \in \XCal$ is a Nash equilibrium if and only if $(x_i-x_i^\star)^\top v_i(x^\star) \leq 0$ for all $x_i \in \XCal_i$ and $i \in \NCal$.
\end{proposition}
Proposition~\ref{prop:CG} shows that Nash equilibria of a smooth and concave game can be characterized as the solution set of variational inequality (VI), so the existence results follow from the seminar results {\red in} the VI literature~\citep{Facchinei-2007-Finite}. 

\subsection{Strongly Monotone Games}
The study of (strongly) monotone games dates to~\citet{Rosen-1965-Existence}, with many subsequent developments; see, e.g.,~\citet{Facchinei-2007-Finite}. Specifically,~\citet{Rosen-1965-Existence} considered a class of games that satisfy the \textit{diagonal strict concavity} (DSC) condition and prove that they admit a unique Nash equilibrium. Further work in this vein appeared in~\citet{Sandholm-2015-Population} and~\citet{Sorin-2016-Finite}, where games that satisfy DSC are referred to as ``contractive" and ``dissipative". These conditions are equivalent to strict monotonicity in the context of convex analysis~\citep{Bauschke-2011-Convex}. To avoid confusion, we provide the formal definition of strongly monotone games.   
\begin{definition}\label{def:SMS}
A smooth and concave game $\GCal$ is called $(\beta, \{\lambda_i\}_{i \in \NCal})$-strongly monotone\footnote{In general, we assume that $\beta > 0$ and $\lambda_i \geq 0$ for all $i \in \NCal$. } if we have $\sum_{i \in \NCal} \lambda_i\langle x'_i-x_i, v_i(x') - v_i(x)\rangle \leq -\beta\|x - x'\|^2$ for any $x, x' \in \XCal$. 
\end{definition}
The notion of (strong) monotonicity, which will play a crucial role in the subsequent analysis, is not necessarily theoretically artificial but encompasses a very rich class of games. We present four typical examples which satisfy the conditions in Definition~\ref{def:SMS} (see Appendix~\ref{app:examples} for proof details and the selection of $\beta$ and $\{\lambda_i\}_{i \in \NCal}$ in our examples). 
\begin{example}[Cournot Competition]\label{Example:CC}
In the Cournot oligopoly model, there is a finite set $\NCal = \{1, 2, \ldots, N\}$ of firms, each supplying the market with a quantity $x_i \in [0, B_i]$ of some good (or service) up to the firm's production capacity, given here by a positive scalar $B_i > 0$. This good is then priced as a decreasing function $P(x)$ of the total supply to the market, as determined by each firm's production; for concreteness, we focus on the linear model $P(x) = a - b\sum_{i \in \NCal} x_i$ where $a$ and $b$ are positive constants. In this model, the reward of the $i^\textnormal{th}$ firm (considered here as a player) is given by
\begin{equation*}
u_i(x) = x_i P(x) - c_i x_i, 
\end{equation*}
where $c_i$ denotes the marginal production cost of the $i^\textnormal{th}$ firm, i.e., as the income obtained by producing $x_i$ units of the good in question minus the corresponding production cost. Letting $\XCal_i = [0, B_i]$ denote the space of possible production values for each firm, we can show that the game $\GCal(\NCal, \XCal, u)$ is $(\beta, \{\lambda_i\}_{i \in \NCal})$-strongly monotone with $\beta = b$ and $\lambda_i = 1$ for all $i \in \NCal$. Note that $a$ and $b$ are unknown to all the players and each player only knows its own $c_i$ and observes the market price, from which only the bandit feedback of its reward function can be recovered. 
\end{example}
\begin{example}[Strongly Concave Potential Game]\label{Example:SCPG}
A game $\GCal = (\NCal, \XCal = \prod_{i=1}^N \XCal_i, \{u_i\}_{i=1}^N)$ is called a potential game~\citep{Monderer-1996-Potential, Sandholm-2001-Potential} if there exists a potential function $f: \XCal \mapsto \br$ such that 
\begin{equation*}
u_i(x_i; x_{-i}) - u_i(\tilde{x}_i; x_{-i}) = f(x_i; x_{-i}) - f(\tilde{x}_i; x_{-i}), 
\end{equation*}
for all $i \in \NCal$, all $x \in \XCal$ and all $\tilde{x}_i \in \XCal_i$. If the potential function $f$ is $\beta$-strongly concave, we can show from some basic results in the context of convex analysis~\citep{Bauschke-2011-Convex} that the game $\GCal(\NCal, \XCal, u)$ is $(\beta, \{\lambda_i\}_{i \in \NCal})$-strongly monotone with $\lambda_i = 1$ for all $i \in \NCal$. 
\end{example}
\begin{example}[Kelly Auctions] \label{Example:KA}
Consider a service provider with a number of splittable resources $s \in \SCal = \{1, 2, \ldots, S\}$ (representing, e.g., bandwidth, server time, ad space on a website, etc.). These resources can be leased to a set of $N$ bidders (players) who can place monetary bids $x_{is} \geq 0$ for the utilization of each resource $s \in \SCal$ up to each player's total budget $B_i$, i.e., $\sum_{s \in \SCal} x_{is} \leq B_i$. A popular and widely used mechanism to allocate resources in this case is the Kelly mechanism~\citep{Kelly-1998-Rate} whereby resources are allocated proportionally to each player's bid, i.e., the $i^\textnormal{th}$ player gets
\begin{equation*}
\rho_{is} = \frac{q_s x_{is}}{d_s + \sum_{j \in \NCal} x_{js}} 
\end{equation*}
units of the $s^\textnormal{th}$ resource (in the above, $q_s$ denotes the available units of said resource and $d_s \geq 0$ is the ``entry barrier" for bidding on it). A simple model for the reward of the $i^\textnormal{th}$ player is given by
\begin{equation*}
u_i(x_i; x_{-i}) = \sum_{s \in \SCal} (g_i\rho_{is} - x_{is}), 
\end{equation*}
where $g_i$ denotes the player's marginal gain from acquiring a unit slice of resources. If we write $\XCal_i = \{x_i \in \br_+^S: \sum_{s \in \SCal} x_{is} \leq b_i\}$ for the entire space of possible bids of the $i^\textnormal{th}$ player on the set of resources $\SCal$, we can show that the game $\GCal(\NCal, \XCal, u)$ is $(\beta, \{\lambda_i\}_{i \in \NCal})$-strongly monotone with $\beta = \frac{\min_{s \in \SCal}\{q_s d_s\}}{(\sum_{s \in \SCal} d_s + \sum_{i \in \NCal} B_i)^3}$ and $\lambda_i = \frac{1}{g_i}$ for all $i \in \NCal$. Although $q_s$ and $d_s$ are known to all bidders, each bidder's reward gradient is not computable: other bidders' bids are not observable and hence only $\rho_{is}$ is observable. Consequently, this is a strongly monotone game with only bandit feedback available.
\end{example}

\begin{example}[Retailer Pricing Games]\label{Example:pricing}
We start with the single-retailer multi-product pricing setting. A retailer (with unlimited inventory) sells $n$ products over a horizon $T$ and makes pricing decisions $p^t = (p_1^t, p_2^t, \ldots, p_n^t)$ at each period $t$ in order to maximize overall revenue across all products. Those products may be substitutes or complements to each other, and consequently, each product's price not only affects its own demand, but also that of the other products. We focus on the linear model $D^t: \br^d \rightarrow \br^d$, where each component $D_i^t(p) = \sum_{j=1}^n a_{ij}^t p_j + b_i^t$ gives the demand for the $i^\textnormal{th}$ product under the decision $p$. Note that $a_{ij}^t <0 $ if the $j^\textnormal{th}$ product is a complement to the $i^\textnormal{th}$ product and $a_{ij}^t > 0$ if the $j^\textnormal{th}$ product is a substitute to the $i^\textnormal{th}$ product; further, $b_i^t > 0$ and all prices are bounded (i.e. lie in some sets $[0, B_i]$). By definition, we have $a_{ii}^t < 0$. The total revenue is $u^t(p) = \sum_{i=1}^n p_i^t D_i^t(p)$, which is $\beta$-strongly concave in $p$ if $A^t + (A^t)^\top \leq -\beta I_n$ for all $t \geq 1$, where $A^t = (a_{ij}^t)_{1 \leq i, j \leq n} \in \br^{n \times n}$ is not necessarily symmetrical. Note that the retailer observes the realized demand $D^t(p^t)$ at the end of period $t$, which only gives the bandit feedback. Interestingly, in this single-setting, without any additional assumptions,
applying our algorithm yields $\tilde{O}(n\sqrt{T})$ regret bound, which appears to be novel. 

We also consider the multi-retailer pricing setting, and for simplicity, we focus on the single product setting. Consider the set of retailers $\NCal = \{1, 2, \ldots, N\}$, each selling a different product that may either be a substitute or complement to products sold by other retailers. The $i^\textnormal{th}$ retailer's action is the price $p_i$ for its product, which is assumed to lie in some closed interval. For the $i^\textnormal{th}$ retailer, the demand for its product depends on the joint price vector of all retailers (as similar products act as substitutes and hence influence this product's demand): $D_i(p) =  \sum_{j=1}^N a_{ij}p_j + b_i$. Here, $a_{ij} <0 $ if the $j^\textnormal{th}$ retailer's product is a complement to the $i^\textnormal{th}$ retailer's product and $a_{ij} > 0$ if the $j^\textnormal{th}$ retailer's product is a substitute to the $i^\textnormal{th}$ retailer's product; further, $b_i > 0$ and $a_{ii} < 0$ for all $i$. The $i^\textnormal{th}$ retailer's reward function is its revenue $u_i(p_i, p_{-i}) = p_i D_i(p)$. Th resulting game is $(\beta, \{\lambda_i\}_{i \in \NCal})$-strongly monotone with $\lambda_i = 1$ for all $i \in \NCal$ if $A + A^\top \leq -\beta I_n$ where the matrix $A = (a_{ij})_{1 \leq i, j \leq n} \in \br^{n \times n}$ is not necessarily symmetrical. 

Note that in both cases, a sufficient condition for strong convexity and/or strong monotonicity is when the matrix $A$ is strictly diagonally dominant. This condition has a interpretable meaning: each product's price impact on its own demand is larger than that of all other products combined. Unless products have very little differentiation, this assumption easily holds in practice.
\end{example}
There are also many other application problems that can be cast into the framework of strongly monotone games~\citep{Orda-1993-Competitive, Cesa-2006-Prediction, Sandholm-2015-Population, Sorin-2016-Finite, Mertikopoulos-2017-Distributed}. Typical examples include strongly-convex-strongly-concave zero-sum two-player games, congestion games~\citep{Mertikopoulos-2019-Learning}, wireless network games~\citep{Weeraddana-2012-Weighted, Tan-2014-Wireless, Zhou-2021-Robust} and various online decision-making problems~\citep{Cesa-2006-Prediction}. With that being said, the notion of strong monotonicity has been widely used for characterizing real application problems. To make this argument more convincing, we consider the zero-sum two-player strongly monotone games which are simply strongly-convex-strongly-concave min-max problems. This problem class is well motivated and has covered many application problems arising from operations research~\citep{Facchinei-2007-Finite}. A class of online decision-making problems with strongly convex loss functions is also well motivated and well studied in the literature (see, e.g.,~\citet{Hazan-2016-Introduction}).

From an economic point of view, one appealing feature of strongly monotone games is that the last-iterate convergence can be achieved by some learning algorithms~\citep{Zhou-2021-Robust}, which is more natural than the time-average-iterate convergence~\citep{Fudenberg-1998-Theory, Cesa-2006-Prediction}. The finite-time convergence rate is derived in terms of the distance between $x^t$ and $x^\star$, where $x^t$ is the realized action and $x^\star$ is the unique Nash equilibrium (under convex and compact action sets). In view of all this (and unless explicitly stated otherwise), we will focus on strongly monotone games throughout this paper. 

\subsection{Algorithm}
In multi-agent learning with bandit feedback, at each round $t = 1, 2, \ldots, T$, every (randomized) decision maker $i \in \NCal$ selects an action $x_i^t \in \XCal_i$. The reward $u_i(x^t)$ is realized after all decision makers have chosen their actions. In multi-agent bandit learning, the feedback is limited to the value of the reward function at the point that each player has chosen, i.e., $u_i(x^t)$. 

We propose a simple multi-agent bandit learning algorithm where each player chooses her action using Algorithm~\ref{Alg:SA}. In other words, it is a straightforward extension of Algorithm~\ref{Alg:SA} from single-agent setting to multi-agent setting. Notably, our new algorithm differs from multi-agent FKM (cf.~\citet[Algorithm~1]{Bravo-2018-Bandit}) in two aspects: ellipsoidal SPSA estimator v.s. spherical SPSA estimator and self-concordant Bregman divergence v.s. general Bregman divergence. We discuss these two crucial components before summarizing the scheme of our multi-agent bandit learning algorithm.

\paragraph{Single-shot ellipsoidal SPSA estimator.}~\citet{Bravo-2018-Bandit} have recently extended the spherical estimator in Eq.~\eqref{def:spherical} from single-agent setting to multi-agent setting; indeed, their approach posits that all the players use a simultaneous perturbation stochastic approximation (SPSA) approach~\citep{Spall-1997-One} to estimate their individual reward gradients $\hat{v}_i$ based off a single reward function evaluation. More specifically, let $u_i: \br^n \mapsto \br$ be a reward function, $\delta > 0$ and the query directions $z_i \sim \bs^{n_i}$ be drawn independently across players where $\bs^{n_i}$ is a $n_i$-dimensional unit sphere, a single-shot spherical SPSA estimator is defined by 
\begin{equation}\label{def:spherical-SPSA}
\hat{v}_i = \frac{n_i}{\delta} \cdot u_i(x_i + \delta z_i; x_{-i} + \delta z_{-i})z_i.  
\end{equation}
This estimator is an unbiased prediction for a partial gradient of a smoothed version of $u_i$; indeed, we have $\EE[\hat{v}_i] = \nabla_i \widehat{u}_i(x)$ where $\hat{u}_i(x) = \EE_{w_i \sim \BB^{n_i}}\EE_{z_{-i} \sim \Pi_{j \neq i} \bs^{n_j}}[u_i(x_i+\delta w_i; x_{-i}+\delta z_{-i})]$ where $\BB^{n_i}$ is a $n_i$-dimensional unit ball. We can easily see the bias-variance dilemma here: as $\delta \rightarrow 0^+$, $\hat{v}_i$ becomes more accurate since $\|\nabla_i \hat{u}_i(x) - \nabla_i u_i(x)\| = O(\delta)$, while the variability of $\hat{v}_i$ grows unbounded since the second moment of $\hat{v}_i$ grows as $O(\delta^{-2})$. By carefully choosing $\delta > 0$,~\citet{Bravo-2018-Bandit} provided the best-known last-iterate convergence rate of $\tilde{O}(\sqrt[3]{\frac{n^2}{T}})$ which almost matches the lower bound established in~\citet{Shamir-2013-Complexity}. However, a gap still remains and they believe it can be closed by using a more sophisticated single-shot estimator.

We now provide a single-shot ellipsoidal SPSA estimator by extending the estimator in Eq.~\eqref{def:ellipsoidal} to multi-agent setting. More specifically, we let $\hat{x}_i = x_i + A_i z_i$ and define
\begin{equation}\label{def:ellipsoidal-SPSA}
\hat{v}_i = n_i \cdot u_i(\hat{x}_i; \hat{x}_{-i})(A_i)^{-1} z_i.  
\end{equation}
The following lemma summarizes some results for the estimator in Eq.~\eqref{def:ellipsoidal-SPSA} and we provide the detailed proof in Appendix~\ref{app:ellipsodial-SPSA} for the sake of completeness.  
\begin{lemma}\label{Lemma:ellipsodial-SPSA}
Suppose that $u_i$ is a concave function and $A_i \in \br^{n_i \times n_i}$ is an invertible matrix for each $i \in \NCal$, we define the smoothed version of $u_i$ with respect to $A_i$ by $\hat{u}_i(x) = \EE_{w_i \sim \BB^{n_i}}\EE_{\z_{-i} \sim \Pi_{j \neq i} \bs^{n_j}}[u_i(x_i+A_iw_i; \hat{x}_{-i})]$ where $\bs^{n_i}$ is a $n_i$-dimensional unit sphere, $\BB^{n_i}$ is a $n_i$-dimensional unit ball and $\hat{x}_i = x_i + A_i z_i$ for all $i \in \NCal$. Then, the following statements hold true: 
\begin{enumerate}
\item $\nabla_i \hat{u}_i(x) = \EE[n_i \cdot u_i(\hat{x}_i; \hat{x}_{-i})(A_i)^{-1}z_i \mid x_1, x_2, \ldots, x_N]$. 
\item If $v_i$ is $\ell_i$-Lipschitz continuous and we let $\sigma_{\max}(A)$ be the largest eigenvalue of $A$, we have $\|\nabla_i \hat{u}_i(x) - \nabla_i u_i(x)\| \leq \ell_i\sqrt{\sum_{j \in \NCal} (\sigma_{\max}(A_j))^2}$. 
\end{enumerate}
\end{lemma}
\begin{remark}
Lemma~\ref{Lemma:ellipsodial-SPSA} generalizes~\citet[Lemma~C.1]{Bravo-2018-Bandit} which is proved for a single-shot spherical SPSA estimator. It shows that $\nabla_i \hat{u}_i(x) = \EE[\hat{v}_i \mid x]$ in which $\hat{u}_i$ is a smoothed version of $u_i$. In our algorithm, the construction of $A_i$ relies on the self-concordant barrier function $R_i$ for $\XCal_i$. 
\end{remark}
\paragraph{Mirror descent.} The general prox-mapping in Eq.~\eqref{def:prox-map-old} has been used in~\citet{Bravo-2018-Bandit} to construct their multi-agent bandit learning algorithm, which achieves the suboptimal regret minimization and last-iterate convergence. In contrast, our multi-agent bandit learning algorithm updates $x_i^{t+1}$ using the specific prox-mapping in Eq.~\eqref{def:prox-map} that can exploit the problem structure and the rule is given by
\begin{equation}\label{def:prox-map-MA}
\PCal_{R_i}(x_i, \hat{v}_i, \eta, \lambda_i, \beta) = \argmin_{x'_i \in \XCal_i} \ \eta\langle\hat{v}_i, x_i - x'_i\rangle + \tfrac{\eta \beta(t+1)}{2\lambda_i}\|x_i - x'_i\|^2 + D_{R_i}(x'_i, x_i),  
\end{equation}
for all $x_i \in \textnormal{int}(\XCal_i)$ and all $\hat{v}_i \in \br^{n_i}$, where $R_i$ is a self-concordant barrier function for $\XCal_i$. In our algorithm, we employ the recursion $x_i^{t+1} \leftarrow \PCal_{R_i}(x_i^t, \hat{v}_i^t, \eta_t)$  in which $\eta_t > 0$ is the chosen step-size and $\hat{v}_i^t$ is a single-shot ellipsoidal SPSA estimator as mentioned before.  
\begin{remark}\label{remark:last-iterate}
The use of mirror descent framework with specific prox-mapping in Eq.~\eqref{def:prox-map-MA} is crucial to last-iterate convergence analysis in the next subsection. Roughly speaking, it allows us to directly bound the term $\sum_{i \in \NCal} \|x_i^T - x_i^\star\|^2$ and leads to a near-optimal last-iterate convergence rate. Moreover, we can see from the proof of Theorem~\ref{Thm:last-iterate-perfect} that the time-dependent coefficient of quadratic term in Eq.~\eqref{def:prox-map-MA} plays a key role in bounding $\sum_{i \in \mathcal{N}} \|x_i^T - x_i^\star\|^2$ by $\tilde{O}(\sqrt{\frac{n^2}{T}})$. This is consistent with the fact that only multi-agent mirror descent is proven to achieve last-iterate convergence rate in strongly monotone games~\citep{Zhou-2021-Robust} while no such results hold for multi-agent FTRL.
\end{remark}
\begin{algorithm}[!t]
\caption{Multi-Agent Mirror Descent Self-Concordant Barrier Bandit Learning}\label{Alg:MA}
\begin{algorithmic}[1]
\STATE \textbf{Input:} step size $\eta_t > 0$, weight $\lambda_i > 0$, module $\beta > 0$, and barrier $R_i: \textnormal{int}(\XCal_i) \mapsto \br$.  
\STATE \textbf{Initialization:} $x_i^1 = \argmin_{x_i \in \XCal_i} R_i(x_i)$.  
\FOR{$t = 1, 2, \ldots$} 
\FOR{$i \in \NCal$} 
\STATE set $A_i^t \leftarrow (\nabla^2 R_i(x_i^t) + \tfrac{\eta_t \beta(t+1)}{\lambda_i} I_{n_i})^{-1/2}$.  \hfill \textsc{\# scaling matrix} 
\STATE draw $z_i^t \sim \bs^{n_i}$.  \hfill \textsc{\# perturbation direction} 
\STATE play $\hat{x}_i^t \leftarrow x_i^t + A_i^t z_i^t$. \hfill \textsc{\# choose action}
\ENDFOR 
\STATE receive $\hat{u}_i^t \leftarrow u_i(\hat{x}^t)$ for all $i \in \NCal$. \hfill \textsc{\# get reward}
\FOR{$i \in \NCal$} 
\STATE set $\hat{v}_i^t \leftarrow n_i\hat{u}_i^t (A_i^t)^{-1} z_i^t$. \hfill \textsc{\# estimate gradient}
\STATE update $x_i^{t+1} \leftarrow \PCal_{R_i}(x_i^t, \hat{v}_i^t, \eta_t, \lambda_i, \beta)$. \hfill \textsc{\# update iterate}
\ENDFOR 
\ENDFOR
\end{algorithmic}
\end{algorithm}

\subsection{Finite-Time Last-Iterate Convergence Rate}
We consider the multi-agent setting in a strongly monotone game in which each player's reward function $u_i: \XCal \mapsto \br$ satisfies that $|u_i(x)| \leq L$ for all $x \in \XCal$ and $\XCal$ is a closed, convex and compact set. With the components presented before, we summarize the scheme of our algorithm in Algorithm~\ref{Alg:MA}. 

In addition to regret minimization, the convergence to Nash equilibria serves as an important criterion for measuring the performance of multi-agent learning algorithms. The following theorem shows that Algorithm~\ref{Alg:MA} can achieve the last-iterate convergence to a unique Nash equilibrium at a rate of $\tilde{O}(nT^{-1/2})$. This has improved the rate of $\tilde{O}(n^{2/3}T^{-1/3})$ achieved by multi-agent FKM~\citep{Bravo-2018-Bandit} and matched the lower bound established in~\citet{Shamir-2013-Complexity}. 

\begin{theorem}\label{Thm:last-iterate-perfect}
Suppose that $x^\star \in \XCal$ is a unique Nash equilibrium of a smooth and $(\beta, \{\lambda_i\}_{i \in \NCal})$-strongly monotone game. Each function $u_i$ satisfies that $|u_i(x)| \leq L$ for all $x \in \XCal$. If $T \geq 1$ is fixed and each player employs Algorithm~\ref{Alg:MA} with the stepsize choice of $\eta_t = \tfrac{1}{2nL\sqrt{t}}$, we have 
\begin{equation*}
\EE\left[\sum_{i \in \NCal}\|\hat{x}_i^T - x_i^\star\|^2\right] = \tilde{O}(nT^{-1/2}). 
\end{equation*}
\end{theorem}
\begin{remark}
Theorem~\ref{Thm:last-iterate-perfect} shows that Algorithm~\ref{Alg:MA} attains a near-optimal rate of last-iterate convergence in a smooth and strongly monotone game. It provides the first doubly optimal bandit learning algorithm, in that it achieves (up to log factors) both optimal regret in a single-agent setting and optimal last-iterate convergence rate in a multi-agent setting. In contrast, it remains unclear if the multi-agent FTRL~\citep{Hazan-2014-Bandit} can achieve the last-iterate convergence rate. 
\end{remark}
To prove Theorem~\ref{Thm:last-iterate-perfect}, we present our descent inequality for the iterates generated by Algorithm~\ref{Alg:MA}. 
\begin{lemma}\label{Lemma:MA-perfect-main}
Suppose that the iterate $\{x^t\}_{t \geq 1}$ is generated by Algorithm~\ref{Alg:MA} and let each function $u_i$ satisfy that $|u_i(x)| \leq L$ for all $x \in \XCal$ and $0 < \eta_t \leq \frac{1}{2nL}$, we have
\begin{eqnarray*}
\lefteqn{\sum_{i \in \NCal} \lambda_i D_{R_i}(p_i, x_i^{t+1}) + \tfrac{\eta_{t+1}\beta(t+1)}{2}\left(\sum_{i \in \NCal} \|x_i^{t+1} - p_i\|^2\right)} \\
& \leq & \sum_{i \in \NCal} \lambda_i D_{R_i}(p_i, x_i^t) + \tfrac{\eta_t\beta(t+1)}{2}\left(\sum_{i \in \NCal} \|x_i^t - p_i\|^2\right) + 2\eta_t^2\left(\sum_{i \in \NCal} \lambda_i \|A_i^t\hat{v}_i^t\|^2\right) + \eta_t\left(\sum_{i \in \NCal} \lambda_i \langle\hat{v}_i^t, x_i^t - p_i\rangle\right).
\end{eqnarray*}
where $p_i \in \XCal_i$ and the sequence $\{\eta_t\}_{t \geq 1}$ is assumed to be non-increasing. 
\end{lemma}
See the proofs of Lemma~\ref{Lemma:MA-perfect-main} and Theorem~\ref{Thm:last-iterate-perfect} in Appendix~\ref{app:MA-perfect-main} and~\ref{app:last-iterate-perfect}. We also provide the last-iterate convergence rate results under the imperfect bandit feedback in Appendix~\ref{app:imperfect}.  
\begin{remark}
Currently, there is no computational ``free lunch": there is a trade-off between the cost of solving a subproblem and the total iteration counts. Compared to the multi-agent FKM, our algorithm requires less iteration counts but more time to solve each subproblem in practice. One reason is that the FKM subproblem is classic and has received detailed treatment in the literature, where the contribution of~\citet{Condat-2016-Fast} is dedicated to solving this problem in a very satisfactory manner. On the contrary, the subproblem in our optimal-regret algorithm is new and arises from a design that primarily targets statistical efficiency (i.e., regret). As such, we do not yet know whether more computationally efficient methods and/or solvers exist for this subproblem. However, we argue that the inferior performance of our algorithm in terms of the solution time does not weaken our contribution from a theoretical viewpoint. Indeed, the goal of this paper is to design an \textit{optimal} no-regret learning algorithms in which the regret bound and the last-iterate convergence rate are only measure by \textit{iteration count}. Can we develop a practically fast algorithm with the same strong theoretical guarantee? We leave this for future work and believe this would be a worthwhile research problem that deserves a specialized treatment on its own.
\end{remark}
\begin{remark}
As we have mentioned before, one appealing feature of strongly monotone games is that the finite-time convergence rate can be derived in terms of $\sum_{i \in \NCal}\|x_i^t - x_i^\star\|^2$~\citep{Zhou-2021-Robust} (under convex and compact action sets). On the other hand, $x^t$ possibly converges to a limit cycle or repeatedly hit the boundary in monotone games~\citep{Daskalakis-2018-Training, Mertikopoulos-2018-Cycles} despite the time-average iterate $\frac{1}{t}(\sum_{k=1}^t x^k)$ converges. More recently,~\citet{Cai-2022-Finite} and~\citet{Gorbunov-2022-Last} analyzed the optimistic gradient method (which is a no-regret learning algorithm) in monotone games (under convex and compact action sets) and proved that last-iterate convergence to Nash equilibria at an optimal rate of $O(T^{-1/2})$ in terms of \textbf{a gap function}. However, we are not aware of any no-regret learning algorithm that can achieve the last-iterate convergence rate in terms of $\sum_{i \in \mathcal{N}} \|x_i^t - x_i^\star\|^2$ for \textit{general monotone games}, and it seems impossible to attain such a rate without the quadratic growth of strong monotonicity. Indeed,~\citet{Tseng-1995-Linear} proved that the extragradient (EG) method and other projection-type methods achieved the linear convergence in terms of $\sum_{i \in \mathcal{N}} \|x_i^t - x_i^\star\|^2$ for monotone games \textit{but} under the additional projection-type error bound condition (see Eq. (5) in his paper). In addition, the focus of our paper is to investigate the finite-time last-iterate convergence of no-regret \textit{bandit} learning algorithms while the EG method is gradient-based and has been shown not \textit{no-regret}~\citep[Appendix A.3]{Golowich-2020-Tight}. 

Moreover, analyzing the last-iterate convergence rate of no-regret bandit learning algorithms is completely out of reach of the existing techniques of~\citet{Cai-2022-Finite} and~\citet{Gorbunov-2022-Last}. Can we propose the no-regret bandit learning algorithms to handle monotone games and prove the last-iterate convergence rate in terms of a gap function? We leave the answers to future work.
\end{remark}

\section{Numerical Experiments}\label{sec:experiments}
In this section, we conduct the numerical experiments using Cournot competition and Kelly auction. The baseline approach is multi-agent FKM (cf.~\citet[Algorithm~1]{Bravo-2018-Bandit}) and the platform is MATLAB R2021b on a MacBook Pro with an Intel Core i9 2.4GHz and 16GB memory. The code is provided in \url{https://github.com/tyDLin/Doubly_Optimal_Bandits} for reproducibility. 
\begin{figure*}[!t]
\centering
\includegraphics[width=0.45\textwidth]{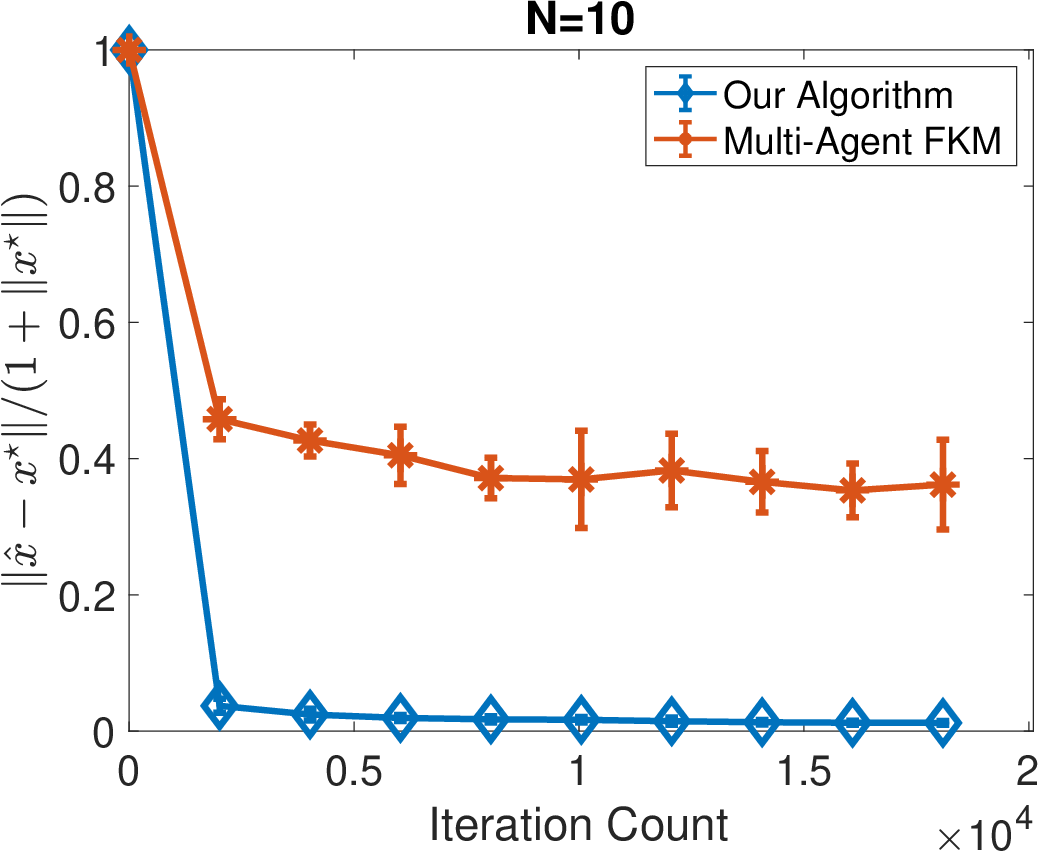}
\includegraphics[width=0.45\textwidth]{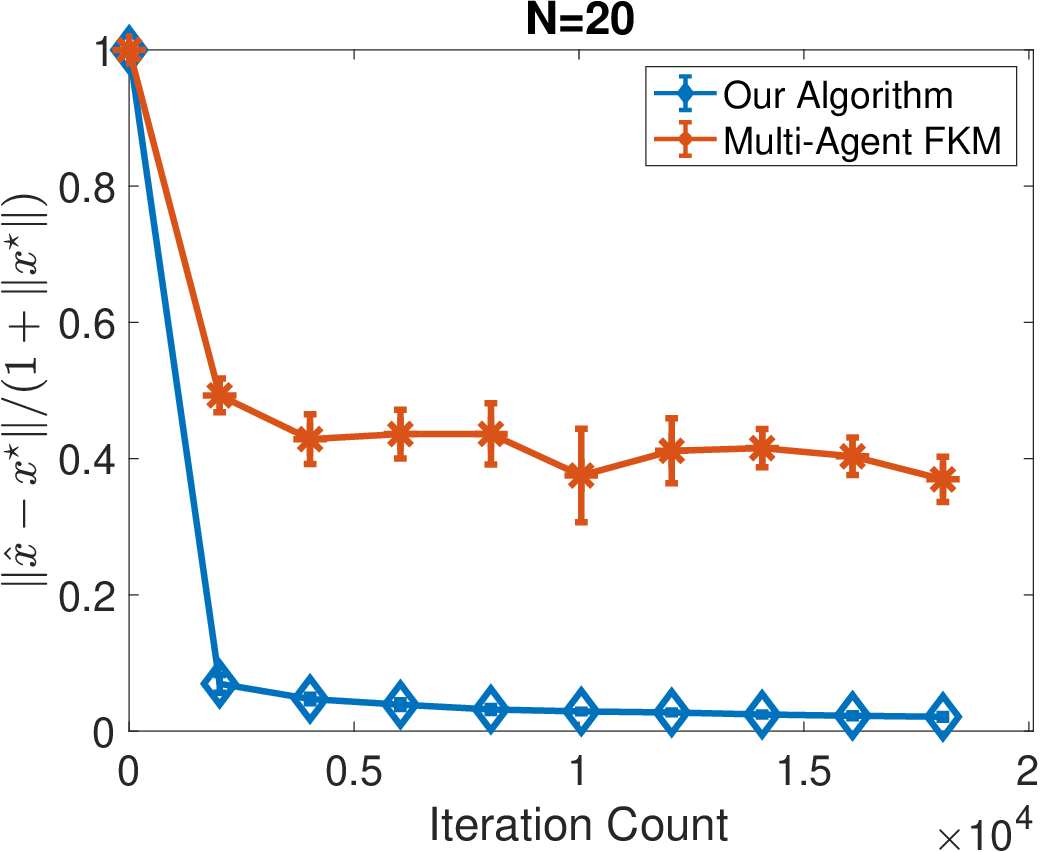}
\includegraphics[width=0.45\textwidth]{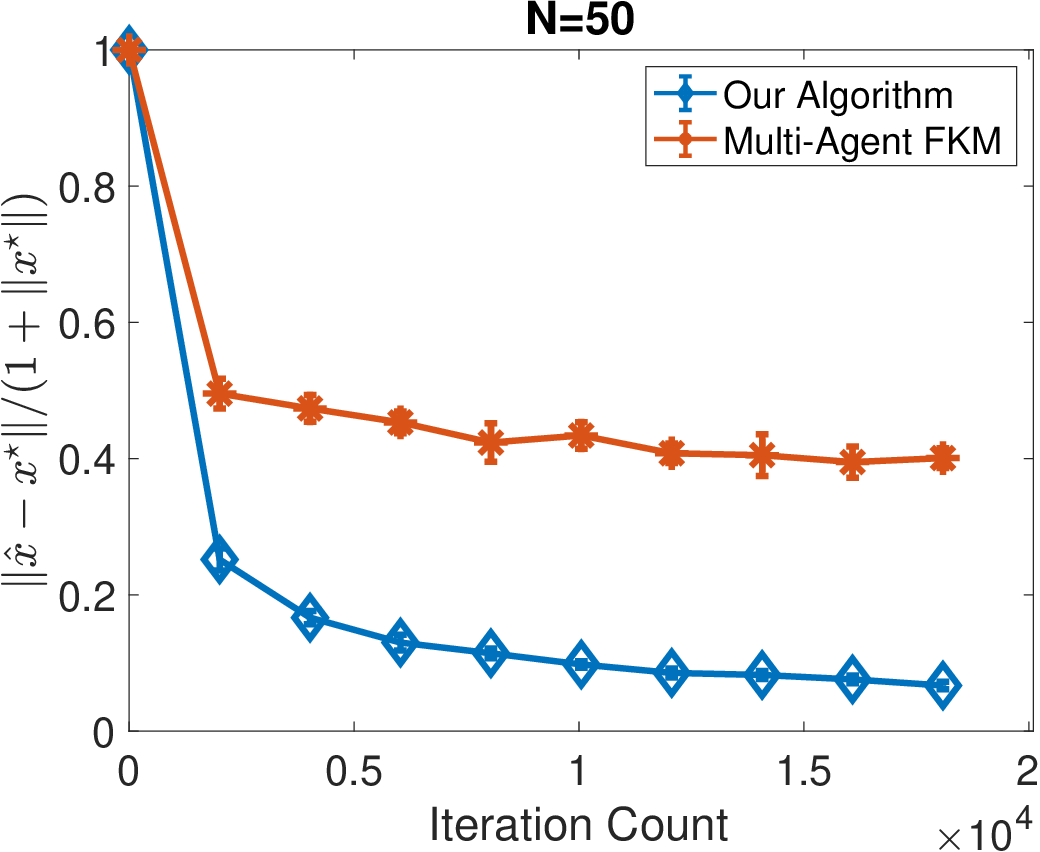}
\includegraphics[width=0.45\textwidth]{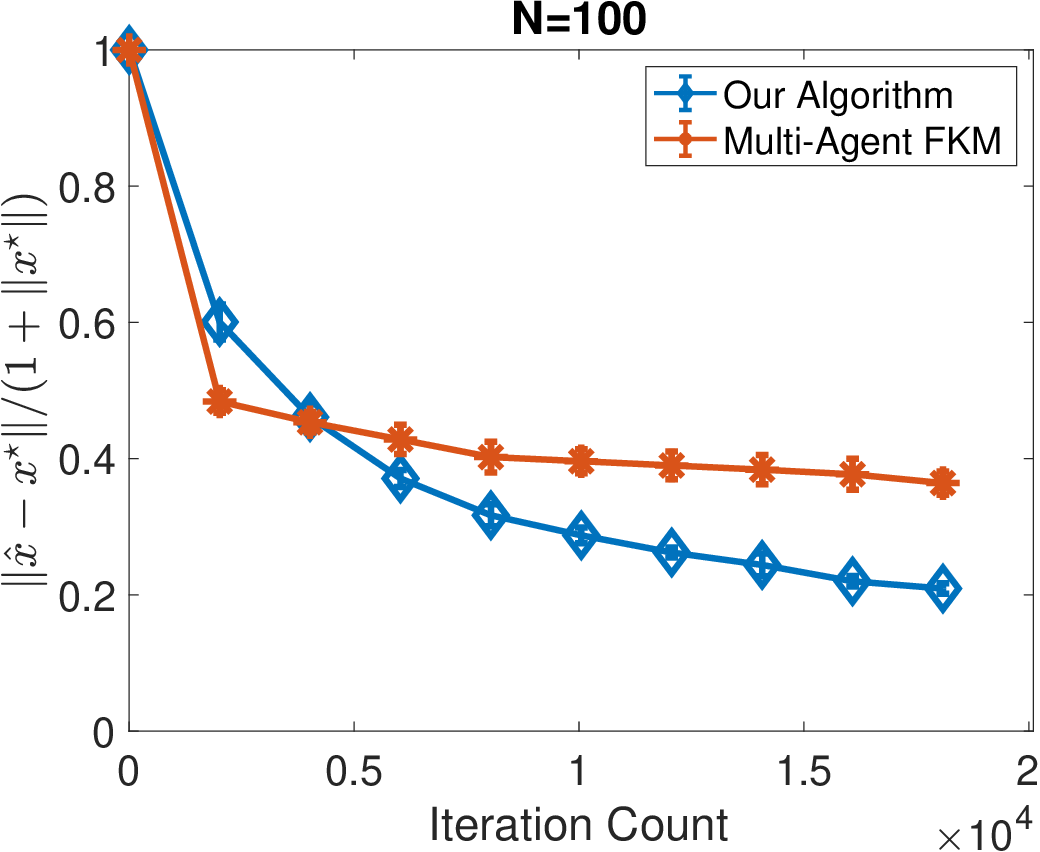}
\caption{Performance of both algorithms for solving Counot competition. The number of independent trials is 10 and $(a, b) = (10, 0.05)$ is fixed. The numerical results are presented as error v.s. iteration count.}\label{Fig:CC}
\end{figure*}
\subsection{Cournot Competition}
We fix $B_i = 1$ for all $i \in \NCal$ and evaluate the algorithms by varying $N \in \{10, 20, 50, 100\}$, $a \in \{10, 20\}$ and $b \in \{0.05, 0.1\}$. For each $i \in \NCal$, we draw $c_i \in \br$ independently from an uniform distribution on the interval $[0, 1]$. The game $\GCal$ is $(\beta, \{\lambda_i\}_{i \in \NCal})$-strongly monotone with $\beta = b$ and $\lambda_i = 1$ for all $i \in \NCal$.

For both Algorithm~\ref{Alg:MA} and multi-agent FKM, we consider the theoretically-correct choices of step sizes without fine-tuning. Indeed, we set $\lambda_i = 1$ for all $i \in \NCal$ and $\beta = b$ for Algorithm~\ref{Alg:MA} due to the structure of Cournot competition. Since $\XCal_i = [0, B_i]$, we set $R_i(x_i) = - \log(x_i) - \log(B_i - x_i)$. According to Theorem~\ref{Thm:last-iterate-perfect}, we set $\eta_t = \frac{1}{2N\sqrt{t}}$. For all $i \in \NCal$, we set $r_i = p_i = \frac{1}{2}B_i$ in multi-agent FKM. We also set $\delta_t = \min\{\min_{i \in \NCal} r_i, \frac{1}{t^{1/3}}\}$ and $\gamma_t = \tfrac{1}{3\beta t}$ according to~\citet[Theorem~5.2]{Bravo-2018-Bandit}. Moreover, it is well known that the Cournot competition is a strongly concave potential game and can be solved by minimizing a quadratic function as follows. 
\begin{equation}\label{prb:CC}
\min_{(x_1, x_2, \ldots, x_N) \in \br^N} \ f(x_1, x_2, \ldots, x_N) = \tfrac{b}{2}\left(\sum_{i \in \NCal} x_i\right)^2 + \tfrac{b}{2}\left(\sum_{i \in \NCal} x_i^2\right) + \sum_{i \in \NCal} c_i x_i - a\left(\sum_{i \in \NCal} x_i\right).  
\end{equation}
The evaluation metric is $\frac{\|\hat{x}^t - x^\star\|}{1 + \|x^\star\|}$ where $\hat{x}^t$ is generated by the algorithms and $x^\star$ is an approximate Nash equilibrium with high accuracy (we employ \textsc{quadprog} in Matlab for minimizing $f$ in Eq.~\eqref{prb:CC}). This point is a benchmark for evaluating the quality of the solution obtained by the algorithms.
\begin{table}
\centering\caption{The solution quality on Cournot competition after 20000 iterations.}\label{Tab:CC}
\begin{tabular}{|c||c|c|} \hline
$(N, a, b)$ 		& Multi-Agent FKM 					& Our Algorithm \\ \hhline{===} 
(10, 10, 0.05) 		& 3.3e-01	$\pm$ 4.5e-02 			& 1.2e-02	$\pm$ 3.6e-03 \\
(10, 10, 0.10) 		& 2.6e-01	$\pm$ 4.9e-02 			& 1.3e-02	$\pm$ 2.9e-03 \\
(10, 20, 0.05) 	& 4.5e-01 	$\pm$ 3.3e-02 		& 9.0e-03	$\pm$ 4.3e-03 \\ 
(10, 20, 0.10) 		& 3.4e-01	$\pm$ 4.7e-02 			& 7.7e-03 	$\pm$ 2.3e-03 \\ \hhline{===}   
(20, 10, 0.05) 	& 3.6e-01	$\pm$ 2.5e-02 			& 2.1e-02 	$\pm$ 1.4e-03 \\  
(20, 10, 0.10) 		& 2.7e-01 	$\pm$ 3.0e-02 		& 2.4e-02	$\pm$ 2.3e-03 \\ 
(20, 20, 0.05) 	& 4.9e-01 	$\pm$ 3.3e-02 		& 1.2e-02	$\pm$ 1.9e-03 \\ 
(20, 20, 0.10) 	& 3.8e-01	$\pm$ 2.1e-02 			& 1.2e-02	$\pm$ 2.0e-03 \\ \hhline{===}    
(50, 10, 0.05) 	& 3.8e-01	$\pm$ 2.4e-02 			& 7.0e-02	$\pm$ 6.5e-03 \\ 
(50, 10, 0.10) 		& 2.3e-01 $\pm$ 2.0e-02 			& 9.2e-02	$\pm$ 5.1e-03 \\
(50, 20, 0.05) 	& 5.0e-01 $\pm$ 3.3e-02 		& 2.8e-02	$\pm$ 1.5e-03 \\ 
(50, 20, 0.10) 	& 3.9e-01 $\pm$ 3.0e-02 		& 3.2e-02	$\pm$ 1.4e-03 \\ \hhline{===}    
(100, 10, 0.05) 	& 3.7e-01 	$\pm$ 2.3e-02 			& 2.0e-01 	$\pm$ 9.6e-03 \\
(100, 10, 0.10) 	& 1.2e-01 	$\pm$ 5.4e-03 		& 1.4e-01 	$\pm$ 1.6e-02 \\
(100, 20, 0.05) 	& 5.4e-01 $\pm$ 2.2e-02 			& 6.8e-02	$\pm$ 2.4e-03 \\
(100, 20, 0.10) 	& 3.7e-01 	$\pm$ 2.7e-02 			& 9.5e-02	$\pm$ 5.0e-03 \\ \hline
\end{tabular}
\end{table}

\paragraph{Experimental results.} Fixing $(a, b) = (10, 0.05)$, we investigate the convergence behavior of both algorithms with a varying number of players, i.e., $N \in \{10, 20, 50, 100\}$. Figure~\ref{Fig:CC} indicates that our algorithm outperforms multi-agent FKM as it exhibits a faster convergence to Nash equilibrium in terms of iteration count. Note that the reason why the multi-agent FKM can be better within the 1000 iterations when $N = 100$ is that its stepsize is \textit{reasonably larger} than that of our algorithm. Indeed, the stepsizes are $\frac{1}{2N\sqrt{t}}$ for our algorithm and $\frac{1}{3\beta t}$ for multi-agent FKM. If $N$ (the dimension) is large and $\beta$ (strongly monotone parameter) is estimated accurately, the stepsize of multi-agent FKM will be larger than that of our algorithm but not explode problematically within the 1000 iterations. This will guarantee that the multi-agent FKM outperforms our algorithm in practice. We also present the numerical results for all $(N, a, b)$ in Table~\ref{Tab:CC}. 
\begin{remark}
It is worth remarking that the multi-agent FKM can outperform our algorithm within the 1000 iterations when the dimension is \textit{large} and the estimation of strongly monotone parameters is \textit{accurate}. The Cournot competition has a relatively simple structure so that the estimation of $\beta$ is exact. Thus, the multi-agent FKM outperforms our algorithm within the 1000 iterations for the case of $N = 100$ (sufficiently large). In contrast, we can see that the multi-agent FKM does not perform well for solving Kelly auction \textit{even} with large dimension. This is because the complex structure of Kelly auction makes the accurate estimation of $\beta$ difficult. 
\end{remark}

\subsection{Kelly Auction}
We fix $B_i = 1$ for all $i \in \NCal$ and evaluate the algorithms by varying $N \in \{10, 20, 50, 100\}$, $S \in \{2, 5\}$ and $\bar{d} \in \{0.5, 1\}$. For each $i \in \NCal$, we draw $g_i \in \br$ independently of a uniform distribution at the interval $[0, 1]$. For each $s \in \SCal$, we draw $q_s$ independently from an uniform distribution on the interval $[0, 1]$ and $d_s$ independently from an uniform distribution on the interval $[0, \bar{d}]$. The game $\GCal$ is $(\beta, \{\lambda_i\}_{i \in \NCal})$-strongly monotone with $\beta = \frac{\min_{s \in \SCal}\{q_s d_s\}}{(\sum_{s \in \SCal} d_s + \sum_{i \in \NCal} B_i)^3}$ and $\lambda_i = \frac{1}{g_i}$ for all $i \in \NCal$. 
\begin{figure*}[!t]
\centering
\includegraphics[width=0.45\textwidth]{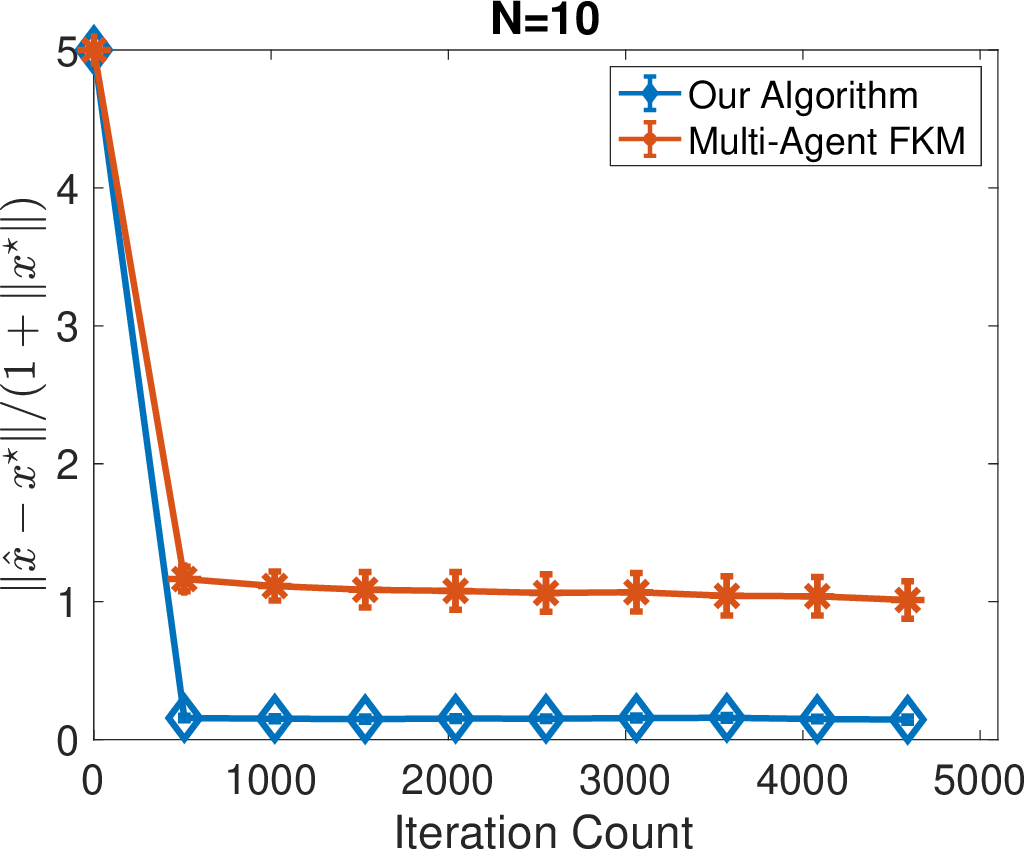}
\includegraphics[width=0.45\textwidth]{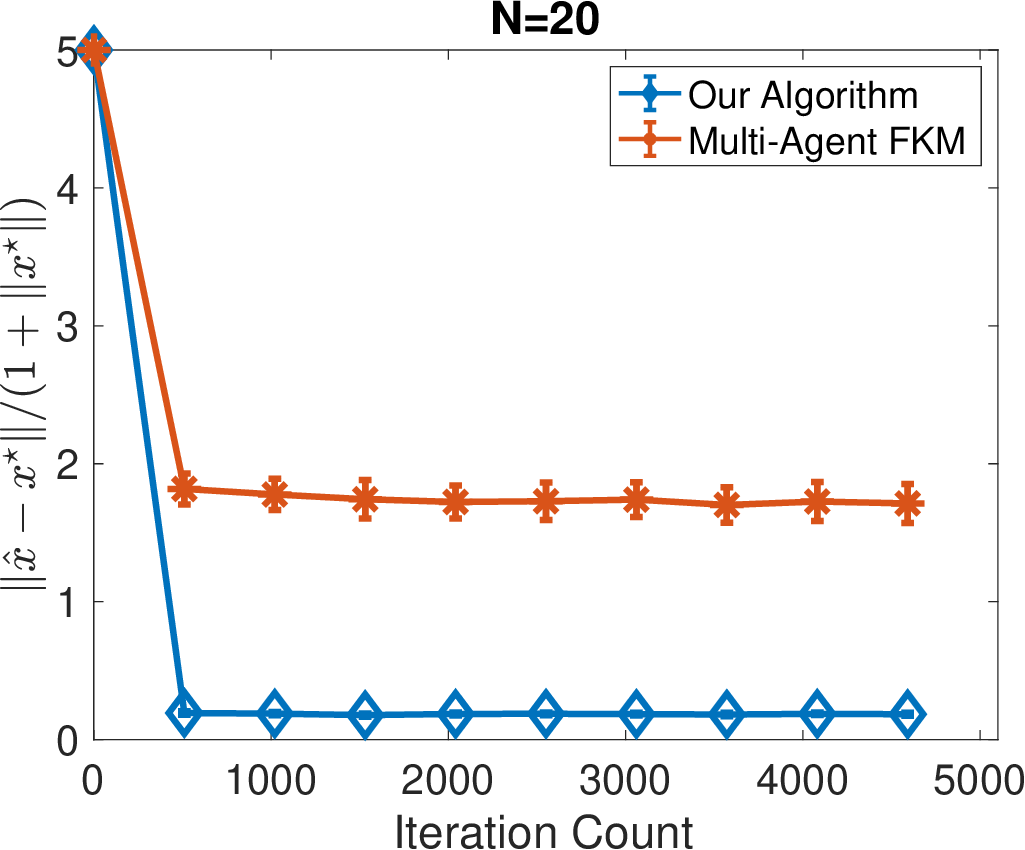}
\includegraphics[width=0.45\textwidth]{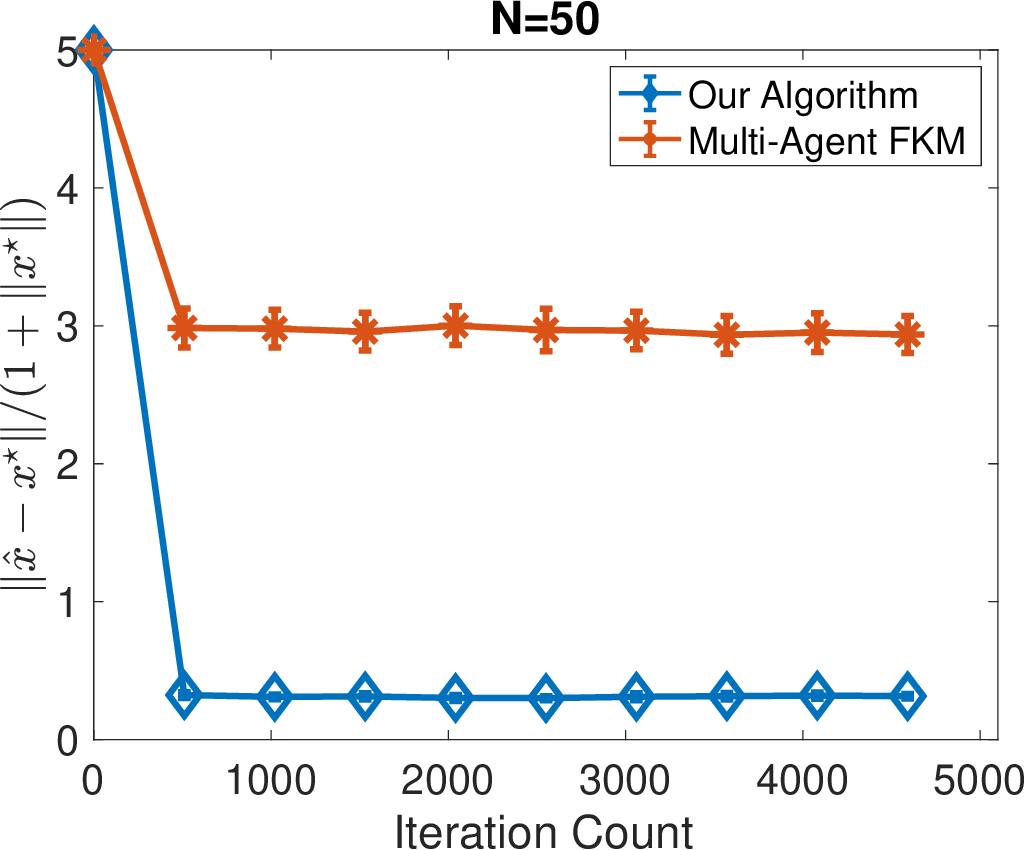}
\includegraphics[width=0.45\textwidth]{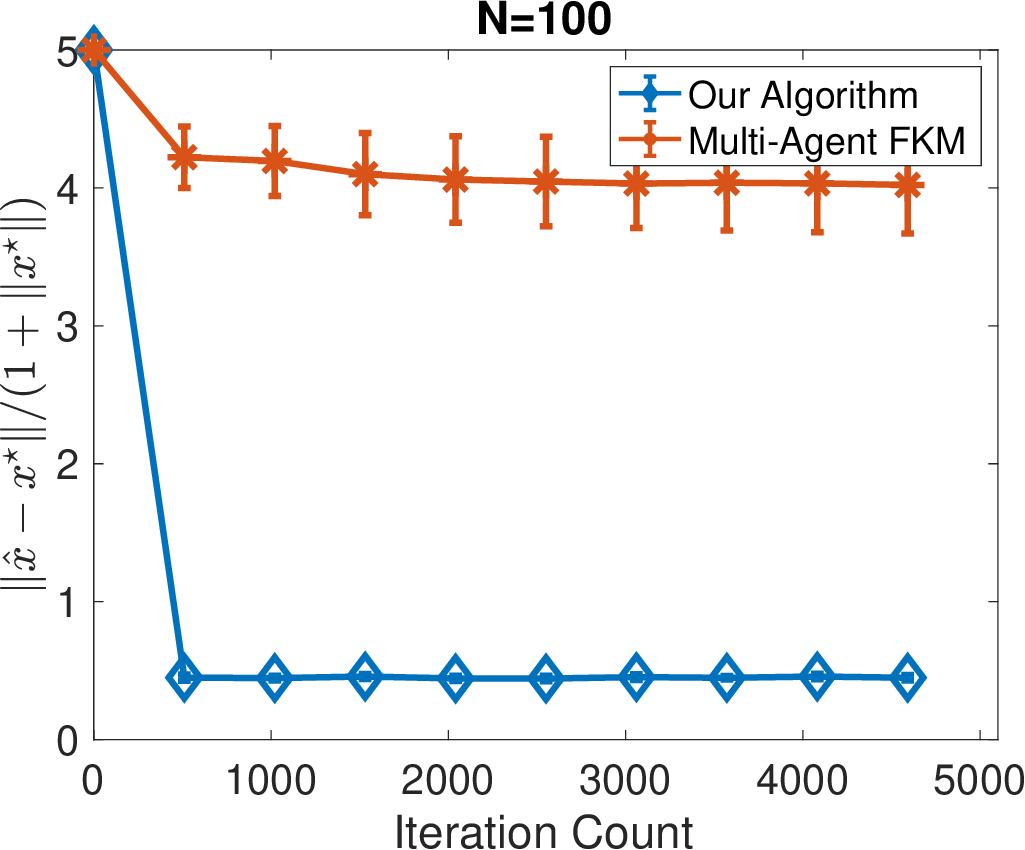}
\caption{Performance of both algorithms for solving Kelly auction. The number of independent trials is 10 and $(S, \bar{d}) = (2, 0.5)$ is fixed. The numerical results are presented as error v.s. iteration count.}\label{Fig:KA}
\end{figure*}

For both Algorithm~\ref{Alg:MA} and the multi-agent FKM, we consider the theoretically-correct choices of step sizes without fine-tuning. Indeed, we set $\lambda_i = \frac{1}{g_i}$ for all $i \in \NCal$ and $\beta = \frac{\min_{s \in \SCal}\{q_s d_s\}}{(\sum_{s \in \SCal} d_s + \sum_{i \in \NCal} B_i)^3}$ for Algorithm~\ref{Alg:MA} due to the structure of Kelly auction. Since $\XCal_i = \{x_i \in \br^S: \one_S^\top x_i \leq B_i, x_i \geq 0\}$ where $\one_S$ is a vector of $S$ dimensional with all one, we set $R_i(x_i) = - \sum_{s \in \SCal}\log(x_{is}) -  \log(B_i- \sum_{s \in \SCal} x_{is})$. According to Theorem~\ref{Thm:last-iterate-perfect}, we set $\eta_t = \frac{1}{2NS\sqrt{t}}$. For all $i \in \NCal$, we set $r_i = \frac{B_i}{S(S+1)}$ and $p_i = \frac{B_i}{S}\one_S$ in multi-agent FKM. We also set $\delta_t = \min\{\min_{i \in \NCal} r_i, \frac{1}{t^{1/3}}\}$ and $\gamma_t = \tfrac{1}{3\beta t}$ according to~\citet[Theorem~5.2]{Bravo-2018-Bandit}. Moreover, it is well known that the Kelly auction in Example~\ref{Example:KA} has a variational characterization~\cite[Proposition~2.1]{Mertikopoulos-2019-Learning} and the equilibrium computation is equivalent to finding a point $x^\star \in \XCal$ such that $\langle v(x^\star), x - x^\star\rangle \leq 0$ for all $x \in \XCal$. The evaluation metric is $\frac{\|\hat{x}^t - x^\star\|}{1 + \|x^\star\|}$ where $\hat{x}^t$ is generated by the algorithms and $x^\star$ is an approximate Nash equilibrium with high accuracy (we obtain it by employing the optimistic gradient method~\citep{Hsieh-2019-Convergence} to find an approximate solution of the VI mentioned above). This point is a benchmark for evaluating the quality of the solution obtained by the algorithms.
\begin{table}
\centering\caption{The solution quality on Kelly auction after 5000 iterations.}\label{Tab:KA}
\begin{tabular}{|c||c|c|} \hline
$(N, S, \bar{d})$  	& Multi-Agent FKM 					& Our Algorithm \\ \hhline{===} 
(10, 2, 0.5) 				& 1.2e+00 $\pm$ 2.0e-01 			& 1.4e-01 $\pm$ 2.3e-02 \\ 
(10, 2, 1.0) 				& 1.2e+00 $\pm$ 1.1e-01 			& 1.3e-01 $\pm$ 1.8e-02 \\
(10, 5, 0.5) 			& 8.7e-01 $\pm$ 1.3e-01 			& 1.8e-01 $\pm$ 5.8e-02 \\
(10, 5, 1.0) 				& 9.2e-01 $\pm$ 9.2e-02 		& 1.6e-01 $\pm$ 3.9e-02 \\ \hhline{===} 
(20, 2, 0.5) 			& 1.6e+00 $\pm$ 4.0e-01 		& 2.0e-01 $\pm$ 2.4e-02 \\ 
(20, 2, 1.0) 			& 1.9e+00 $\pm$ 1.8e-01 			& 1.9e-01 $\pm$ 2.1e-02 \\
(20, 5, 0.5) 			& 1.3e+00 $\pm$ 8.8e-02 		& 1.9e-01 $\pm$ 3.2e-02 \\
(20, 5, 1.0) 			& 1.4e+00 $\pm$ 1.6e-01 			& 1.7e-01 $\pm$ 2.8e-02 \\ \hhline{===} 
(50, 2, 0.5) 			& 2.6e+00 $\pm$ 4.7e-01 			& 3.0e-01 $\pm$ 2.8e-02 \\
(50, 2, 1.0) 			& 3.2e+00 $\pm$ 2.2e-01 		& 3.3e-01 $\pm$ 1.7e-02 \\
(50, 5, 0.5) 			& 2.2e+00 $\pm$ 1.9e-01 			& 2.3e-01 $\pm$ 1.5e-02 \\
(50, 5, 1.0) 			& 2.2e+00 $\pm$ 1.6e-01 			& 2.3e-01 $\pm$ 1.9e-02 \\ \hhline{===} 
(100, 2, 0.5) 			& 4.2e+00 $\pm$ 3.4e-01 		& 4.6e-01 $\pm$ 3.9e-02 \\ 
(100, 2, 1.0) 			& 4.7e+00 $\pm$ 1.3e-01 			& 5.1e-01 $\pm$ 3.1e-02 \\
(100, 5, 0.5) 			& 3.2e+00 $\pm$ 6.0e-02 		& 2.7e-01 $\pm$ 8.7e-03 \\
(100, 5, 1.0) 			& 3.3e+00 $\pm$ 8.5e-02 		& 2.9e-01 $\pm$ 2.1e-02 \\ \hline
\end{tabular}
\end{table}
\paragraph{Experimental results.} Fixing $(S, \bar{d}) = (5, 0.5)$, we investigate the convergence behavior of both algorithms with a varying number of players, i.e., $N \in \{10, 20, 50, 100\}$. Figure~\ref{Fig:KA} indicates that our algorithm outperforms multi-agent FKM as it exhibits a faster convergence to Nash equilibrium in terms of iteration count. Compared to Figure~\ref{Fig:CC}, the quality of the solutions obtained by both algorithms is poorer. Indeed, the simplex set is more complicated than the box set from the Cournot competition. Compared to the multi-agent FKM which gets stuck after a few iterations, our algorithm approaches Nash equilibrium slowly but steadily. This is possibly because of the inaccurate estimation of $\beta$ (too conservative) which results in problematic step sizes for the multi-agent FKM. We also present the numerical results for all $(N, S, \bar{d})$ in Table~\ref{Tab:KA}.

\section{Concluding Remarks and Future Directions}\label{sec:conclu}
In a multi-agent online environment with bandit feedback, the most sensible choice for a player who is oblivious to the presence of others (or who are conservative), is to deploy an optimal no-regret learning algorithm. With this in mind, we investigate the long-run behavior of individual optimal regularized no-regret learning policies. We show that, in strongly monotone games, the joint actions of all players converge to a (necessarily) unique Nash equilibrium, and the rate of convergence matches the lower bound established in~\citet{Shamir-2013-Complexity} up to log factors. We conduct the experiments using Cournot competition and Kelly auction, the numerical results evidence our main results on the convergence property of our algorithm (cf. Theorem~\ref{Thm:last-iterate-perfect}). 

Our work settles an open problem and contributes to the landscape of bandit game-theoretical learning by identifying the first doubly optimal bandit learning algorithm, in that it achieves (up to log factors) both optimal regret in the single-agent learning and optimal last-iterate convergence rate in the multi-agent learning. Future works include the design of a fully decentralized bandit learning algorithm where the players' updates need not be synchronous and the applications of our algorithms to online decision-making problems in practice. 


\bibliographystyle{plainnat}
\bibliography{ref}

\clearpage
\appendix
\section{Proof of Lemma~\ref{Lemma:SA-main}}\label{app:SA-main}
By the definition of $\PCal_R(\cdot, \cdot, \cdot)$ in Eq.~\eqref{def:prox-map}, the iterate $x^{t+1} = \PCal_R(x^t, \hat{v}^t, \eta_t)$ satisfies that 
\begin{equation*}
- \eta_t\hat{v}^t + \eta_t\beta(t+1)(x^{t+1} - x^t) + \nabla R(x^{t+1}) - \nabla  R(x^t) = 0. 
\end{equation*}
Using Lemma~\ref{Lemma:three-point} with $x = x^{t+1}$ and $x' = x^t$, we have
\begin{eqnarray*}
\lefteqn{D_R(p, x^t) = D_R(p, x^{t+1}) +  D_R(x^{t+1}, x^t) + \langle\nabla R(x^t) - \nabla R(x^{t+1}), x^{t+1} - p\rangle} \\ 
& = &  D_R(p, x^{t+1}) +  D_R(x^{t+1}, x^t) - \eta_t \langle\hat{v}^t, x^{t+1} - p\rangle + \eta_t\beta(t+1) \langle x^{t+1} - x^t, x^{t+1} - p\rangle \\
& = & D_R(p, x^{t+1}) - \eta_t \langle\hat{v}^t, x^{t+1} - x^t\rangle - \eta_t \langle\hat{v}^t, x^t - p\rangle + \left(D_R(x^{t+1}, x^t) + \tfrac{\eta_t\beta(t+1)}{2}\|x^{t+1} - x^t\|^2\right) \\
& & + \tfrac{\eta_t\beta(t+1)}{2}\|x^{t+1} - p\|^2 - \tfrac{\eta_t\beta(t+1)}{2}\|x^t - p\|^2. 
\end{eqnarray*}
By the definition, we have $D_R(x^{t+1}, x^t) \geq 0$. Putting these pieces together yields that 
\begin{equation}\label{inequality:SA-main}
D_R(p, x^{t+1}) + \tfrac{\eta_t\beta(t+1)}{2}\|x^{t+1} - p\|^2 \leq D_R(p, x^t) + \tfrac{\eta_t\beta(t+1)}{2}\|x^t - p\|^2 + \eta_t \langle\hat{v}^t, x^{t+1} - x^t\rangle + \eta_t \langle\hat{v}^t, x^t - p\rangle. 
\end{equation}
Then we bound the term $\langle\hat{v}^t, x^{t+1} - x^t\rangle$ by Lemma~\ref{Lemma:SC-key-estimate}. Indeed, we let the function $g$ be defined as
\begin{equation*}
g(x) = \eta_t \langle\hat{v}^t, x^t - x\rangle + \tfrac{\eta_t \beta(t+1)}{2}\|x^t - x\|^2 + D_R(x, x^t)
\end{equation*}
Since $g$ is the sum of a self-concordant barrier function $R$ and a quadratic function, we have $g$ is also a self-concordant function and have
\begin{equation}\label{def:SCB-grad-hess}
\nabla g(x^t) = -\eta_t\hat{v}^t, \quad \nabla^2 g(x^t) = \eta_t \beta(t+1)I_n + \nabla^2 R(x^t). 
\end{equation}
By definition of $A^t$, we have $A^t = (\nabla^2 g(x^t))^{-1/2}$. Then, Eq.~\eqref{def:SCB-grad-hess} implies that 
\begin{equation}\label{inequality:SA-first}
\lambda(x^t, g) = \|(\nabla^2 g(x^t))^{-1}\nabla g(x^t)\|_{x^t} = \eta_t\|\hat{v}^t\|_{x^t, \star}. 
\end{equation}
To apply Lemma~\ref{Lemma:SC-key-estimate}, we need to guarantee that $\lambda(x^t, g) \leq \tfrac{1}{2}$. Indeed, by the definition of $\hat{v}^t$ and using $A^t = (\nabla^2 g(x^t))^{-1/2}$, we have
\begin{equation*}
\lambda(x^t, g) = \eta_t\|\hat{v}^t\|_{x^t, \star} = \eta_t\|A^t\hat{v}^t\| =  n \cdot \eta_t|u^t(x^t + A^t z^t)|\|z^t\| \leq n \cdot \eta_t|u^t(x^t + A^t z^t)|. 
\end{equation*}
Since $x^t + A^t z^t \in \XCal$ and $|u^t(x)| \leq L$ for all $x \in \XCal$, we have $\lambda(x^t, g) \leq n\eta_t L$. Combining it with $0 < \eta_t \leq \tfrac{1}{2nL}$ yields that $\lambda(x^t, g) \leq \tfrac{1}{2}$. By Lemma~\ref{Lemma:SC-key-estimate}, we have
\begin{equation*}
\|x^{t+1} - x^t\|_{x^t} = \|x^t - \argmin_{x' \in \XCal} g(x')\| \leq 2\lambda(x^t, g) \overset{\textnormal{Eq.~\eqref{inequality:SA-first}}}{\leq} 2\eta_t\|\hat{v}^t\|_{x^t, \star}. 
\end{equation*}
This together with the H\"{o}lder's inequality yields that
\begin{equation*}
\langle\hat{v}^t, x^{t+1} - x^t\rangle \leq \|\hat{v}^t\|_{x^t, \star}\|x^{t+1} - x^t\|_{x^t} \leq 2\eta_t\|\hat{v}^t\|_{x^t, \star}^2.  
\end{equation*}
Using $A^t = (\nabla^2 g(x^t))^{-1/2}$ again, we have $\langle\hat{v}^t, x^{t+1} - x^t\rangle \leq 2\eta_t\|A^t\hat{v}^t\|^2$. Plugging this inequality into Eq.~\eqref{inequality:SA-main} yields the desired inequality. 

\section{Proof of Theorem~\ref{Thm:regret-optimal}}\label{app:regret-optimal}
We are in a position to prove Theorem~\ref{Thm:regret-optimal} regarding the regret bound of Algorithm~\ref{Alg:SA}. For simplicity, we assume that $\max_{x, x' \in \XCal} \|x-x'\| \leq B$ for some $B>0$. Fixing a point $p \in \XCal$, we have
\begin{eqnarray}\label{inequality:regret-optimal-main}
\lefteqn{\sum_{t=1}^T u^t(p) - \EE\left[\sum_{t=1}^T u^t(\hat{x}^t)\right] = \underbrace{\sum_{t=1}^T \EE[u^t(p) - \hat{u}^t(p)]}_{\textbf{I}}}  \\
& & + \underbrace{\sum_{t=1}^T \EE[\hat{u}^t(p) - \hat{u}^t(x^t)]}_{\textbf{II}} + \underbrace{\sum_{t=1}^T \EE[\hat{u}^t(x^t) - u^t(x^t)]}_{\textbf{III}} + \underbrace{\sum_{t=1}^T \EE[u^t(x^t) - u^t(\hat{x}^t)]}_{\textbf{IV}}, \nonumber 
\end{eqnarray}
where $\hat{u}^t(x) = \EE_{w \sim \BB^n}[u^t(x+A^t w)]$. Since $\nabla^2 R(x)$ is positive definite, $\eta_t = \frac{1}{2nL\sqrt{T}}$ and $A^t = (\nabla^2 R(x^t) + \eta_t \beta(t+1) I_n)^{-1/2}$, we have
\begin{equation}\label{inequality:upper-bound-A}
\EE[(\sigma_{\max}(A^t))^2] \leq \tfrac{1}{\eta_t\beta(t+1)} = \tfrac{2nL\sqrt{T}}{\beta(t+1)}. 
\end{equation}
By Lemma~\ref{Lemma:ellipsodial} and Eq.~\eqref{inequality:upper-bound-A}, we have
\begin{equation}\label{inequality:estm-I}
\textbf{I} \leq \tfrac{\ell}{2} \sum_{t=1}^T \EE[(\sigma_{\max}(A^t))^2] \leq \tfrac{n\ell L\sqrt{T}}{\beta}\sum_{t=1}^T \tfrac{1}{t+1} \leq \tfrac{n\ell L\sqrt{T}\log(T+1)}{\beta}. 
\end{equation}
Since $u^t$ is concave and $\EE[A^t z^t \mid x^t] = 0$, we have
\begin{equation}\label{inequality:estm-III}
\textbf{III} = \sum_{t=1}^T \EE[\EE[\hat{u}^t(x^t) - u^t(x^t) \mid x^t]] \leq 0. 
\end{equation}
By the definition of $\hat{x}^t$, we have
\begin{equation*}
\textbf{IV} = \sum_{t=1}^T \EE[u^t(x^t) - u^t(\hat{x}^t)] = \sum_{t=1}^T \EE[\EE[u^t(x^t) - u^t(x^t + A^t z^t) \mid x^t]]. 
\end{equation*}
By Lemma~\ref{Lemma:ellipsodial} and Eq.~\eqref{inequality:upper-bound-A} again, we have
\begin{equation}\label{inequality:estm-IV}
\textbf{IV} \leq \tfrac{\ell}{2} \sum_{t=1}^T \EE[(\sigma_{\max}(A^t))^2] \leq \tfrac{n\ell L\sqrt{T}}{\beta}\sum_{t=1}^T \tfrac{1}{t+1} \leq \tfrac{n\ell L\sqrt{T}\log(T+1)}{\beta}. 
\end{equation}
It remains to bound the second term. Since $\eta_t = \tfrac{1}{2nL\sqrt{T}} \leq \tfrac{1}{2nL}$, Lemma~\ref{Lemma:SA-main} implies that 
\begin{equation*}
D_R(p, x^{t+1}) + \tfrac{\eta_t\beta(t+1)}{2}\|x^{t+1} - p\|^2 \leq D_R(p, x^t) + \tfrac{\eta_t\beta(t+1)}{2}\|x^t - p\|^2 + 2\eta_t^2\|A^t\hat{v}^t\|^2 + \eta_t \langle\hat{v}^t, x^t - p\rangle.
\end{equation*}
Taking the expectation of both sides conditioned on $x^t$ and dividing both sides by $\eta_t$, we have
\begin{equation*}
\tfrac{\EE[D_R(p, x^{t+1}) \mid x^t]}{\eta_t} + \tfrac{\beta(t+1)}{2}\EE[\|x^{t+1} - p\|^2 \mid x^t] \leq \tfrac{D_R(p, x^t)}{\eta_t} + \tfrac{\beta(t+1)}{2}\|x^t - p\|^2 + 2\eta_t\EE[\|A^t\hat{v}^t\|^2 \mid x^t] + \EE[\langle\hat{v}^t, x^t - p\rangle \mid x^t].
\end{equation*}
By the definition of $\hat{v}^t$ and using Lemma~\ref{Lemma:ellipsodial}, we have
\begin{equation*}
\EE[\|A^t\hat{v}^t\|^2 \mid x^t] = n^2 \cdot \EE[|u^t(x^t + A^t z^t)|^2\|z^t\|^2 \mid x^t] \leq n^2 L^2, 
\end{equation*}
and 
\begin{equation*}
\EE[\langle\hat{v}^t, x^t - p\rangle \mid x^t] = \langle \nabla \hat{u}^t(x^t), x^t - p\rangle. 
\end{equation*}
Since $u^t$ is $\beta$-strongly concave, Lemma~\ref{Lemma:ellipsodial} implies that $\hat{u}^t$ is also $\beta$-strongly concave and we have $\EE[\langle\hat{v}^t, x^t - p\rangle \mid x^t] \leq \hat{u}^t(x^t) - \hat{u}^t(p) - \tfrac{\beta}{2}\|x^t - p\|^2$. Putting these pieces together yields that 
\begin{equation*}
\tfrac{\EE[D_R(p, x^{t+1}) \mid x^t]}{\eta_t} + \tfrac{\beta(t+1)}{2}\EE[\|x^{t+1} - p\|^2 \mid x^t] \leq \tfrac{D_R(p, x^t)}{\eta_t} + \tfrac{\beta t}{2}\|x^t - p\|^2 + 2\eta_t n^2L^2 + (\hat{u}^t(x^t) - \hat{u}^t(p)).
\end{equation*}
Rearranging the above inequality and take the expectation of both sides, we have
\begin{equation*}
\EE[\hat{u}^t(p) - \hat{u}^t(x^t)] \leq \tfrac{\EE[D_R(p, x^t)]-\EE[D_R(p, x^{t+1})]}{\eta_t} + \tfrac{\beta t}{2}\EE[\|x^t - p\|^2] - \tfrac{\beta(t+1)}{2}\EE[\|x^{t+1} - p\|^2] + 2\eta_t n^2L^2.
\end{equation*}
Therefore, we have
\begin{eqnarray}\label{inequality:estm-II-main}
\textbf{II} & \leq & \sum_{t=1}^T \tfrac{\EE[D_R(p, x^t)]-\EE[D_R(p, x^{t+1})]}{\eta_t} + \tfrac{\beta}{2}\sum_{t=1}^T (t\EE[\|x^t - p\|^2] - (t+1)\EE[\|x^{t+1} - p\|^2]) +  2n^2L^2 \sum_{t=1}^T \eta_t \nonumber \\
& = & 2nL\sqrt{T}D_R(p, x^1) + \tfrac{\beta}{2}\|x^1 - p\|^2 + nL\sqrt{T} \leq nL\sqrt{T}(1+2D_R(p, x^1)) + \tfrac{\beta B^2}{2}. 
\end{eqnarray}
Since $x^1 = \argmin_{x \in \XCal} R(x)$, we have $\nabla R(x^1) = 0$ and $D_R(p, x^1) = R(p) - R(x^1)$. Then, we consider the following two cases: 
\begin{itemize}
\item A point $p \in \XCal$ satisfies that $\pi_{x^1}(p) \leq 1 - \frac{1}{\sqrt{T}}$. By Lemma~\ref{Lemma:SCB-upper-bound}, we have $D_R(p, x^1) = R(p) - R(x^1) \leq \nu\log(T)$. This together with Eq.~\eqref{inequality:estm-II-main} implies that 
\begin{equation}\label{inequality:estm-II}
\textbf{II} \leq nL\sqrt{T}(1+2\nu\log(T)) + \tfrac{\beta B^2}{2}. 
\end{equation}
Plugging Eq.~\eqref{inequality:estm-I}, Eq.~\eqref{inequality:estm-III}, Eq.~\eqref{inequality:estm-IV} and Eq.~\eqref{inequality:estm-II} into Eq.~\eqref{inequality:regret-optimal-main} yields that 
\begin{equation}\label{inequality:regret-optimal-result}
\sum_{t=1}^T u^t(p) - \EE\left[\sum_{t=1}^T u^t(\hat{x}^t)\right] = \tilde{O}(n\sqrt{T}). 
\end{equation}
\item A point $p \in \XCal$ satisfies that $\pi_{x^1}(p) > 1 - \tfrac{1}{\sqrt{T}}$. Thus, we can find $p' \in \XCal$ such that $\|p' - p\| = O(\tfrac{1}{\sqrt{T}})$ and $\pi_{x_1}(p') \leq 1 - \tfrac{1}{\sqrt{T}}$. Since Eq.~\eqref{inequality:regret-optimal-result} holds true for any fixed $p \in \XCal$ satisfying $\pi_{x_1}(p) \leq 1 - \tfrac{1}{\sqrt{T}}$, we have
\begin{equation*}
\sum_{t=1}^T u^t(p') - \EE\left[\sum_{t=1}^T u^t(\hat{x}^t)\right] = \tilde{O}(n\sqrt{T}). 
\end{equation*}
Since each function $u^t$ is Lipschitz continuous, we have 
\begin{equation*}
\left|\sum_{t=1}^T u^t(p') - u^t(p)\right| = O(\sqrt{T}). 
\end{equation*}
Therefore, Eq.~\eqref{inequality:regret-optimal-result} also holds true in this case. 
\end{itemize}
Combining the above two cases with the fact that $p \in \XCal$ is arbitrarily chosen, we have
\begin{equation*}
\textnormal{Reg}_T = \max_{p \in \XCal} \left\{\sum_{t=1}^T u^t(p) - \EE\left[\sum_{t=1}^T u^t(\hat{x}^t)\right]\right\} = \tilde{O}(n\sqrt{T}). 
\end{equation*}
This completes the proof. 

\section{Proofs for Example~\ref{Example:CC},~\ref{Example:KA} and~\ref{Example:pricing}}\label{app:examples}
We show that Cournot competition and Kelly auction in Example~\ref{Example:CC} and~\ref{Example:KA} are $(\beta, \{\lambda_i\}_{i \in \NCal})$-strongly monotone games (cf. Definition~\ref{def:SMS}) for some $\beta > 0$ and $\lambda_i > 0$ for all $i \in \NCal$. 

\paragraph{Cournot competition.} Example~\ref{Example:CC} satisfies Definition~\ref{def:SMS} with $\beta = b$ and $\lambda_i = 1$ for all $i \in \NCal$. Indeed, each players' reward function is given by 
\begin{equation*}
u_i(x) = x_i\left(a - b\left(\sum_{j \in \NCal} x_j\right)\right) - c_ix_i, \quad \textnormal{for some } a, b, c_i > 0. 
\end{equation*}
Taking the derivative of $u_i(x)$ with respect to $x_i$ yields that $v_i(x) = a - c_i - b\left(\sum_{j \in \NCal} x_j\right) - bx_i$. Then, we have  
\begin{equation*}
\sum_{i \in \NCal} \langle x'_i - x_i, v_i(x') - v_i(x)\rangle = -b\left(\sum_{i \in \NCal} x_i\right)^2 - b\|x' - x\|^2 \leq  - b\|x' - x\|^2, \quad \textnormal{for all } x, x' \in \XCal. 
\end{equation*}
Therefore, we conclude the desired result as mentioned before. 

\paragraph{Kelly auction.} Example~\ref{Example:KA} satisfies Definition~\ref{def:SMS} with $\beta = \min_{s \in \SCal}\{q_s d_s\}\left(\sum_{s \in \SCal} d_s + \sum_{i \in \NCal} B_i\right)^{-3}$ and $\lambda_i = 1/g_i$ for all $i \in \NCal$. Indeed, each players' reward function is given by 
\begin{equation*}
u_i(x) = \left(\sum_{s \in \SCal} \frac{g_i q_s x_{is}}{d_s + \sum_{j\in \NCal} x_{js}}\right) - \sum_{s \in \SCal} x_{is}, \quad \textnormal{for some } g_i, d_s, q_s >0. 
\end{equation*}
Define $v_i(x) = \nabla_i u_i(x)$ for all $i \in \NCal$, it suffices to prove that 
\begin{equation*}
\sum_{i \in \NCal} \tfrac{1}{g_i} \langle x'_i - x_i, v_i(x') - v_i(x)\rangle \leq -\tfrac{\min_{s \in \SCal}\{q_s d_s\}}{(\sum_{s \in \SCal} d_s + \sum_{i \in \NCal} B_i)^3}\|x' - x\|^2, \quad \textnormal{for all } x, x' \in \XCal. 
\end{equation*}
The following proposition is a restatement of~\citet[Theorem~6]{Rosen-1965-Existence} and plays an important role in the subsequent analysis. 
\begin{proposition}\label{prop:DSC}
Given a continuous game $\GCal = (\NCal, \XCal = \prod_{i=1}^N \XCal_i, \{u_i\}_{i=1}^N)$, where each $u_i$ is twice continuously differentiable. For each $x \in \XCal$ , define the $\lambda$-weighted Hessian matrix $H^\lambda(x)$ as follows:
\begin{equation*}
H_{ij}^\lambda(x) = \frac{1}{2}\lambda_i \nabla_j v_i(x) + \frac{1}{2}\lambda_j (\nabla_i v_j(x))^\top. 
\end{equation*}
If $H^\lambda(x)$ is negative-definite for every $x \in \XCal$, we have
\begin{equation*}
\sum_{i \in \NCal} \lambda_i \langle x'_i - x_i, v_i(x') - v_i(x)\rangle \leq 0 \quad \textnormal{for all } x, x' \in \XCal, 
\end{equation*}
where the equality holds true if and only if $x = x'$. 
\end{proposition}
As a consequence of Proposition~\ref{prop:DSC}, we have $\sum_{i \in \NCal} \lambda_i \langle x'_i - x_i, v_i(x') - v_i(x)\rangle \leq -\beta\|x' - x\|^2$ for all $x, x' \in \XCal$ if $H^\lambda(x) \preceq -\beta I_n$ for all $x \in \XCal$. Thus, it suffices to show that 
\begin{equation}\label{condition:KA-main}
H^\lambda(x) \preceq -\tfrac{\min_{s \in \SCal}\{q_s d_s\}}{(\sum_{s \in \SCal} d_s + \sum_{i \in \NCal} B_i)^3} I_n, \quad \textnormal{for all } x \in \XCal. 
\end{equation}
where $H_{ij}^\lambda(x) = \frac{1}{2g_i} \nabla_j v_i(x) + \frac{1}{2g_j} (\nabla_i v_j(x))^\top$ for each $x \in \XCal$ and all $(i, j) \in [N] \times [N]$. 

We define the so-called social welfare function as follows, 
\begin{equation}\label{def:KA-welfare}
U(x) = \sum_{k \in \NCal} \tfrac{1}{g_k} u_k(x) = \sum_{s \in \SCal} q_s \left(\tfrac{\sum_{i \in \NCal} x_{is}}{d_s + \sum_{i \in \NCal} x_{is}}\right) - \sum_{i \in \NCal} \sum_{s \in \SCal} x_{is}. 
\end{equation}
By the definition, we have
\begin{equation*}
\begin{array}{lclcl}
\nabla_{ii}^2 U(x) & = & \sum_{k \in \NCal} \tfrac{1}{g_k} \nabla_{ii}^2 u_k(x) & = & 2H_{ii}^\lambda(x) - \tfrac{1}{g_i} \nabla_{ii}^2 u_i(x) + \sum_{k \neq i} \tfrac{1}{g_k} \nabla_{ii}^2 u_k(x), \\
\nabla_{ij}^2 U(x) & = & \sum_{k \in \NCal} \tfrac{1}{g_k} \nabla_{ij}^2 u_k(x) & = & 2H_{ij}^\lambda(x) + \sum_{k \neq i, j} \tfrac{1}{g_k} \nabla_{ij}^2 u_k(x). 
\end{array}
\end{equation*}
For the simplicity, we just write $M_{ij}(x) = \nabla_{ij}^2 U(x)$, $D_{ij}(x) = \frac{1}{g_i} \nabla_{ij} u_i(x) \delta_{ij}$ and $B_{ij}^k(x) = \nabla_{ij} u_k(x) (1-\delta_{ik})(1-\delta_{jk})$. Then, we have
\begin{equation}\label{condition:KA-first}
2H^\lambda(x) = M(x) + D(x) - \sum_{i \in \NCal} \tfrac{1}{g_i} B^i(x), \quad \textnormal{for all } x \in \XCal. 
\end{equation}
We are now ready to prove Eq.~\eqref{condition:KA-main}. Indeed, since the $f(x) = \frac{x}{c+x}$ is concave for all $c > 0$, we have
\begin{itemize}
\item The social welfare function $U(x)$ is concave in $x$. 
\item Each reward function $u_i(x_i; x_{-i})$ is concave in $x_i$ and convex in $x_{-i}$. 
\end{itemize}
Since $U(x)$ is concave in $x$, we have $M(x) \preceq 0$ for all $x \in \XCal$. Since $u_i(x_i; x_{-i})$ is convex in $x_{-i}$, we have $B^i(x) \succeq 0$ for all $i \in \NCal$ and all $x \in \XCal$ (simply note that $B^i(x)$ is the Hessian of $u_i(x_i; x_{-i})$ with the variable $x_i$ omitted). Putting all these pieces together with Eq.~\eqref{condition:KA-first} yields that 
\begin{equation}\label{condition:KA-second}
H^\lambda(x) \preceq \tfrac{1}{2}D(x), \quad \textnormal{for all } x \in \XCal. 
\end{equation}
Note that $D(x) \in \br^{n \times n}$ is a block diagonal matrix with $\frac{1}{g_i} \nabla_{ii}^2 u_i(x_i; x_{-i})$ being the diagonal elements. Since $u_i(x)$ is the sum of $S$ functions where the $s^{\textnormal{th}}$ one only depends on $(x_{1s}, x_{2s}, \ldots, x_{Ns})$, we have
\begin{equation*}
\nabla_{ii}^2 u_i(x_i; x_{-i}) = \begin{pmatrix}
- \frac{2g_i q_1(d_1 + \sum_{j \neq i} x_{j1})}{(d_1 + \sum_{j \in \NCal} x_{j1})^3}  & & & \\ 
& - \frac{2g_i q_2(d_2 + \sum_{j \neq i} x_{j2})}{(d_2 + \sum_{j \in \NCal} x_{j2})^3} & & \\
& & & \ddots & \\ 
& & & & - \frac{2g_i q_S(d_S + \sum_{j \neq i} x_{jS})}{(d_S + \sum_{j \in \NCal} x_{jS})^3} 
\end{pmatrix}
\end{equation*}
Since $0 < d_s + \sum_{j \neq i} x_{js} \leq d_s + \sum_{j \in \NCal} x_{js} \leq \sum_{s \in \SCal} d_s + \sum_{i \in \NCal} B_i$ for all $x \in \XCal$, we have
\begin{equation}\label{condition:KA-third}
\tfrac{1}{g_i} \nabla_{ii}^2 u_i(x_i; x_{-i}) \preceq -\tfrac{2\min_{s \in \SCal}\{q_s d_s\}}{(\sum_{s \in \SCal} d_s + \sum_{i \in \NCal} B_i)^3} I_{n_i}. 
\end{equation}
Plugging Eq.~\eqref{condition:KA-third} into Eq.~\eqref{condition:KA-second}, we have
\begin{equation*}
H^\lambda(x) \preceq -\tfrac{\min_{s \in \SCal}\{q_s d_s\}}{(\sum_{s \in \SCal} d_s + \sum_{i \in \NCal} B_i)^3} I_n. 
\end{equation*}
Therefore, we conclude the desired result as mentioned before. 

\paragraph{Retailer Pricing Games.} For the single-retailer multi-product pricing setting, we have 
\begin{equation*}
u^t(p) = \sum_{i=1}^n p_i^t D_i^t(p) = \sum_{i=1}^n p_i\left(\sum_{j=1}^n a_{ij}^t p_j + b_i^t\right) = \sum_{i=1}^n \sum_{j=1}^n a_{ij}^t p_i p_j + \sum_{i=1}^n b_i^t p_i. 
\end{equation*}
Recall that $A^t = (a_{ij}^t)_{1 \leq i, j \leq n} \in \br^{n \times n}$ (not necessarily symmetrical) and we let $b = (b_i^t)_{1 \leq i \leq n}$ , we have 
\begin{equation*}
u^t(p) = p^\top A p + b^\top p. 
\end{equation*}
Since the transpose of a scalar is itself, we have $u^t(p) = (u^t(p))^\top$ for all $p$.  Then, we have 
\begin{equation*}
u^t(p) = \tfrac{1}{2}\left(p^\top A^t p + b^\top p + p^\top (A^t)^\top p + p^\top b\right) = \tfrac{1}{2}p^\top (A^t + (A^t)^\top)p + b^\top p. 
\end{equation*}
If $A^t + (A^t)^\top \leq -\beta I_n$ for all $t \geq 1$, we can easily verify that 
\begin{equation*}
u^t(p) - u^t(p') - (p - p')^\top \nabla u^t(p) \leq -\tfrac{\beta}{2}\|p - p'\|^2. 
\end{equation*}
which implies that the total revenue function $u^t(p)$ is $\beta$-strongly concave in $p$. 

For the multi-retailer pricing with a single product (for simplicity), we can show that the game satisfies Definition~\ref{def:SMS} with $\beta > 0$ and $\lambda_i = 1$ for all $i \in \NCal$ if $A + A^\top \leq -\beta I_n$. Indeed, each players' reward function is given by 
\begin{equation*}
u_i(p) = p_i D_i(p) = \sum_{j \in \NCal} a_{ij} p_i p_j + b_i p_i
\end{equation*}
Taking the derivative of $u_i(p)$ with respect to $p_i$ yields that $v_i(p) = \sum_{j \in \NCal} a_{ij} p_j + a_{ii} p_i + b_i$.  As a consequence of Proposition~\ref{prop:DSC}, we have $\sum_{i \in \NCal} \lambda_i \langle p'_i - p_i, v_i(p') - v_i(p)\rangle \leq -\beta\|p' - p\|^2$ for all $p, p' \in \br^n$ if $H^\lambda(p) \preceq -\beta I_n$ for all $p \in \br^n$. Thus, it suffices to show that 
\begin{equation}\label{condition:pricing-main}
H(p) \preceq -\beta I_n, \quad \textnormal{for all } p \in \br^n,  
\end{equation}
where $H_{ij}(p) = \frac{1}{2} \nabla_j v_i(p) + \frac{1}{2} (\nabla_i v_j(p))^\top$ for each $p \in \br^n$ and all $1 \leq i, j \leq n$. Through some simple calculations, we have $H(p) = A + A^\top$ for all $p \in \br^n$ where $A = (a_{ij})_{1 \leq i, j \leq n} \in \br^{n \times n}$. Therefore, we conclude the desired result in Eq.~\eqref{condition:pricing-main} from the condition assumed in Example~\ref{Example:pricing}. 

\section{Proof of Lemma~\ref{Lemma:ellipsodial-SPSA}}\label{app:ellipsodial-SPSA}
By the definition, we have
\begin{equation*}
\EE[\hat{v}_i \mid x_1, x_2, \ldots, x_N] = \frac{n_i}{\prod\limits_{i \in \NCal} \vol(\bs_i)} \int_{\bs_1} \cdots \int_{\bs_N} u_i(x_1 + A_1z_1, \ldots, x_N + A_Nz_N) (A_i)^{-1} z_i \; dz_1 \ldots dz_N.  
\end{equation*}
Since all of $A_i$ are invertible, we can define the auxiliary functions $U_i(x)$ by 
\begin{equation}\label{def:U}
U_i(x_1, x_2, \ldots, x_N) = u_i(A_1x_1, A_2x_2, \ldots, A_N x_N), \quad \textnormal{for all } i \in \NCal, 
\end{equation}
and have
\begin{equation*}
\EE[\hat{v}_i \mid x_1, x_2, \ldots, x_N] = (A_i)^{-1}\left(\frac{n_i}{\prod\limits_{i \in \NCal} \vol(\bs_i)} \int_{\bs_1} \cdots \int_{\bs_N} U_i((A_1)^{-1}x_1 + z_1, \ldots, (A_N)^{-1}x_N + z_N) z_i \; dz_1 \ldots dz_N\right).  
\end{equation*}
For simplicity, we let $\tilde{x}_i = (A_i)^{-1}x_i$. By applying the same argument as used in~\cite[Lemma~C.1]{Bravo-2018-Bandit} with the independence of sampling directions $\{z_i\}_{i \in \NCal}$, we have
\begin{eqnarray*}
\lefteqn{\frac{n_i}{\prod\limits_{i \in \NCal} \vol(\bs_i)} \int_{\bs_1} \cdots \int_{\bs_N} U_i(\tilde{x}_1 + z_1, \ldots, \tilde{x}_N + z_N) z_i \; dz_1 \ldots dz_N} \\
& = & \frac{n_i}{\prod\limits_{i \in \NCal} \vol(\bs_i)} \int_{\bs_1} \cdots \int_{\bs_N} U_i(\tilde{x}_1 + z_1, \ldots, \tilde{x}_N + z_N) \frac{z_i}{\|z_i\|} \; dz_1 \ldots dz_N \\
& = & \frac{n_i}{\prod\limits_{i \in \NCal} \vol(\bs_i)} \int_{\bs_i} \int_{\prod\limits_{j \neq i} \bs_j} U_i(\tilde{x}_i + z_i; \tilde{x}_{-i} + z_{-i}) \frac{z_i}{\|z_i\|} \; dz_i dz_{-i} \\
& = & \frac{n_i}{\prod\limits_{i \in \NCal} \vol(\bs_i)} \int_{\BB_i} \int_{\prod\limits_{j \neq i} \bs_j} \nabla_i U_i(\tilde{x}_i + w_i; \tilde{x}_{-i} + z_{-i}) \frac{z_i}{\|z_i\|} \; dw_i dz_{-i}, 
\end{eqnarray*}
where, in the last equality, we use the identity 
\begin{equation*}
\nabla \int_{\BB} F(x+w) dw = \int_{\bs} F(x+z)\frac{z}{\|z\|} \; dz, 
\end{equation*}
which, in turn, follows from the Stoke's theorem. Since $\vol(\BB_i) = (1/n_i)\vol(\bs_i)$, the above argument indeed implies that $\EE[\hat{v}_i \mid x] = (A_i)^{-1}\nabla_i \hat{U}_i(\tilde{x})$ with $\hat{U}_i(x) = \EE_{w_i \sim \BB^{n_i}}\EE_{z_{-i} \sim \Pi_{j \neq i} \bs^{n_j}}[U_i(x_i+w_i; x_{-i}+z_{-i})]$. Using Eq.~\eqref{def:U} and $\tilde{x}_i = (A_i)^{-1}x_i$, we have
\begin{equation*}
\EE[\hat{v}_i \mid x_1, x_2, \ldots, x_N] = (A_i)^{-1} A_i\nabla_i \hat{u}_i(x_i) = \nabla_i \hat{u}_i(x_i), 
\end{equation*}
with $\hat{u}_i(x) = \EE_{w_i \sim \BB^{n_i}}\EE_{z_{-i} \sim \Pi_{j \neq i} \bs^{n_j}}[u_i(x_i+A_iw_i; \hat{x}_{-i})]$ where $\hat{x}_i = x_i + A_i z_i$ for all $i \in \NCal$.
	
Moreover, we have $v_i(x)$ is $\ell_i$-Lipschitz continuous. Thus, we have
\begin{equation*}
\|v_i(x_i+A_iw_i; \hat{x}_{-i}) - v_i(x)\| \leq \ell_i\sqrt{\|A_iw_i\|^2 + \sum_{j \neq i} \|A_j z_j\|^2} \leq \ell_i\sqrt{\sum_{j \in \NCal} (\sigma_{\max}(A_j))^2}. 
\end{equation*}
This completes the proof. 

\section{Proof of Lemma~\ref{Lemma:MA-perfect-main}}\label{app:MA-perfect-main}
Since Algorithm~\ref{Alg:MA} is developed in which each player chooses her decision using Algorithm~\ref{Alg:SA}, Lemma~\ref{Lemma:SA-main} implies that 
\begin{equation*}
D_{R_i}(p_i, x_i^{t+1}) + \tfrac{\eta_t\beta(t+1)}{2\lambda_i}\|x_i^{t+1} - p_i\|^2 \leq D_{R_i}(p_i, x_i^t) + \tfrac{\eta_t\beta(t+1)}{2\lambda_i}\|x_i^t - p_i\|^2 + 2\eta_t^2\|A_i^t\hat{v}_i^t\|^2 + \eta_t \langle\hat{v}_i^t, x_i^t - p_i\rangle, 
\end{equation*}
where $p_i \in \XCal_i$ and $\{\eta_t\}_{t \geq 1}$ is a nonincreasing sequence satisfying that $0 < \eta_t \leq \tfrac{1}{2n_i L}$. Since $n = \sum_{i \in \NCal} n_i$ in the multi-agent setting, the above inequality holds true for all $i \in \NCal$ if $0 < \eta_t \leq \tfrac{1}{2nL}$. Multiplying it by $\lambda_i$, summing over $i \in \NCal$, and then using $\eta_t \geq \eta_{t+1}$, we obtain the desired inequality.

\section{Proof of Theorem~\ref{Thm:last-iterate-perfect}}\label{app:last-iterate-perfect}
We are in a position to prove Theorem~\ref{Thm:last-iterate-perfect} regarding the last-iterate convergence rate of Algorithm~\ref{Alg:MA} for the case when the bandit feedback is available. For simplicity, we assume that $\max_{x, x' \in \XCal_i} \|x-x'\| \leq B_i$ for some $B_i>0$. Since $\eta_t = \tfrac{1}{2nL\sqrt{t}} \leq \tfrac{1}{2nL}$, Lemma~\ref{Lemma:MA-perfect-main} implies that 
\begin{eqnarray*}
\lefteqn{\sum_{i \in \NCal} \lambda_i D_{R_i}(p_i, x_i^{t+1}) + \tfrac{\eta_{t+1}\beta(t+1)}{2}\left(\sum_{i \in \NCal} \|x_i^{t+1} - p_i\|^2\right)} \\
& \leq & \sum_{i \in \NCal} \lambda_i D_{R_i}(p_i, x_i^t) + \tfrac{\eta_t\beta(t+1)}{2}\left(\sum_{i \in \NCal} \|x_i^t - p_i\|^2\right) + 2\eta_t^2\left(\sum_{i \in \NCal} \lambda_i \|A_i^t\hat{v}_i^t\|^2\right) + \eta_t\left(\sum_{i \in \NCal} \lambda_i \langle\hat{v}_i^t, x_i^t - p_i\rangle\right).
\end{eqnarray*}
Taking the expectation of both sides conditioned on $x^t$, we have
\begin{eqnarray}\label{inequality:estm-main}
\lefteqn{\EE\left[\sum_{i \in \NCal} \lambda_i D_{R_i}(p_i, x_i^{t+1}) \mid x^t\right] + \tfrac{\eta_{t+1}\beta(t+1)}{2}\EE\left[\sum_{i \in \NCal} \|x_i^{t+1} - p_i\|^2 \mid x^t\right]} \\
& \leq & \sum_{i \in \NCal} \lambda_i D_{R_i}(p_i, x_i^t) + \tfrac{\eta_t\beta(t+1)}{2}\left(\sum_{i \in \NCal} \|x_i^t - p_i\|^2\right) + 2\eta_t^2\EE\left[\sum_{i \in \NCal} \lambda_i \|A_i^t\hat{v}_i^t\|^2 \ \Big\vert x^t\right] + \eta_t\EE\left[\sum_{i \in \NCal} \lambda_i\langle\hat{v}_i^t, x_i^t - p_i\rangle \mid x^t\right]. \nonumber
\end{eqnarray}
By the definition of $\hat{v}_i^t$ and using Lemma~\ref{Lemma:ellipsodial-SPSA}, we have
\begin{equation}\label{inequality:estm-first}
\EE[\|A_i^t\hat{v}_i^t\|^2 \mid x^t] = n_i^2 \cdot \EE[|u_i(x_1^t + A_1^t z_1^t, \ldots, x_N^t + A_N^t z_N^t)|^2\|z_i^t\|^2 \mid x^t] \leq n_i^2 L^2, 
\end{equation}
and 
\begin{equation*}
\EE[\langle\hat{v}_i^t, x_i^t - p_i\rangle \mid x^t] = \langle \nabla_i \hat{u}_i(x^t), x_i^t - p_i\rangle. 
\end{equation*}
By using Lemma~\ref{Lemma:ellipsodial-SPSA} again and the Young's inequality, we have
\begin{eqnarray}\label{inequality:estm-residue}
\lefteqn{\EE[\langle\hat{v}_i^t, x_i^t - p_i\rangle \mid x^t] = \langle \nabla_i u_i(x^t), x_i^t - p_i\rangle + \langle \nabla_i \hat{u}_i(x^t) - \nabla_i u_i(x^t), x_i^t - p_i\rangle} \\
& \leq & \langle \nabla_i u_i(x^t), x_i^t - p_i\rangle + \tfrac{\lambda_i}{2\beta}\|\nabla_i \hat{u}_i(x^t) - \nabla_i u_i(x^t)\|^2 + \tfrac{\beta}{2\lambda_i}\|x_i^t - p_i\|^2 \nonumber \\
& \leq & \langle \nabla_i u_i(x^t), x_i^t - p_i\rangle + \tfrac{\lambda_i\ell_i^2}{2\beta}\left(\sum_{i \in \NCal} (\sigma_{\max}(A_i^t))^2\right) + \tfrac{\beta}{2\lambda_i}\|x_i^t - p_i\|^2. \nonumber
\end{eqnarray}
Since $\nabla^2 R_i(x)$ is positive definite for all $x \in \XCal$ and $\eta_t = \frac{1}{2nL\sqrt{t}}$, we have
\begin{equation}\label{inequality:upper-bound-Ai}
(\sigma_{\max}(A_i^t))^2 \leq \tfrac{\lambda_i}{\eta_t\beta(t+1)} = \tfrac{2nL\lambda_i}{\beta\sqrt{t+1}}. 
\end{equation}
Plugging Eq.~\eqref{inequality:upper-bound-Ai} into Eq.~\eqref{inequality:estm-residue} and summing up the resulting inequality over $i \in \NCal$, we have
\begin{equation*}
\EE\left[\sum_{i \in \NCal} \lambda_i \langle\hat{v}_i^t, x_i^t - p_i\rangle \mid x^t\right] \leq \sum_{i \in \NCal} \lambda_i \langle \nabla_i u_i(x^t), x_i^t - p_i\rangle +\tfrac{nL}{\beta^2\sqrt{t+1}}\left(\sum_{i \in \NCal} \lambda_i\right)\left(\sum_{i \in \NCal} \lambda_i^2 \ell_i^2\right) + \tfrac{\beta}{2}\left(\sum_{i \in \NCal} \|x_i^t - p_i\|^2\right).  
\end{equation*}
Since a game is $(\beta, \{\lambda_i\}_{i \in \NCal}\}$-strongly monotone, we have
\begin{equation*}
\sum_{i \in \NCal} \lambda_i \langle \nabla_i u_i(x^t), x_i^t - p_i\rangle \leq -\beta\left(\sum_{i \in \NCal} \|x_i^t - p_i\|^2\right) + \sum_{i \in \NCal} \lambda_i\langle \nabla_i u_i(p), x_i^t - p_i\rangle. 
\end{equation*}
Putting these pieces together yields that 
\begin{equation}\label{inequality:estm-second}
\EE\left[\sum_{i \in \NCal} \lambda_i\langle\hat{v}_i^t, x_i^t - p_i\rangle \mid x^t\right] \leq \tfrac{nL}{\beta^2\sqrt{t+1}}\left(\sum_{i \in \NCal} \lambda_i\right)\left(\sum_{i \in \NCal} \lambda_i^2\ell_i^2\right) - \tfrac{\beta}{2}\left(\sum_{i \in \NCal} \|x_i^t - p_i\|^2\right) + \sum_{i \in \NCal} \lambda_i\langle \nabla_i u_i(p), x_i^t - p_i\rangle.
\end{equation}
Plugging Eq.~\eqref{inequality:estm-first} and Eq.~\eqref{inequality:estm-second} into Eq.~\eqref{inequality:estm-main} and taking the expectation of both sides, we have
\begin{eqnarray*}
\lefteqn{\EE\left[\sum_{i \in \NCal} \lambda_i D_{R_i}(p_i, x_i^{t+1})\right] + \tfrac{\eta_{t+1}\beta(t+1)}{2}\EE\left[\sum_{i \in \NCal} \|x_i^{t+1} - p_i\|^2\right] \leq \eta_t\EE\left[\sum_{i \in \NCal} \lambda_i \langle\nabla_i u_i(p), x_i^t - p_i\rangle\right]} \\
& & + \EE\left[\sum_{i \in \NCal} \lambda_i D_{R_i}(p_i, x_i^t)\right] + \tfrac{\eta_t\beta t}{2}\EE\left[\sum_{i \in \NCal} \|x_i^t - p_i\|^2\right] + 2\eta_t^2 L^2\left(\sum_{i \in \NCal} n_i^2\lambda_i \right) + \tfrac{\eta_t nL}{\beta^2\sqrt{t+1}}\left(\sum_{i \in \NCal} \lambda_i\right)\left(\sum_{i \in \NCal} \lambda_i^2\ell_i^2\right).
\end{eqnarray*}
Summing up the above inequality over $t = 1, 2, \ldots, T-1$ and using $\eta_t = \frac{1}{2nL\sqrt{t}}$, we have
\begin{eqnarray*}
\lefteqn{\EE\left[\sum_{i \in \NCal} \lambda_i D_{R_i}(p_i, x_i^T)\right] + \tfrac{\beta\sqrt{T}}{4nL}\EE\left[\sum_{i \in \NCal} \|x_i^T - p_i\|^2\right] \leq \sum_{t=1}^{T-1} \eta_t\EE\left[\sum_{i \in \NCal} \lambda_i \langle \nabla_i u_i(p), x_i^t - p_i\rangle\right]} \\
& & + \sum_{i \in \NCal} \lambda_i D_{R_i}(p_i, x_i^1) + \tfrac{\beta}{4nL}\left(\sum_{i \in \NCal} \|x_i^1 - p_i\|^2\right) + \tfrac{1}{2n^2}\left(\sum_{i \in \NCal} n_i^2\lambda_i \right)\left(\sum_{t=1}^{T-1} \tfrac{1}{t}\right) + \tfrac{1}{2\beta^2}\left(\sum_{i \in \NCal} \lambda_i\right)\left(\sum_{i \in \NCal} \lambda_i^2\ell_i^2\right)\left(\sum_{t=1}^{T-1} \tfrac{1}{t}\right).  
\end{eqnarray*}
Rearranging the above inequality and using $D_{R_i}(x_i^\star, x_i^T) \geq 0$ and $\|x_i^1 - p_i\| \leq B_i$, we have
\begin{eqnarray}\label{inequality:estm-third}
\lefteqn{\tfrac{\beta\sqrt{T}}{4nL}\EE\left[\sum_{i \in \NCal} \|x_i^T - p_i\|^2\right] \leq \sum_{t=1}^{T-1} \eta_t\EE\left[\sum_{i \in \NCal} \lambda_i\langle \nabla_i u_i(p), x_i^t - p_i\rangle\right]} \\
& & + \sum_{i \in \NCal} \lambda_i D_{R_i}(p_i, x_i^1) + \tfrac{\beta}{4nL}\left(\sum_{i \in \NCal} B_i^2\right) + \left(\tfrac{1}{2n^2}\left(\sum_{i \in \NCal} n_i^2\lambda_i \right) + \tfrac{1}{2\beta^2}\left(\sum_{i \in \NCal} \lambda_i\right)\left(\sum_{i \in \NCal} \lambda_i^2\ell_i^2\right)\right)\log(T). \nonumber 
\end{eqnarray}
By the initialization $x_i^1 = \argmin_{x \in \XCal} R_i(x)$, we have $\nabla R_i(x_i^1) = 0$ which implies that $D_{R_i}(p_i, x_i^1) = R_i(p_i) - R(x_i^1)$. Then, let us inspect each coordinate of a unique Nash equilibrium $x^\star$ and set $p$ coordinate by coordinate in Eq.~\eqref{inequality:estm-third}; indeed, we consider the following two cases: 
\begin{itemize}
\item A point $x_i^\star \in \XCal_i$ satisfies that $\pi_{x_i^1}(x_i^\star) \leq 1 - \tfrac{1}{\sqrt{T}}$. By Lemma~\ref{Lemma:SCB-upper-bound}, we have $D_{R_i}(x_i^\star, x_i^1) = R_i(x_i^\star) - R_i(x_i^1) \leq \nu_i\log(T)$. In this case, we set $p_i = x_i^\star$. 
\item A point $x_i^\star \in \XCal_i$ satisfies that $\pi_{x_i^1}(x_i^\star) > 1 - \tfrac{1}{\sqrt{T}}$. Thus, we can find $\bar{x}_i \in \XCal_i$ such that $\|\bar{x}_i - x_i^\star\| = O(\frac{1}{\sqrt{T}})$ and $\pi_{x_i^1}(\bar{x}_i) \leq 1 - \tfrac{1}{\sqrt{T}}$. By Lemma~\ref{Lemma:SCB-upper-bound}, we have $D_{R_i}(\bar{x}_i, x_i^1) = R_i(\bar{x}_i) - R_i(x_i^1) \leq \nu_i\log(T)$. In this case, we set $p_i = \bar{x}_i$. 
\end{itemize}
For simplicity, we denote the set of coordinates in the former case by $\ICal$ which implies that the set of coordinates in the latter case is $\NCal \setminus \ICal$. Then, Eq.~\eqref{inequality:estm-third} becomes 
\begin{eqnarray*}
\lefteqn{\tfrac{\beta\sqrt{T}}{4nL}\EE\left[\sum_{i \in \ICal} \|x_i^T - x_i^\star\|^2 + \sum_{i \in \NCal \setminus \ICal} \|x_i^T - \bar{x}_i\|^2\right]} \\
& \leq & \sum_{t=1}^{T-1} \eta_t\left(\EE\left[\sum_{i \in \ICal} \lambda_i \langle \nabla_i u_i(x_\ICal^\star, \bar{x}_{\NCal \setminus \ICal}), x_i^t - x_i^\star\rangle + \sum_{i \in \NCal \setminus \ICal} \lambda_i \langle \nabla_i u_i(x_\ICal^\star, \bar{x}_{\NCal \setminus \ICal}), x_i^t - \bar{x}_i\rangle\right]\right) + \tfrac{\beta}{4nL}\left(\sum_{i \in \NCal} B_i^2\right) \\ 
& & + \left(\tfrac{1}{2n^2}\left(\sum_{i \in \NCal} n_i^2\lambda_i \right) + \tfrac{1}{2\beta^2}\left(\sum_{i \in \NCal} \lambda_i\right)\left(\sum_{i \in \NCal} \lambda_i^2\ell_i^2\right) + \sum_{i \in \NCal} \nu_i\lambda_i\right)\log(T). 
\end{eqnarray*}
Note that 
\begin{equation*}
\EE\left[\sum_{i \in \NCal \setminus \ICal} \|x_i^T - x_i^\star\|^2\right] \leq \EE\left[\sum_{i \in \NCal \setminus \ICal} \|x_i^T - \bar{x}_i\|^2\right] + \EE\left[\sum_{i \in \NCal \setminus \ICal} \|x_i^T - x_i^\star\|\|x_i^\star - \bar{x}_i\|\right]. 
\end{equation*}
Since $\nabla_i u_i: \XCal_i \mapsto \br^{n_i}$ is $\ell_i$-Lipschitz continuous, we have
\begin{eqnarray*}
\lefteqn{\EE\left[\sum_{i \in \ICal} \lambda_i\langle \nabla_i u_i(x_\ICal^\star, \bar{x}_{\NCal \setminus \ICal}), x_i^t - x_i^\star\rangle + \sum_{i \in \NCal \setminus \ICal} \lambda_i\langle \nabla_i u_i(x_\ICal^\star, \bar{x}_{\NCal \setminus \ICal}), x_i^t - \bar{x}_i\rangle\right]} \\
& = & \EE\left[\sum_{i \in \NCal} \lambda_i \langle \nabla_i u_i(x^\star), x_i^t - x_i^\star\rangle\right] + \EE\left[\sum_{i \in \NCal} \lambda_i \langle \nabla_i u_i(x_\ICal^\star, \bar{x}_{\NCal \setminus \ICal}) - \nabla_i u_i(x^\star), x_i^t - x_i^\star\rangle\right] \\ 
& & + \EE\left[\sum_{i \in \NCal \setminus \ICal} \lambda_i \langle \nabla_i u_i(x_\ICal^\star, \bar{x}_{\NCal \setminus \ICal}), x_i^\star - \bar{x}_i\rangle\right] \\
& \leq & \EE\left[\sum_{i \in \NCal} \lambda_i \langle \nabla_i u_i(x^\star), x_i^t - x_i^\star\rangle\right] + \EE\left[\sum_{i \in \NCal} \lambda_i\|\nabla_i u_i(x_\ICal^\star, \bar{x}_{\NCal \setminus \ICal}) - \nabla_i u_i(x^\star)\|\|x_i^t - x_i^\star\|\right] \\
& & + \EE\left[\sum_{i \in \NCal \setminus \ICal} \lambda_i \|\nabla_i u_i(x_\ICal^\star, \bar{x}_{\NCal \setminus \ICal})\|\|x_i^\star - \bar{x}_i\|\right] \\
& \leq & \EE\left[\sum_{i \in \NCal} \lambda_i \langle \nabla_i u_i(x^\star), x_i^t - x_i^\star\rangle\right] + \EE\left[\left(\sum_{i \in \NCal} \lambda_i \ell_i \|\|x_i^t - x_i^\star\|\right)\left(\sum_{i \in \NCal \setminus \ICal} \|x_i^\star - \bar{x}_i\| \right)\right] \\
& & + \EE\left[\sum_{i \in \NCal} \lambda_i \|\nabla_i u_i(x_\ICal^\star, \bar{x}_{\NCal \setminus \ICal})\|\|x_i^\star - \bar{x}_i\|\right]. 
\end{eqnarray*}
Since $x^\star \in \XCal$ is a Nash equilibrium and $x^t \in \XCal$, we have $\langle \nabla_i u_i(x^\star), x_i^t - x_i^\star\rangle \leq 0$ for all $i \leq \NCal$ (cf. Proposition~\ref{prop:CG}). Putting these pieces together yields that 
\begin{eqnarray*}
\lefteqn{\EE\left[\sum_{i \in \NCal} \|x_i^T - x_i^\star\|^2\right] \leq \EE\left[\sum_{i \in \NCal \setminus \ICal} \|x_i^T - x_i^\star\|\|x_i^\star - \bar{x}_i\|\right]}  \\
& & + \tfrac{4nL}{\beta\sqrt{T}} \sum_{t=1}^{T-1} \eta_t\left(\EE\left[\left(\sum_{i \in \NCal} \lambda_i \ell_i \|\|x_i^t - x_i^\star\|\right)\left(\sum_{i \in \NCal \setminus \ICal} \|x_i^\star - \bar{x}_i\| \right) + \sum_{i \in \NCal} \lambda_i \|\nabla_i u_i(x_\ICal^\star, \bar{x}_{\NCal \setminus \ICal})\|\|x_i^\star - \bar{x}_i\| \right]\right) \\ 
& & + \tfrac{1}{\sqrt{T}}\left(\sum_{i \in \NCal} B_i^2\right) + \tfrac{4nL}{\beta\sqrt{T}}\left(\tfrac{1}{2n^2}\left(\sum_{i \in \NCal} n_i^2\lambda_i \right) + \tfrac{1}{2\beta^2}\left(\sum_{i \in \NCal} \lambda_i\right)\left(\sum_{i \in \NCal} \lambda_i^2\ell_i^2\right) + \sum_{i \in \NCal} \nu_i\lambda_i\right)\log(T). 
\end{eqnarray*}
Since $\|x_i^\star - \bar{x}_i\| = O(\tfrac{1}{\sqrt{T}})$ and $\|x_i^t - x_i^\star\| \leq B_i$ for all $i \in \NCal$ and all $t = 1, 2, \ldots, T$, and $\eta_t = \frac{1}{2nL\sqrt{t}}$ for all $t = 1, 2, \ldots, T$, we have
\begin{equation*}
\EE\left[\sum_{i \in \NCal \setminus \ICal} \|x_i^T - x_i^\star\|\|x_i^\star - \bar{x}_i\|\right] = O(T^{-1/2}), 
\end{equation*}
and 
\begin{equation*}
\tfrac{4nL}{\beta\sqrt{T}} \sum_{t=1}^{T-1} \eta_t\left(\EE\left[\left(\sum_{i \in \NCal} \lambda_i \ell_i \|\|x_i^t - x_i^\star\|\right)\left(\sum_{i \in \NCal \setminus \ICal} \|x_i^\star - \bar{x}_i\| \right)\right]\right) = O(T^{-1/2}). 
\end{equation*}
Since $\sum_{i \in \NCal} n_i^2\lambda_i \leq n^2 (\max_{i \in \NCal} \lambda_i)$, we have
\begin{equation*}
\tfrac{4nL}{\beta\sqrt{T}}\left(\tfrac{1}{2n^2}\left(\sum_{i \in \NCal} n_i^2\lambda_i \right) + \tfrac{1}{2\beta^2}\left(\sum_{i \in \NCal} \lambda_i\right)\left(\sum_{i \in \NCal} \lambda_i^2\ell_i^2\right) + \sum_{i \in \NCal} \nu_i\lambda_i\right)\log(T) = \tilde{O}(nT^{-1/2}). 
\end{equation*}
Since $\nabla_i u_i$ is Lipschitz continuous, $x^\star, \bar{x} \in \XCal$ where $\XCal$ is compact and bounded, $\|x_i^\star - \bar{x}_i\| = O(\tfrac{1}{\sqrt{T}})$ and $\eta_t = \frac{1}{2nL\sqrt{t}}$ for all $t = 1, 2, \ldots, T$, we have 
\begin{equation*}
\tfrac{4nL}{\beta\sqrt{T}} \sum_{t=1}^{T-1} \eta_t\left(\EE\left[\sum_{i \in \NCal} \lambda_i \|\nabla_i u_i(x_\ICal^\star, \bar{x}_{\NCal \setminus \ICal})\|\|x_i^\star - \bar{x}_i\| \right]\right) = O(T^{-1/2}). 
\end{equation*}
In addition, we have $\frac{1}{\sqrt{T}}\left(\sum_{i \in \NCal} B_i^2\right) = O(T^{-1/2})$. Putting these pieces together yields that 
\begin{equation*}
\EE\left[\sum_{i \in \NCal}\|x_i^T - x_i^\star\|^2\right] = \tilde{O}(nT^{-1/2}). 
\end{equation*}
By the definition of $\hat{x}^t$, we have 
\begin{equation*}
\EE\left[\sum_{i \in \NCal}\|\hat{x}_i^T - x_i^\star\|^2\right] \leq 2\EE\left[\sum_{i \in \NCal}\|\hat{x}_i^T - x_i^T\|^2\right] + 2\EE\left[\sum_{i \in \NCal}\|x_i^T - x_i^\star\|^2\right]. 
\end{equation*}
By the update formula of $A_i^t$, we have
\begin{equation*}
\EE\left[\sum_{i \in \NCal}\|\hat{x}_i^T - x_i^T\|^2\right] \leq \EE\left[\sum_{i \in \NCal} (\sigma_{\max}(A_i^T))^2\right] \leq \tfrac{2nL(\sum_{i \in \NCal} \lambda_i)}{\beta\sqrt{T}}. 
\end{equation*}
Putting these pieces together yields that 
\begin{equation*}
\EE\left[\sum_{i \in \NCal}\|\hat{x}_i^T - x_i^\star\|^2\right] \leq \tfrac{4nL(\sum_{i \in \NCal} \lambda_i)}{\beta\sqrt{T}} + 2\EE\left[\sum_{i \in \NCal}\|x_i^T - x_i^\star\|^2\right] = \tilde{O}(nT^{-1/2}). 
\end{equation*}
This concludes the proof. 

\section{Imperfect Bandit Feedback}\label{app:imperfect}
We turn to the setting with imperfect bandit feedback and prove that Algorithm~\ref{Alg:MA} still achieves the last-iterate convergence to a unique Nash equilibrium at a rate of $\tilde{O}(nT^{-1/2})$. In particular, we assume that each player $i \in \NCal$ has access to a ``black box" feedback mechanism --- \textit{an oracle} --- which returns an estimate of their reward functions at their current action profile. This information can be imperfect for a multitude of reasons: for instance, (i) estimates may be susceptible to random measurement errors; (ii) the transmission of this information could be subject to noise; and/or (iii) the game's reward functions may be stochastic expectations of the form
\begin{equation*}
u_i(x) = \EE[\hat{u}_i(x, \omega)] \quad \textnormal{for some random variable } \omega. 
\end{equation*}
With all this in mind, we consider Algorithm~\ref{Alg:MA} with imperfect bandit feedback; that is, we replace Step 8 in Algorithm~\ref{Alg:MA} by $\hat{u}_i^t \leftarrow u_i(\hat{x}^t) + \xi_{i, t}$ for all $i \in \NCal$ where $\xi_t = (\xi_{i, t})_{i \in \NCal}$ is a bounded martingale difference adapted to the filtration generated by $\hat{x}^t$ (i.e., $\xi_{t-1}$ is measurable with respect to $x^t$ but $\xi_t$ is not). More specifically, $(\xi_t)_{t \geq 1}$ satisfies the statistical hypotheses (almost surely): (i) \textit{Zero-mean:} $\EE[\xi_{i, t} \mid x^t] = 0$; and (ii) \textit{Boundedness:} $\|\xi_{i, t}\| \leq \sigma$ for some constant $\sigma \geq 0$. Alternatively, the above conditions simply posit that the players' individual reward estimates are conditionally unbiased and the difference remains bounded, i.e., 
\begin{equation*}
\EE[\hat{u}_i^t \mid x^t] = u_i(\hat{x}^t), \quad \|\hat{u}_i^t - u_i(\hat{x}^t)\| \leq \sigma, \quad \textnormal{for all } i \in \NCal.  
\end{equation*}
The above conditions are satisfied by a broad range of error processes, including all compactly supported distributions (In particular, we will not be assuming i.i.d. errors; this point is crucial for applications to distributed control where measurements are typically correlated with the state of the system); further, many applications in operations research, control, network theory, and machine learning can be used as motivations for the bounded noise framework. 
\begin{theorem}\label{Thm:last-iterate-imperfect}
Suppose that $x^\star \in \XCal$ is a unique Nash equilibrium of a smooth and $(\beta, \{\lambda_i\}_{i \in \NCal})$-strongly monotone game. Each function $u_i$ satisfies that $|u_i(x)| \leq L$ for all $x \in \XCal$. If $T \geq 1$ is fixed and each player employs Algorithm~\ref{Alg:MA} with imperfect bandit feedback where $\xi_{i, t} > 0$ satisfies the above statistical hypotheses (almost surely) and the stepsize choice of $\eta_t = \tfrac{1}{2n(L+\sigma)\sqrt{t}}$, we have
\begin{equation*}
\EE\left[\sum_{i \in \NCal}\|\hat{x}_i^T - x_i^\star\|^2\right] = \tilde{O}\left(\sqrt{\frac{n^2}{T}}\right). 
\end{equation*}
\end{theorem}
\begin{proof}
The key is to establish a descent lemma for the iterates generated by Algorithm~\ref{Alg:MA} with imperfect bandit feedback and $\eta_t = \tfrac{1}{2n(L+\sigma)\sqrt{t}}$. 
\begin{lemma}\label{Lemma:MA-imperfect-main}
Suppose that the iterate $\{x^t\}_{t \geq 1}$ is generated by Algorithm~\ref{Alg:MA} with imperfect bandit feedback and let each function $u_i$ satisfy that $|u_i(x)| \leq L$ for all $x \in \XCal$ and $0 < \eta_t \leq \frac{1}{2n(L+\sigma)}$, we have
\begin{eqnarray*}
\lefteqn{\sum_{i \in \NCal} \lambda_i D_{R_i}(p_i, x_i^{t+1}) + \tfrac{\eta_t\beta(t+1)}{2}\left(\sum_{i \in \NCal} \|x_i^{t+1} - p_i\|^2\right)} \\
& \leq & \sum_{i \in \NCal} \lambda_i D_{R_i}(p_i, x_i^t) + \tfrac{\eta_t\beta(t+1)}{2}\left(\sum_{i \in \NCal} \|x_i^t - p_i\|^2\right) + 2\eta_t^2\left(\sum_{i \in \NCal} \lambda_i\|A_i^t\hat{v}_i^t\|^2\right) + \eta_t\left(\sum_{i \in \NCal} \lambda_i\langle\hat{v}_i^t, x_i^t - p_i\rangle \right).
\end{eqnarray*}
where $p_i \in \XCal_i$ and the sequence $\{\eta_t\}_{t \geq 1}$ is assumed to be non-increasing. 
\end{lemma}
\begin{proof}
Using the same argument as in Lemma~\ref{Lemma:MA-perfect-main}, we have
\begin{eqnarray}\label{inequality:MA-imperfect-main}
\lefteqn{\sum_{i \in \NCal} \lambda_i D_{R_i}(p_i, x_i^{t+1}) + \tfrac{\eta_t \beta(t+1)}{2}\left(\sum_{i \in \NCal} \|x_i^{t+1} - p_i\|^2\right)} \\
& \leq & \sum_{i \in \NCal} \lambda_i D_{R_i}(p_i, x_i^t) + \tfrac{\eta_t\beta(t+1)}{2}\left(\sum_{i \in \NCal} \|x_i^t - p_i\|^2\right) + \eta_t \left(\sum_{i \in \NCal} \lambda_i\langle\hat{v}_i^t, x_i^{t+1} - x_i^t\rangle\right) + \eta_t\left(\sum_{i \in \NCal} \lambda_i\langle\hat{v}_i^t, x_i^t - p_i\rangle \right). \nonumber
\end{eqnarray}
Next, we bound the term $\langle\hat{v}_i^t, x_i^{t+1} - x_i^t\rangle$ using Lemma~\ref{Lemma:SC-key-estimate}. Indeed, we define the function $g$ by 
\begin{equation*}
g(x_i) = \eta_t \langle\hat{v}_i^t, x_i^t - x_i\rangle + \tfrac{\eta_t \beta(t+1)}{2\lambda_i}\|x_i^t - x_i\|^2 + D_{R_i}(x_i, x_i^t). 
\end{equation*}
and apply the same argument as used in Lemma~\ref{Lemma:SA-main} and Lemma~\ref{Lemma:MA-perfect-main}. The only difference is that we have a new upper bound for the term $\|A_i^t\hat{v}_i^t\|$ here: 
\begin{equation*}
\|\hat{v}_i^t\|_{x_i^t, \star} = \|A_i^t\hat{v}_i^t\| =  n|u_i(x^t + A^t z^t) + \xi_{i, t}|\|z_i^t\| \leq n|u_i(x^t + A^t z^t) + \xi_{i, t}|. 
\end{equation*}
Combining the above inequality with the fact that $0 < \eta_t \leq \tfrac{1}{2n(L+\sigma)}$ gives $\lambda(x^t, g) \leq \tfrac{1}{2}$. Therefore, by Lemma~\ref{Lemma:SC-key-estimate}, we have $\|x_i^{t+1} - x_i^t\|_{x_i^t} \leq 2\eta_t\|\hat{v}_i^t\|_{x_i^t, \star}$. This together with the H\"{o}lder's inequality yields that
\begin{equation*}
\langle\hat{v}_i^t, x_i^{t+1} - x_i^t\rangle \leq \|\hat{v}_i^t\|_{x_i^t, \star}\|x_i^{t+1} - x_i^t\|_{x_i^t} \leq 2\eta_t\|\hat{v}_i^t\|_{x_i^t, \star}^2 = 2\eta_t\|A_i^t\hat{v}_i^t\|^2.  
\end{equation*}
Plugging the above inequality into Eq.~\eqref{inequality:MA-imperfect-main} yields the desired inequality. 
\end{proof}
Since the inequality in Lemma~\ref{Lemma:MA-imperfect-main} is the same as that derived in Lemma~\ref{Lemma:MA-perfect-main}, we can prove Theorem~\ref{Thm:last-iterate-imperfect} by applying the same argument used in Theorem~\ref{Thm:last-iterate-perfect}.
\end{proof}
\begin{remark}
Theorem~\ref{Thm:last-iterate-imperfect} shows that Algorithm~\ref{Alg:MA} can achieve the near-optimal rate of last-iterate convergence in a more challenging environment in which the bandit feedback is only imperfect with zero-mean and bounded noises. Can we generalize Algorithm~\ref{Alg:MA} to handle more general noises, e.g., the one with bounded variance? We leave the answers to future work.
\end{remark}

\section{Additional Experimental Results on Two-Player Zero-Sum Games}
We consider the minimax optimization problem given by 
\begin{equation}\label{prob:L2-minimax}
\min_{x_1 \in \XCal_1} \max_{x_2 \in \XCal_2} \ f(x_1, x_2) = \mu\|x_1\|^2 + x_1^\top A x_2 - \mu\|x_2\|^2, 
\end{equation}
where $A \in \br^{n \times n}$ is a sparse random matrix where each element is nonzero independently with probability $p = 0.5$ and is chosen uniformly from $[0, 1]$ if it is chosen to be nonzero. The constraint sets $\XCal_1$ and $\XCal_2$ are defined as $\XCal_1 =\{x_1 \in \br^n: \one_n^\top x_1 \leq B, x_1 \geq 0\}$ and $\XCal_2 = \{x_2 \in \br^n: \one_n^\top x_2 \leq B, x_2 \geq 0\}$. In our experiment, we set $\mu = 0.01$ and vary $n \in \{10, 20, 50, 100\}$ and $B \in \{0.5, 1\}$. 

Clearly, the problem in Eq.~\eqref{prob:L2-minimax} is a two-player zero-sum game where the decision variables of the first and second players are $x_1$ and $x_2$ and the corresponding model for the reward of the first and second players is given by 
\begin{equation*}
u_1(x_1, x_2) = - f(x_1, x_2), \quad u_2(x_1, x_2) = f(x_1, x_2). 
\end{equation*}
The resulting game $\GCal$ is $(\beta, \{\lambda_1, \lambda_2\})$-strongly monotone for $\beta = 2\mu$ and $\lambda_i = 1$ for all $i \in \{1, 2\}$. 
\begin{figure*}[!t]
\centering
\includegraphics[width=0.45\textwidth]{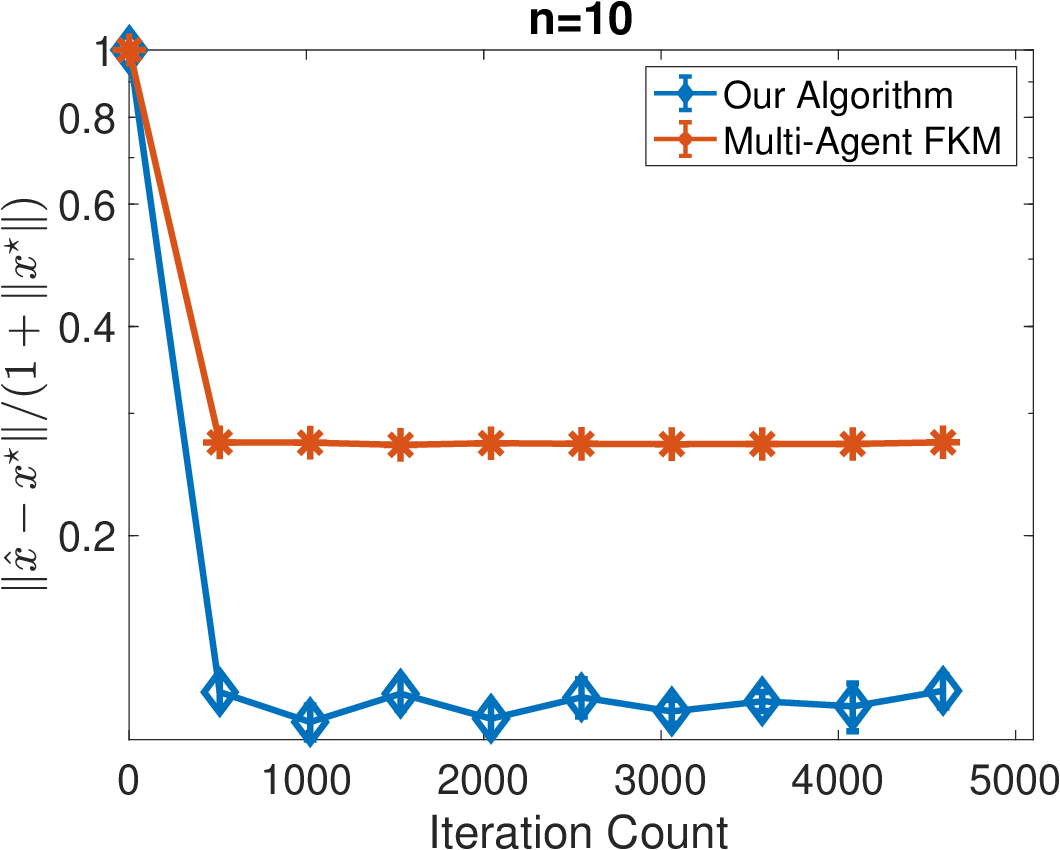}
\includegraphics[width=0.45\textwidth]{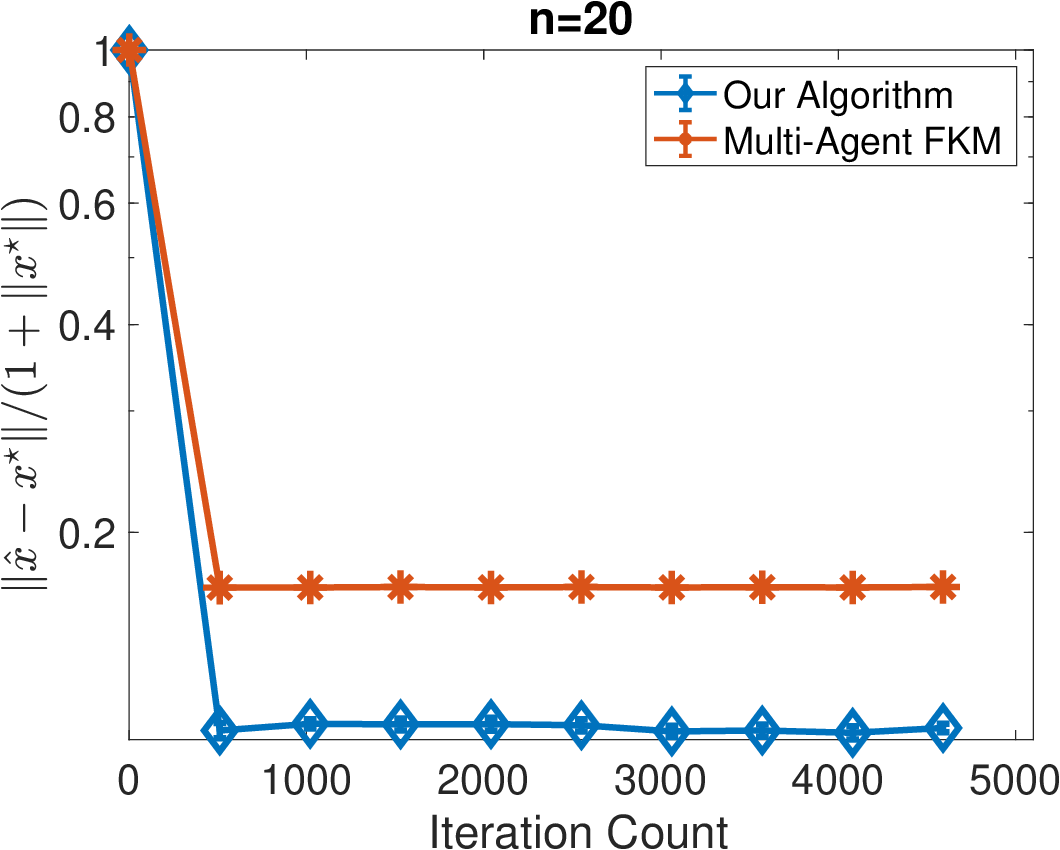}
\includegraphics[width=0.45\textwidth]{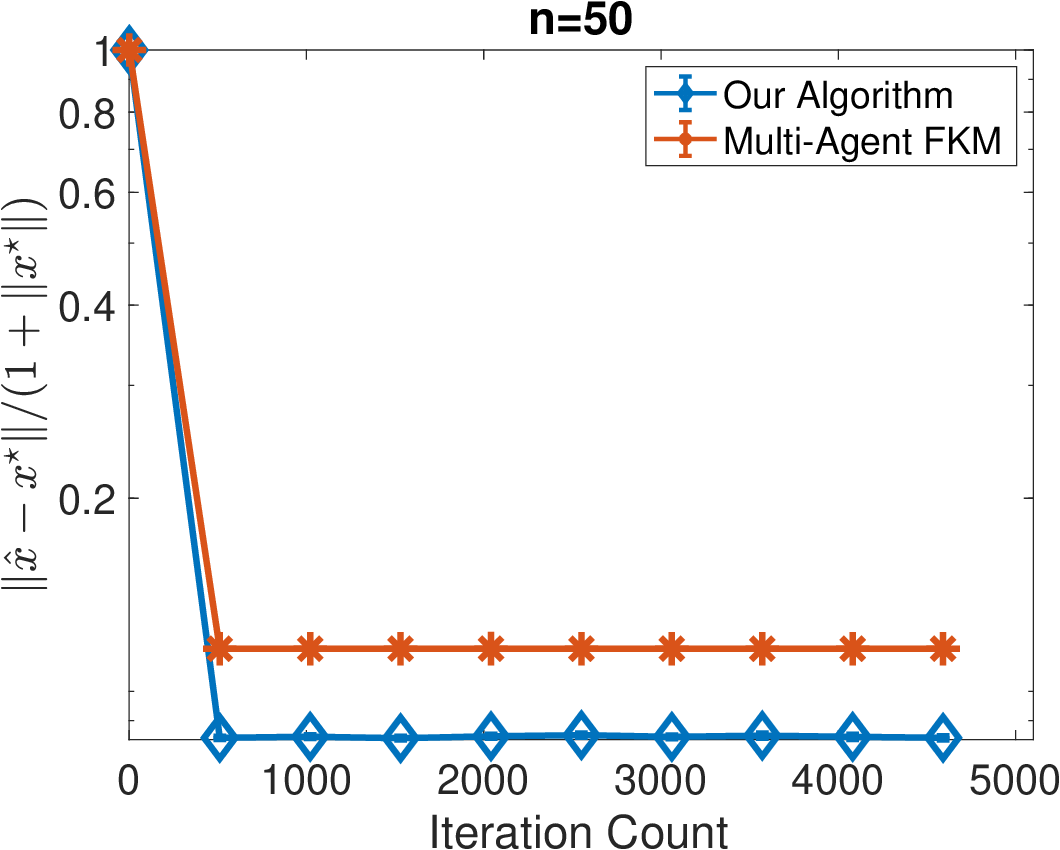}
\includegraphics[width=0.45\textwidth]{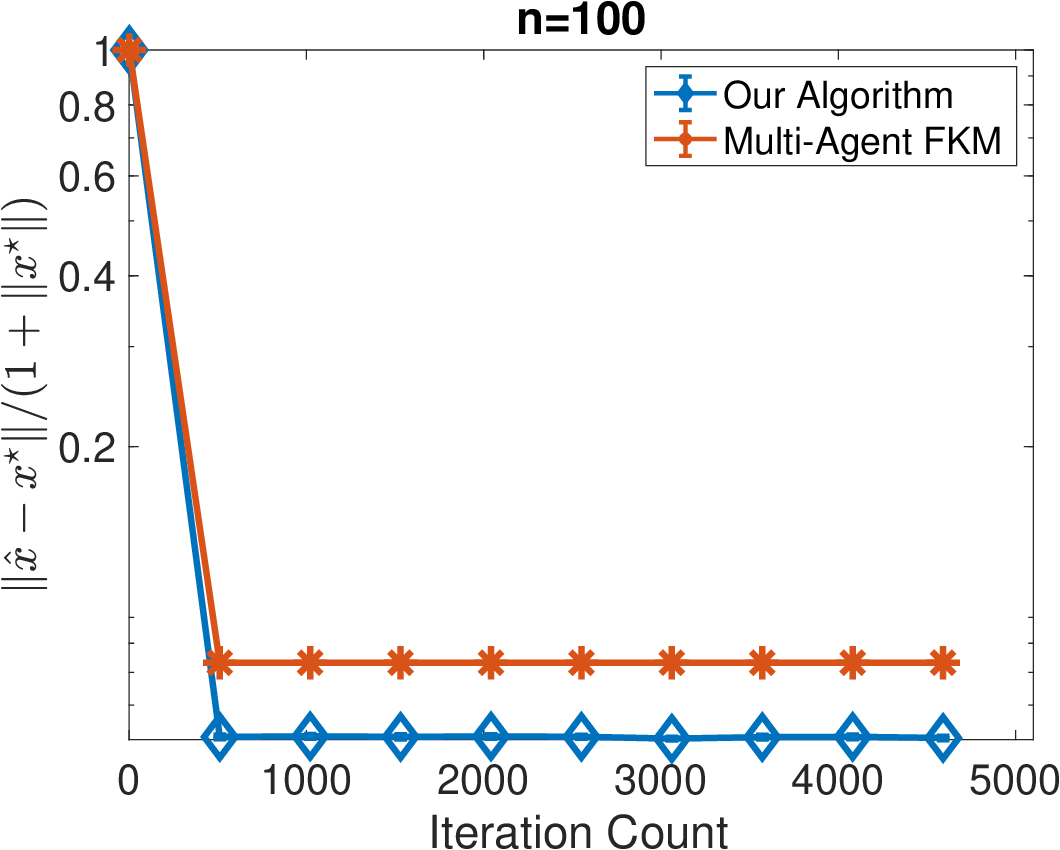}
\caption{Performance of both algorithms for solving two-player zero-sum games. The number of independent trials is 10 and $B=1$ is fixed. The numerical results are presented as error v.s. iteration count.}\label{Fig:MM}
\end{figure*}

For both Algorithm~\ref{Alg:MA} and multi-agent FKM, we consider the theoretically-correct choices of step sizes without fine-tuning. Indeed, we set $\lambda_i = 1$ for all $i \in \NCal$ and $\beta = 2\mu$ for Algorithm~\ref{Alg:MA} due to the structure of a game. Since $\XCal_i =\{x_i \in \br^n: \one_n^\top x_i \leq B, x_i \geq 0\}$, we set $R_i(x_i) = - \sum_{j=1}^n \log(x_{ij}) -\log(B - \sum_{j=1}^n x_{ij})$. According to Theorem~\ref{Thm:last-iterate-perfect}, we set $\eta_t = \frac{1}{4n\sqrt{t}}$. For $i \in \{1, 2\}$, we set $r_i = \frac{B}{n(n+1)}$ and $p_i = \frac{B}{n}\one_n$ in multi-agent FKM. We also set $\delta_t = \min\{\frac{B}{n(n+1)}, \frac{1}{t^{1/3}}\}$ and $\gamma_t = \frac{1}{3\beta t}$ according to~\citet[Theorem~5.2]{Bravo-2018-Bandit}. The evaluation metric is $\frac{\|\hat{x}^t - x^\star\|}{1 + \|x^\star\|}$ where $\hat{x}^t$ is generated by the algorithms and $x^\star$ is a unique Nash equilibrium (we obtain it by employing the optimistic gradient method to find an approximate solution of the optimization problem in Eq.~\eqref{prob:L2-minimax}). This point is a benchmark for evaluating the quality of the solution obtained by the algorithms.
\begin{table}
\centering
\caption{The solution quality on two-player zero-sum games after 5000 iterations.}\label{Tab:MM}
\begin{tabular}{|c||c|c|} \hline
$(n, B)$  	& Multi-Agent FKM 					& Our Algorithm \\ \hhline{===} 
(10, 0.5) 	& 1.3e-01 $\pm$ 3.4e-03 		& 6.6e-02 $\pm$ 4.5e-03 \\ 
(10, 1.0) 	& 2.7e-01 $\pm$ 4.8e-03 			& 1.1e-01 $\pm$ 1.5e-02 \\ 
(20, 0.5) 	& 9.4e-02 $\pm$ 3.5e-04 		& 5.6e-02 $\pm$ 3.0e-03 \\ 
(20, 1.0) 	& 1.7e-01 $\pm$ 8.1e-04 			& 1.0e-01 $\pm$ 4.1e-03 \\ 
(50, 0.5) 	& 5.4e-02 $\pm$ 1.1e-03 			& 4.3e-02 $\pm$ 1.1e-03 \\ 
(50, 1.0) 	& 1.2e-01 $\pm$ 9.7e-05 			& 8.4e-02 $\pm$ 1.6e-03 \\ 
(100, 0.5) 	& 3.6e-02 $\pm$ 9.5e-05 		& 3.2e-02 $\pm$ 8.1e-04 \\ 
(100, 1.0) 	& 8.3e-02 $\pm$ 4.3e-05 		& 6.2e-02 $\pm$ 5.2e-04 \\ \hline
\end{tabular}
\end{table}
\paragraph{Experimental results.} Fixing $B = 1$, we investigate the convergence behavior of both algorithms with a varying dimension $n \in \{10, 20, 50, 100\}$. Figure~\ref{Fig:MM} indicates that our algorithm outperforms the multi-agent FKM as it returns the iterates that are closer to a unique Nash equilibrium in fewer iterations. We also present the numerical results for all $(n, B)$ in Table~\ref{Tab:MM}. 

\section{Additional Experimental Results on Strongly Concave Potential Game}
We reformulate the problem of distributed $\ell_2$-regularized logistic regression as a strongly monotone potential game and compare Algorithm~\ref{Alg:MA} and multi-agent FKM on this task. More specifically, the $\ell_2$-regularized logistic regression problem aims at minimizing the sum of a logistic loss function and a squared $\ell_2$-norm-based regularization term as follows, 
\begin{equation}\label{prob:L2-logistic}
\min_{x \in \XCal} \ f(x) = \tfrac{1}{m}\left(\sum_{j=1}^m \log\left(1 + \exp(-b_j \cdot a_j^\top x)\right)\right) + \mu\|x\|^2, 
\end{equation}
where $(a_i, b_i)_{i=1}^m$ is a set of data samples with $b_i \in \{-1, 1\}$ and $\mu > 0$ is the regularization parameter. The squared $\ell_2$-norm prevents overfitting issue and $\mu > 0$ balances goodness-of-fit and generalization. The constraint set $\XCal$ is defined as $\XCal =\{x \in \br^n: \|x\|_\infty \leq \frac{1}{m}\}$. In real applications, the dimension $n > 0$ can be extremely large and the key challenge is its computation. To resolve this issue, some practitioners suggest to solve the problem in a distributed manner~\citep{Gopal-2013-Distributed}. 

We intend to conduct such distributed computation using multi-agent learning framework and ensure that the unique Nash equilibrium of a game is the same as an optimal solution of Eq.~\eqref{prob:L2-logistic}. Indeed, we assume a finite set of $\NCal = \{1, 2, \ldots, m\}$ of players and let the decision variable of the $i^\textnormal{th}$ player be $x_i \in \br$. Suppose that we have a shared memory such that all players can access a set of data samples $(a_i, b_i)_{i=1}^m$. A simple model for the reward of the $i^\textnormal{th}$ player is given by 
\begin{equation*}
u_i(x) = - f(x).
\end{equation*}
Let $\XCal_i = [-\frac{1}{m}, \frac{1}{m}]$ be the space of possible decisions of the $i^\textnormal{th}$ player, the resulting game $\GCal \equiv \GCal(\NCal, \XCal, u)$ is a potential game with a $2\mu$-strongly concave potential function $-f$ (see Example~\ref{Example:SCPG}). Therefore, the game $\GCal$ is $(\beta, \{\lambda_i\}_{i \in \NCal})$-strongly monotone for $\beta = 2\mu$ and $\lambda_i = 1$ for all $i \in \NCal$. 
\begin{table}[!t]
\caption{Statistics of datasets.}\label{tab:data}
\centering
\begin{tabular}{|c|c|c|} \hline
Dataset 					& Number of Samples ($m$) 	& Dimension ($n$) \\ \hline  
\textsf{a9a} 				& 32561 								& 123 \\
\textsf{mushrooms} 	& 8124 									& 112 \\   
\textsf{news20} 		& 16242 								& 100 \\
\textsf{splice} 			& 1000 									& 60 \\ 
\textsf{svmguide3} 	& 1243 									& 21 \\
\textsf{w8a} 				& 64700 								& 300 \\ \hline
\end{tabular}
\end{table}
\paragraph{Experimental setup.} We use 5 LIBSVM datasets (https://www.csie.ntu.edu.tw/\~{}cjlin/libsvm/) and 20 newsgroup dataset (https://www.cs.nyu.edu/roweis/data.html) for our experiment, and set $\mu = 0.001$ and $\ell = \tfrac{1}{4}(\max_{1 \leq j \leq m} \|a_j\|^2)$ for each dataset. For both Algorithm~\ref{Alg:MA} and multi-agent FKM, we consider the theoretically-correct choices of step sizes without fine-tuning. Indeed, we set $\lambda_i = 1$ for all $i \in \NCal$ and $\beta = 2\mu$ for Algorithm~\ref{Alg:MA} due to the structure of a game. Since $\XCal_i = [-\frac{1}{m}, \frac{1}{m}]$, we set $R_i(x_i) = -\log(\frac{1}{m}+x_i) - \log(\frac{1}{m}-x_i)$. According to Theorem~\ref{Thm:last-iterate-perfect}, we set $\eta_t = \min\{\frac{1}{2m\sqrt{t}}, \frac{1}{2\ell}\}$ where $\frac{1}{2\ell}$ is set due to an upper bound derived for no-regret gradient-based learning algorithms~\citep{Mertikopoulos-2019-Learning}. For all $i \in \NCal$, we set $r_i = \frac{1}{m}$ and $p_i = 0$ in multi-agent FKM since $\XCal_i = [-\frac{1}{m}, \frac{1}{m}]$. We also set $\delta_t = \min\{\frac{1}{m}, \frac{1}{t^{1/3}}\}$ and $\gamma_t = \min\{\frac{1}{3\beta t}, \frac{1}{2\ell}\}$ according to a combination of theoretically-correct choices from~\citet[Theorem~5.2]{Bravo-2018-Bandit} and the aforementioned strategy. The evaluation metric is $\frac{\|\hat{x}^t - x^\star\|}{1 + \|x^\star\|}$ where $\hat{x}^t$ is generated by the algorithms and $x^\star$ is an approximate Nash equilibrium with high accuracy (we obtain it by employing accelerated gradient descent~\citep[Section~2.2]{Nesterov-2018-Lectures} for minimizing a function $f$ in Eq.~\eqref{prob:L2-logistic}). This point will be a benchmark for evaluating the quality of the solution obtained by Algorithm~\ref{Alg:MA} and multi-agent FKM.
\begin{figure*}[!t]
\centering
\includegraphics[width=0.32\textwidth]{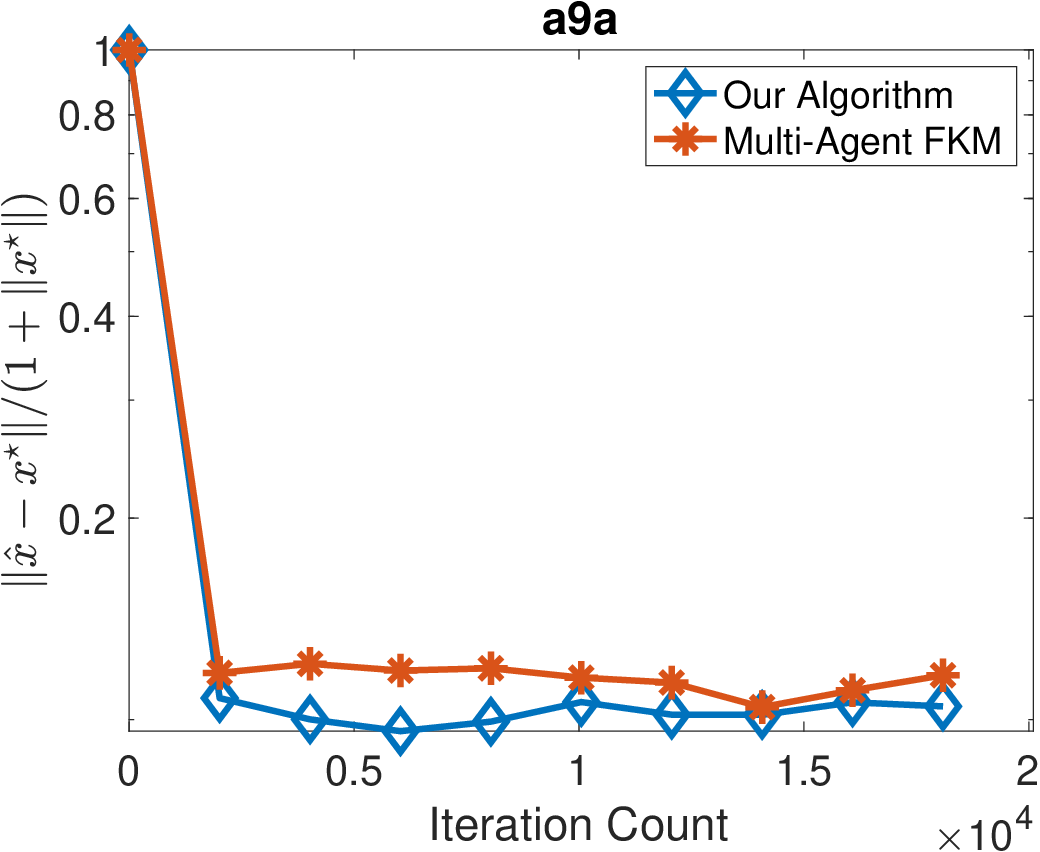}
\includegraphics[width=0.32\textwidth]{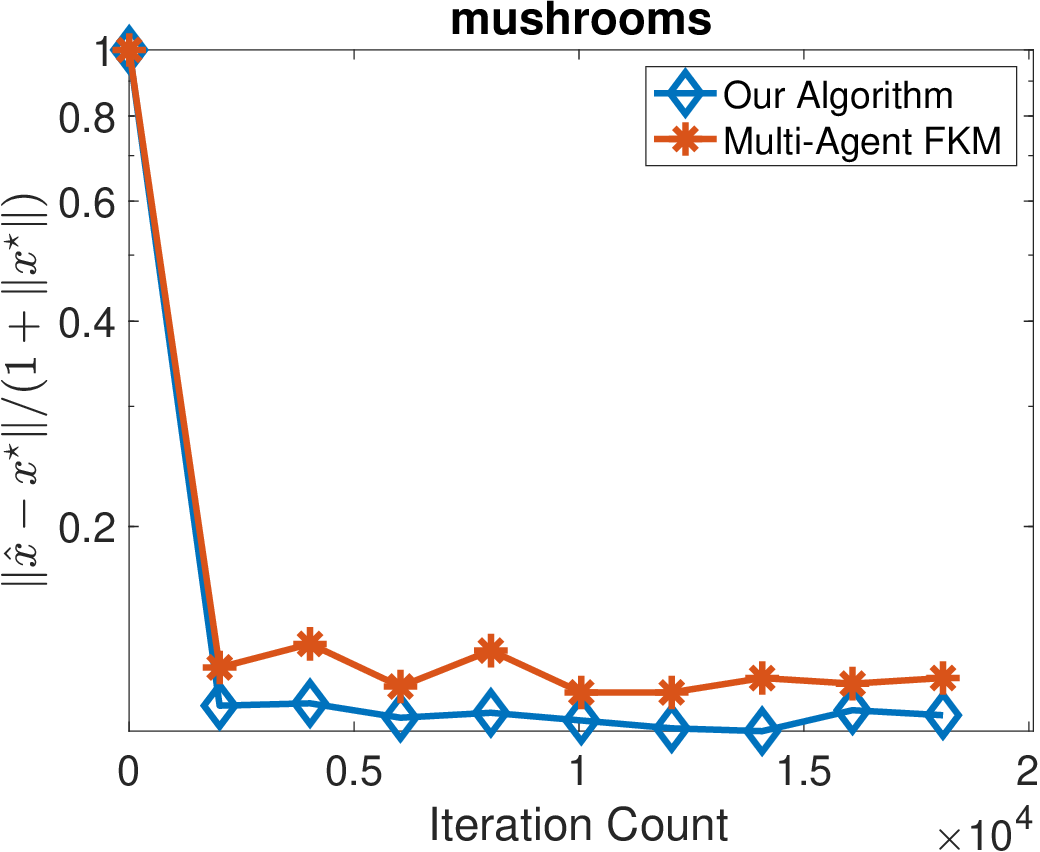}
\includegraphics[width=0.32\textwidth]{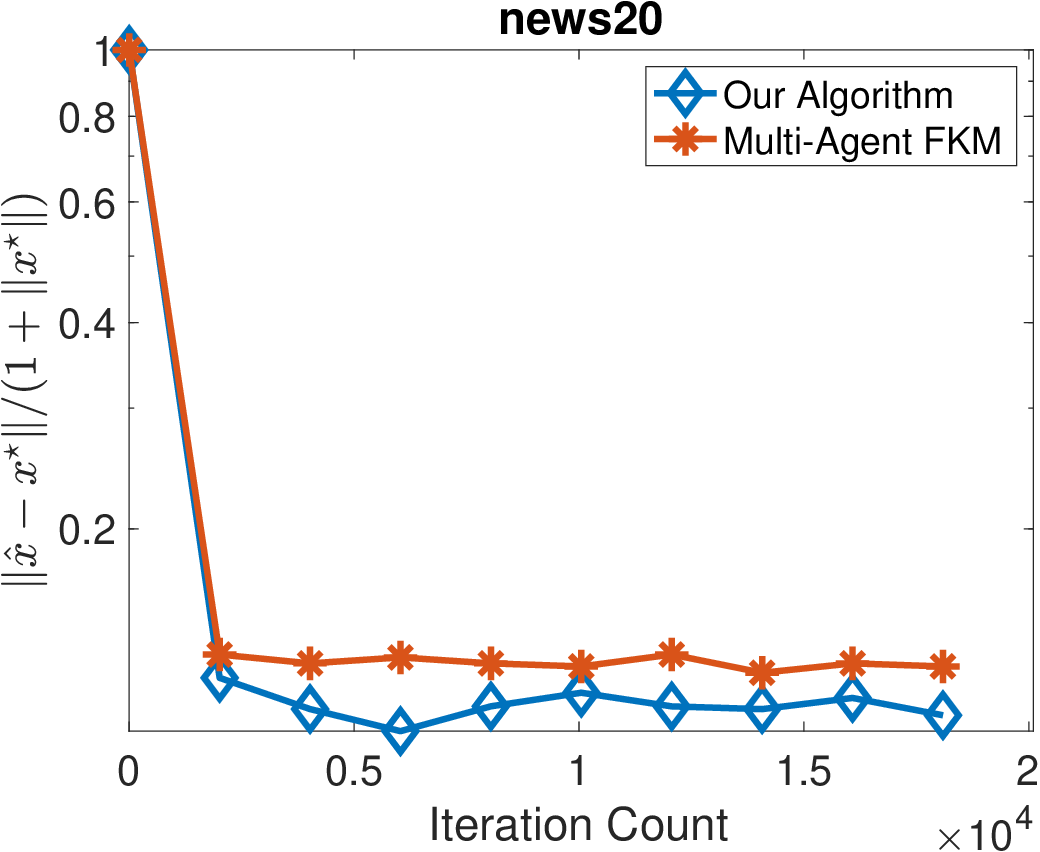} \\ \vspace{5pt}

\includegraphics[width=0.32\textwidth]{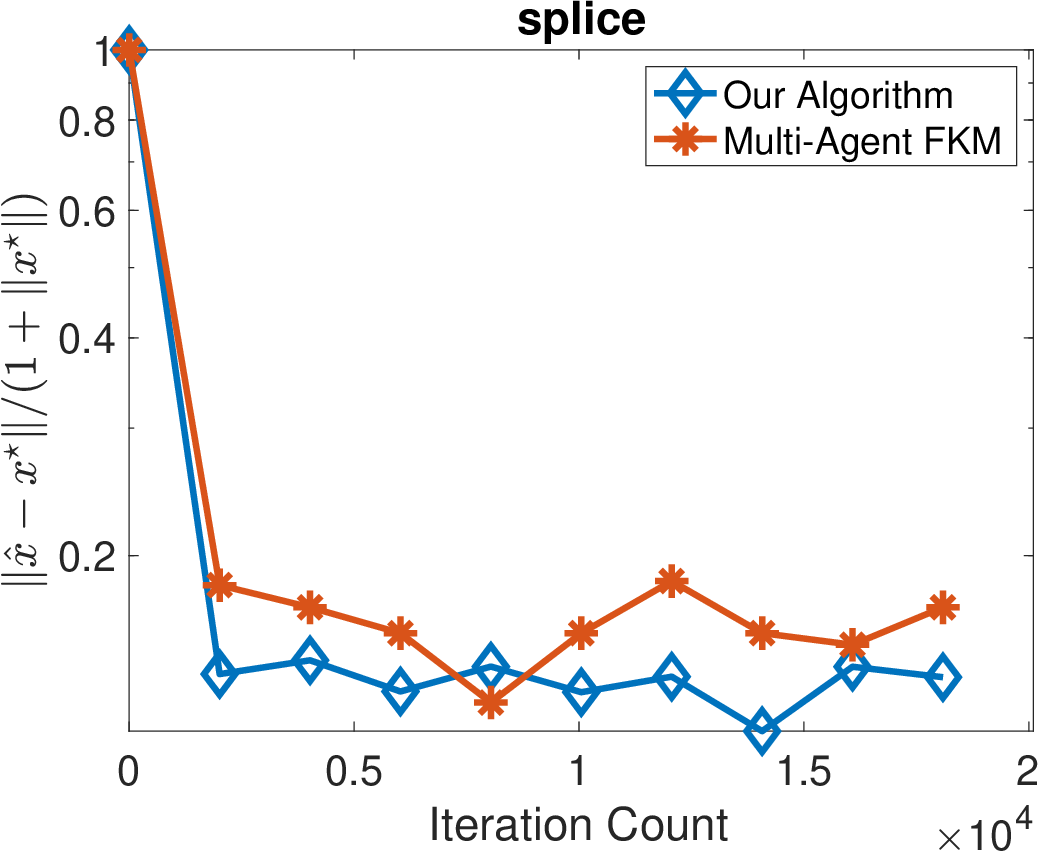}
\includegraphics[width=0.32\textwidth]{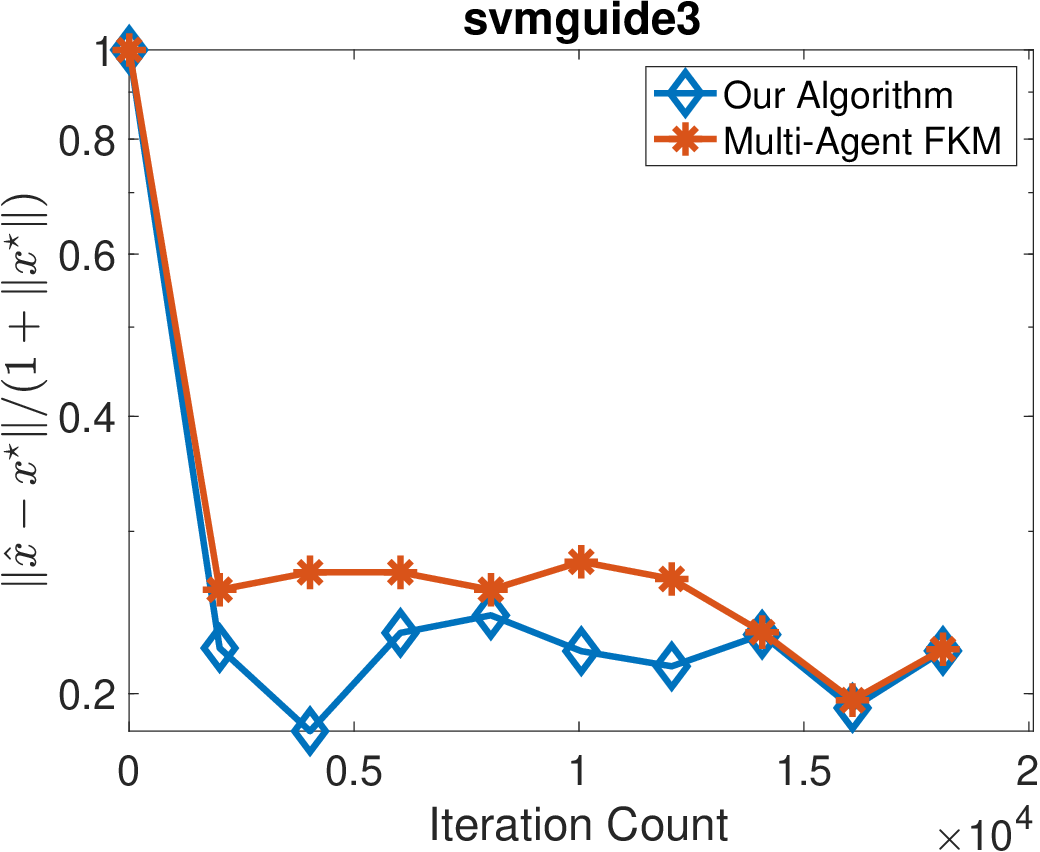}
\includegraphics[width=0.32\textwidth]{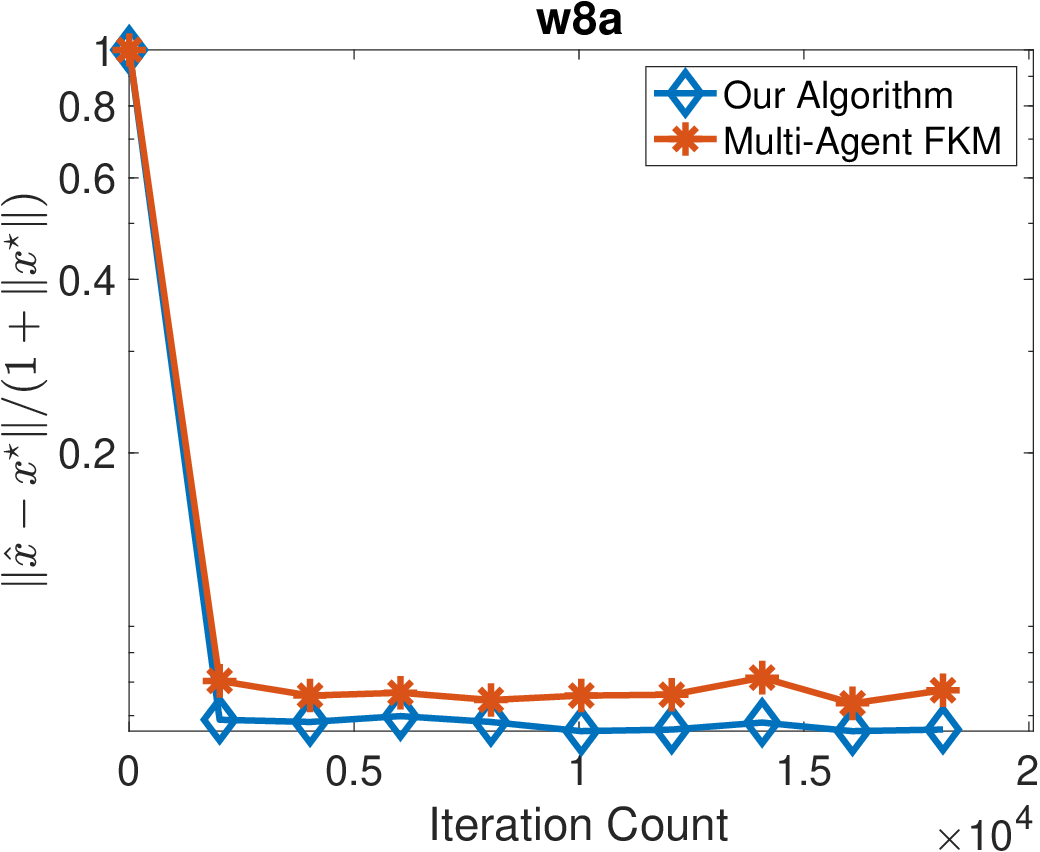}
\caption{Performance of both algorithms for solving distributed logistic regression problems. $\mu = 0.001$ is fixed. The numerical results are presented as error v.s. iteration count.}\label{fig:rlr}
\end{figure*}

\paragraph{Experimental results.} Figure~\ref{fig:rlr} indicates that Algorithm~\ref{Alg:MA} outperforms multi-agent FKM as it exhibits a faster convergence to Nash equilibrium in terms of iteration count. 

\paragraph{Comparisons with other state-of-the-art algorithms for distributed regularized logistics regression.} There is a vast literature on the distributed regularized logistic regression; indeed, the idea is similar and can be summarized in two aspects: (i) the reformulation based on consensus optimization~\citep{Boyd-2011-Distributed}; (ii) the algorithm based on either asynchronous block coordinate gradient descent (BCD)~\citep{Liu-2015-Asynchronous, Mahajan-2017-Distributed} or distributed alternating direction method of multipliers (ADMM)~\citep{Zhang-2014-Asynchronous, Aybat-2017-Distributed}. Focusing on solving the distributed regularized logistics regression problem, these state-of-the-art algorithms outperform Algorithm~\ref{Alg:MA} and multi-agent FKM in practice. However, such comparison seems unfair since these algorithms are all gradient-based algorithms that require the gradient feedback at each iteration while Algorithm~\ref{Alg:MA} and multi-agent FKM only use the (imperfect) bandit feedback at each iteration. 

In our experiment, we intend to restrict ourselves to the comparison among the algorithms that can be applicable to multi-agent bandit learning in games. This is because the goal of this paper is not to propose specific solution schemes for solving the problem of distributed regularized logistics regression. Rather, we aim to investigate optimal no-regret learning in strongly monotone games with bandit feedback, and the distributed regularized logistics regression problem mainly serves as an example that can be formulated as a strongly monotone potential game. As such, we believe that it suffices to compare Algorithm~\ref{Alg:MA} with multi-agent FKM and the numerical results have conveyed the core idea of this paper. Our numerical results are also only the evidence of the existence of problem instances where Algorithm~\ref{Alg:MA} outperforms multi-agent FKM in terms of iteration count, but are not meant to be used to argue that our algorithm is better than other existing algorithms, such as the ones specifically developed for solving distributed regularized logistics regression. 

\end{document}